\documentclass[lettersize,journal]{IEEEtran}
\usepackage{amsmath,amssymb}
\usepackage{amsthm}
 \usepackage{xcolor}
 \usepackage{bm}
 \usepackage{graphicx}
 \usepackage{epstopdf}
\usepackage{subfigure}
\usepackage{csquotes}
\usepackage[hidelinks]{hyperref}

\newtheorem{theorem}{Theorem}
\newtheorem{assumption}{Assumption}
\newtheorem{definition}{\textbf {Definition}}
\newtheorem{lemma}{Lemma}

\newtheorem{remark}{Remark}
\usepackage{enumerate}
\usepackage{enumitem}
\usepackage{cite}

\newcommand{\eat}[1]{}
\usepackage{todonotes}

\allowdisplaybreaks
\begin{document}

\title{ The Effectiveness of Local Updates for Decentralized Learning under Data Heterogeneity }

\author{Tongle Wu, Zhize Li and Ying Sun

\thanks{Tongle Wu is with the School of Electrical Engineering and Computer Science, Pennsylvania State University, USA (email: tfw5381@psu.edu).}
\thanks{Zhize Li is with the School of Computing and Information Systems, Singapore Management University, Singapore (email: zhizeli@smu.edu.sg).}
\thanks{Ying Sun is with the School of Electrical Engineering and Computer Science, Pennsylvania State University, USA (email: ybs5190@psu.edu).}
}




\maketitle

\begin{abstract}
We revisit two fundamental decentralized optimization methods, Decentralized Gradient Tracking (DGT) and Decentralized Gradient Descent (DGD), with multiple local updates. We consider two settings and demonstrate that incorporating local update steps can reduce communication complexity. Specifically, for  $\mu$-strongly convex and $L$-smooth loss functions, we proved that local DGT  achieves communication complexity {}{$\tilde{\mathcal{O}} \Big(\frac{L}{\mu(K+1)} + \frac{\delta + {}{\mu}}{\mu (1 - \rho)} + \frac{\rho }{(1 - \rho)^2} \cdot \frac{L+ \delta}{\mu}\Big)$}, 
{where $K$ is the number of additional local update}, $\rho$ measures the network connectivity and $\delta$ measures the second-order heterogeneity of the local losses. Our results reveal the tradeoff between communication and computation and show increasing $K$ can effectively reduce communication costs when the data heterogeneity is low and the network is well-connected. We then consider the over-parameterization regime where the local losses share the same minimums. We proved that employing local updates in DGD, even without gradient correction, achieves exact linear convergence under the Polyak-Łojasiewicz (PL) condition, which can yield a similar effect as DGT in reducing communication complexity. {}{Customization of the result to linear models is further provided, with improved rate expression. }Numerical experiments validate our theoretical results.
\end{abstract}


\section{Introduction}
We consider a network of $m$ agents collaboratively  solving the following \emph{deterministic} finite sum 
minimization problem:
\begin{equation}\label{problemformula}
\underset{\boldsymbol x \in \mathbb R^d}{\min} ~ \left\{f(\boldsymbol x) : = \frac 1 m \sum_{i=1}^m{f_i(\boldsymbol x)} \right\},
\end{equation}
where $\boldsymbol x \in \mathbb{R}^d$ is the decision variable shared by all agents, each $f_i: \mathbb{R}^d \to \mathbb{R}$ is the local loss function private to agent $i$.  The average loss $f$ is assumed to be $L$-smooth. 
The agents are allowed to communicate with each other via a fixed mesh network. Problem~\eqref{problemformula} has found applications in wide areas, including machine learning, signal processing, telecommunications, multi-agent control, etc. {\cite{lian2017can,DBLP:journals/ftml/BoydPCPE11,DBLP:journals/ftml/Sayed14,nedic2009distributed,jakovetic2014fast,shi2015extra,wu2024implicit} }.

Algorithms for solving~\eqref{problemformula} have been intensively studied in the literature \cite{nedic2018network}.  By collaboratively learning over the network, these algorithms aim to benefit from aggregating the resources (computation and samples) of multiple agents and attain improved training efficiency and model precision. However, due to decentralization, achieving such a goal comes at the cost of communication resources, which usually becomes the bottleneck of the algorithm for large-scale problems. Therefore, designing algorithms with reduced communication costs without sacrificing performance is paramount \cite{lian2017can}.

Among methods developed to improve communication efficiency, a well-known approach is to perform multiple local computation steps before communicating with network neighbors. This idea serves as the basis of the local SGD  algorithm~\cite{mcmahan2017communication} and has been employed as a standard practice for algorithm design in both federated learning (FL) and decentralized settings over the mesh network. Running multiple local gradient steps accelerates the optimization process. However, when the $f_i$s have heterogeneous gradients,  modifying algorithms by naively increasing the number of local updates may introduce ``client drift'', leading to non-convergence with a constant step size~\cite{woodworth2020local,woodworth2020minibatch,koloskova2020unified}.
{}{As a remedy, variance reduction (VR)/gradient tracking (GT) techniques \cite{xin2021improved,9084360, xin2021hybrid,xin2020variance,li2022destress}} 
are often employed to correct such bias. The idea leads to the renowned DGT algorithm with exact convergence~\cite{nedic2017achieving,xu2015augmented,di2016next}. Along this line, a fundamental question to understand is: 
\begin{center}
  \emph{Will increasing  local computations be beneficial in  reducing the communication overhead?}  
\end{center}
Furthermore, since these methods eliminate the impact of gradient heterogeneity, it is interesting to explore: 
\begin{center}
\emph{If, and to what degree, will the difference in the $f_i$s   affect the    computation-communication tradeoff?}    
\end{center}

Apart from the general setting where each local loss admits different minimizers, modern learning tasks often employ models with a large number of parameters. This leads to the over-parameterization regime with the models learned perfectly interpolating the training data~\cite{zhang2021understanding,ma2018power}. In this setting, problem~\eqref{problemformula} has the property that all the $f_i$s have at least one common minimizer, implying that the gradient heterogeneity vanishes at shared optimums. This motivates us to revisit the two questions above for the DGD algorithm without bias correction, which still achieves exact convergence~\cite{qin2022decentralized}.

Common to all these scenarios, the algorithms either do not suffer from or actively eliminate the impact of gradient (first-order) heterogeneity on the convergence. A natural conjecture is that under these cases, the difference in the second-order derivatives of the $f_i$s will dominantly affect the convergence rate. This is the central question we investigate in this work.
\vspace{-0.35cm}
\subsection{Summary of contributions}
In this paper, we study the above questions for the DGT and DGD algorithms with additional $K$ steps of local update, whose vanilla versions perform only one gradient update ($K=0$) followed by network communication. We consider two settings: (i) when the average loss $f$ is $\mu$-strongly convex or satisfies the PL condition, and (ii) the over-parameterization regime where all the $f_i$s share at least one minimizer and $f$ {}{satisfies} PL condition with parameter $\mu$ and $\beta$-weakly convex. Our contributions are: \\[1ex]
\noindent$\bullet$ 
In setting (i) with strongly convex $f$, we proved local DGT (DGT with properly incorporated local computations) converges to an $\varepsilon$-optimal solution linearly with communication complexity {}{${\mathcal{O}} \Big( \Big(\frac{L}{\mu (K+1)} + \frac{\delta + {}{\mu}}{\mu (1 - \rho)} + \frac{\rho }{(1 - \rho)^2} \cdot \frac{L+ \delta}{\mu}\Big) \cdot \log \frac{1}{\epsilon}\Big)$}\footnote{${\mathcal{O}} (\cdot)$ hides universal constant}, 
where  $\rho$ measures the network connectivity, $K$ is the {}{additional} number of local updates per communication round, and  $\delta$ measures the \emph{second-order heterogeneity} of the $f_i$s (cf. Assumption~\ref{hessian}). Our result reveals when the heterogeneity degree $\delta$ is small and the network is well-connected, increasing local updates $K$ can significantly reduce communication overhead. The result is further generalized for PL objectives. \\[1ex]   
 \noindent$\bullet$ In setting (ii), we proved when the degree of heterogeneity is small such that $\zeta:= 1 - (\delta/\mu)^2 \in (0,1)$, local DGD converges linearly with communication complexity  
 $\tilde {\mathcal{O}} \left( \frac{L}{\mu {}{(K+1) \zeta}}  + \frac{1}{1-\rho}  + \frac{\beta + \rho^2 L}{\mu (1-\rho)^2 {}{\zeta^2}}  \right)$.
 As such, increasing local updates is effective when 
 $f$ is (near)-convex, and the network is well-connected. \\[1ex]   
\noindent$\bullet$ We further improved the result for setting (ii) under the linear regression model. For arbitrary bounded $\delta$, it takes local DGD $\tilde{\mathcal{O}} \left( \frac{L}{\mu {}{(K+1)}} + \frac{\delta^2}{\mu^2 (1-\rho)}\right)$ communication rounds to reach the minimizer with minimum $\ell_2$-norm, where $\mu$ and $L$ are respectively the smallest and largest nonzero eigenvalue of the {}{empirical} covariance matrix.
\subsection{Related works}
In recent years, the design and analysis of federated/decentralized learning algorithms have been extensively explored in different settings. Many of them aim to show the algorithm of interest is convergent when employing multiple local computation steps. In the stochastic setting, there are some related works proving that performing local updates can reduce the steady-state error terms related to the gradient sampling variance, we refer the readers to the expression of these convergence rates reported in~\cite{koloskova2020unified,huang2023computation,liu2024decentralized}. Essentially, such a speed up is due to the reduction of variance when averaging independent stochastic gradients both across the iterations and the network agents. A majority of them, however, cannot show the benefit of local updates to the \emph{transient terms}{}{~\cite{chen2021accelerating}.} 
The drawback is two-fold. First, when particularized to the deterministic setting, they fail to address the critical algorithm design question of whether applying local updates can accelerate communication. 

Second, in the stochastic setting, they cannot show algorithms of local SGD type is more communication efficient than the simple minibatch SGD baseline~\cite{woodworth2020local,woodworth2020minibatch,mishchenko2022proxskip}. 
In the following, we review and contrast our results to those managed to show local updates are beneficial in reducing the communication complexity either in the deterministic setting or transient terms in the stochastic setting, which are far less.

In the FL setting where a center node is connected to and coordinates all worker nodes, local GD/SGD~\cite{mcmahan2017communication} has been among the most popular algorithms employing local updates in between communication rounds. In the presence of gradient heterogeneity, it is known that local GD and its variations cannot achieve exact convergence with a constant\footnote{The step size does not depend on the optimization horizon.} step size due to ``client drift''~\cite{karimireddy2020scaffold,woodworth2020minibatch}. With a diminishing step size, \cite{woodworth2020minibatch} is the first to show local SGD improves over minibatch SGD  with each $f_i$ being convex. To mitigate the impact of gradient heterogeneity, 
 VR and primal-dual-based methods correct the gradient direction, leading to exact convergence with constant step size, see e.g.~\cite{karimireddy2020scaffold,haddadpour2021federated,mishchenko2022proxskip}. Provable benefit of performing multiple local updates 
 while guaranteeing the exact convergence is established only in restricted settings, such as requiring each $f_i$ to be strongly convex and smooth ~\cite{mishchenko2022proxskip} or quadratic~\cite{karimireddy2020scaffold}. Our result when particularized to the FL setting, incorporates a larger family of problems and recovers/improves existing results, see discussions in Sec.~\ref{sec:compare-literature}.

In the decentralized setting with a mesh communication network,  
it is known that even the vanilla DGD algorithm suffers from non-convergence under gradient heterogeneity~\cite{yuan2016convergence}. Despite being standard techniques to overcome the issue, algorithms employing gradient correction techniques with proactively designed communication-computation pattern {}{are} discussed in only a few recent works {\cite{mishchenko2022proxskip,nguyen2023performance,huang2023computation,liu2024decentralized,ge2023gradient,alghunaim2024local} }. 
These results show that in the stochastic setting, increasing local updates can reduce the steady-state stochastic error but not the transient terms.   
The only exception we are aware of is Scaffnew~\cite{mishchenko2022proxskip}, which shows provable benefit assuming each $f_i$ is strongly convex and smooth.

When it comes to overparameterized objective functions, the performance of DGD has been studied in~\cite{koloskova2020unified,qin2022decentralized}.  Exact linear convergence is shown in~\cite{koloskova2020unified}  assuming each $f_i$ is strongly convex {}{with a common minimizer}, with communication complexity $\Tilde{\mathcal O} \left( \frac{L}{\mu(1-\rho)}\right)$. The result is further improved to $\Tilde{\mathcal O}\left( \frac{L}{\mu K}\right)$ in~\cite{qin2022decentralized} under the same condition, showing network-independent complexity. However, it is worth noting that in the case of\cite{qin2022decentralized}, all $f_i$s will share the same unique minimizer, implying that the agents can just solve their local problems independently without the network. 
On the contrary, we analyzed the convergence under a more general PL condition on the \emph{average loss} $f$, but allowing local $f_i$s to have different sets of minimizers. Consequently, our problem structure is fundamentally different from~\cite{koloskova2020unified} and~\cite{qin2022decentralized}, and network communication is necessary to find a consensus solution.

Finally, we note that there exist works considering distributed optimization in solving some overparameterized problem instances. For example, a compression-based algorithm was proposed in \cite{song2022distributed}, achieving problem dimension-independent communication complexity. A communication-efficient distributed algorithm was developed in~\cite{khanduri2021decentralized} for a class of overparameterized kernel learning.
Linear convergence of local GD/SGD for overparameterized neural networks has been established in \cite{deng2022local,song2023fedavg}. These results are complementary to ours, as they either study under special problem instances, propose new methods, or prove convergence without revealing the benefit of local updates. In contrast, we aim to understand the usefulness of local updates for the arguably most widely used distributed algorithms: DGD and DGT.

\vspace{0.2cm}
\noindent\emph{Notations.} We use lowercase bold letters (e.g., $\boldsymbol x, \boldsymbol y$) and capital bold letters (e.g., $\boldsymbol A, \boldsymbol B$) to denote vectors and matrices, respectively.  $\|\cdot\|$ and $\|\cdot\|_2$ respectively denotes the Frobenius norm and spectral norm (maximum singular value) of a matrix. Denote $[m] \triangleq \{1, 2, \cdots, m\}$ and $\boldsymbol I_m \in \mathbb R^{m\times m}$ is the identity matrix. We write $a_n \lesssim b_n$ if there exists a universal constant $C$ such that $a_n \leq Cb_n$ and $a_n = \mathcal O(b_n)$ have the same meaning. Moreover, $a_n = \Tilde{\mathcal O}(b_n)$ means $a_n \leq Cb_n$ up to some polylogarithmic factors. 


\section{Preliminaries: DGD/DGT with local updates}\label{Pre}

We consider solving problem~\eqref{problemformula} over a mesh network modeled as an undirected graph $\mathcal{G} := \{\mathcal{V}, \mathcal{E}\}$, with nodes $\mathcal{V}: = \{1, \ldots, m\}$ representing the set of agents and edges $\mathcal{E} \subseteq \mathcal{V} \times \mathcal{V}$ representing the communication links. An unordered pair $\{i,j\} \in \mathcal{E}$ if and only if there is a bi-directional communication link between them. The set of one-hop neighbors for agent $i$ is denoted by {}{$\mathcal N_i = \left\{ j \in \mathcal V| (i,j) \in \mathcal E\right\} \bigcup \left\{ i \right\} $}.

For solving problem~\eqref{problemformula}, the local DGD alternates two steps: 1) local optimization where each agent $i$ updates its local $\bm{x}^i$ by performing $K$ steps of gradient descent; 2) network communication where each agent aggregates information from their neighbors by weighted averaging the model parameters following a gossip consensus communication protocol~\cite{gossip2019tutorial}. Formally, let $\bm{x}^i$ denote the local model of agent $i$. In each communication round $r$, agent $i$ performs:
{}{
\begin{align}\label{alg:DGD}
\begin{split}
        \bm{x}_{r,k+1}^i & = \bm{x}_{r,k}^i - \eta  \nabla f_i (\bm{x}_{r,k}^i), \ \forall k = 0, \ldots, K-1,\\
 \bm{x}_{r+1,0}^{i} & = \sum_{j \in \mathcal{N}_i} w_{ij} \left(\bm{x}_{r,K}^j - \eta  \nabla f_j (\bm{x}_{r,K}^j) \right),
\end{split}
\end{align}
}
\begin{remark}\label{remark1}
It is worth noting that when $K=0$ (without additional local updates), there is only one step of local update in \eqref{alg:DGD}, which reduces to the standard DGD method. 
\end{remark}
It is well recognized that when the $f_i$s are heterogeneous, even the vanilla DGD algorithm cannot converge to the exact solution due to the gradient heterogeneity, i.e., $\sum_{i = 1}^m \|\nabla f_i (\bm{x}) - \nabla f(\bm{x})\| $ is definitely positive \cite{yuan2016convergence,zeng2018nonconvex}. As a remedy, the gradient tracking technique \cite{xu2015augmented,di2016next,nedic2017achieving} is employed to correct the descent direction and achieve exact convergence. When incorporating multiple local updates, the local DGT algorithm is developed in~\cite{nguyen2023performance,huang2023computation}. Specifically, in each communication round $r$, agent $i$ performs:
\begin{align}\label{eq:DGT-comp}
\begin{split}
       \bm{x}_{r,k+1}^i & = \bm{x}_{r,k}^i - \eta  \bm{y}_{r,k}^i\\
   \boldsymbol y^i_{r,k+1} &= \boldsymbol y^i_{r,k} + \nabla f_i(\boldsymbol x^i_{r,k+1}) - \nabla f_i(\boldsymbol x^i_{r,k})
\end{split}
\end{align}
for $k = 0, \ldots, K-1$ while communicates in the last step:
\begin{align}\label{stand-dgt}
    \boldsymbol x^i_{r+1,0} & = \sum_{j\in \mathcal N_i}{w_{ij}\left( \boldsymbol x^j_{r,K} - \eta \boldsymbol y^j_{r,K}\right) }\\
    \boldsymbol y^i_{r+1,0} & = \sum_{j\in \mathcal N_i}{w_{ij}\left( \boldsymbol y^j_{r,K} + \nabla f_j(\boldsymbol x^j_{r+1,0}) - \nabla f_j(\boldsymbol x^j_{r,K})\right)}.\nonumber
\end{align}
The auxiliary variable $\bm{y}^i_{r,0}$ is the tracking variable introduced as an estimator of the average gradient $\nabla f (\bm{x}^i_{r,0})$, initialized as $\bm{y}_{0,0}^i = \nabla f_i (\bm{x}_{0,0}^i)$. 
{}{
\begin{remark}
Similar to Remark \ref{remark1}, when $K=0$ (without additional local updates), \eqref{eq:DGT-comp} is skipped, reducing the local DGT method to the standard  Adapt-Then-Combine (ATC) form of DGT method \cite{alghunaim2024enhanced}, where only one step of the local update is performed as \eqref{stand-dgt}.
\end{remark}
}
We give the following standard assumptions for solving \eqref{problemformula}. 
\begin{assumption}\label{sm}
The average loss function $f: \mathbb{R}^d \to \mathbb{R}$ is  $L$-smooth, i.e., for all $\bm{x}, \bm{y} \in \mathbb{R}^d$ it holds
\begin{align}
    \| \nabla f(\bm{x}) - \nabla f(\bm{y})\| \leq L \| \bm{x} - \bm{y}\|.
\end{align}
\end{assumption}

\begin{assumption}\label{W}
    The communication network $\mathcal{G}$ is connected. The weight matrix $\bm{W} = [w_{ij}]_{ij=1}^m$ of graph $\mathcal{G}$  satisfies 
    (i)    $w_{ij} = 0$ for all $\{i,j\} \notin \mathcal{E}$;
    (ii) doubly stochastic: $\mathbf{1}^T_m \bm{W} = \mathbf{1}^T_m$ and $ \bm{W} \mathbf{1}_m = \mathbf{1}_m$;
    (iii) and $\rho := \left \| \bm{W} - (1/m) \bm{1}_m \bm{1}_m^T \right\|_2 < 1$.
\end{assumption}
 We further assume the $f_i$s have bounded second-order heterogeneity, as given in Assumption~\ref{hessian}.
\begin{assumption}\label{hessian}
  Each $f_i$ is $C^1$ and satisfies for all $\bm{x}, \bm{y} \in \mathbb{R}^d$ and $i 
  \in [m]$:
  \begin{equation}
\| \nabla f(\boldsymbol x) - \nabla f_i(\boldsymbol x) -(\nabla f(\boldsymbol y) - \nabla f_i(\boldsymbol y) )\|^2 \leq \delta^2 \|\boldsymbol x - \boldsymbol y\|^2.
\end{equation}
\end{assumption}
Note that when the $f_i$s are twice continuously differentiable continuously differentiable, Assumption~\ref{hessian} can be implied by the following bound on Hessian similarity:
\begin{equation}
\|\nabla^2 f(\boldsymbol x) - \nabla^2 f_i(\boldsymbol x)\| \leq \delta; \quad \forall \boldsymbol x \in \mathbb R^d, \, \forall i \in [m].
\end{equation}
This assumption has been widely used in distributed optimization literature \cite{karimireddy2020scaffold,pmlr-v247-patel24a,lin2024stochastic}.

\section{Main results}\label{result}
This section gives theoretical results of local DGT and DGD algorithms. We first consider the setting where the average loss $f$ is strongly convex. The communication complexity of local DGT  is presented in Section \ref{sec:local-DGT}. The result is further generalized for weakly convex $f$ under the PL condition. Then we consider over-parameterized problems where the local $f_i$s have shared minimizers. The communication complexity of local DGD is provided in Section \ref{sec:local-DGD}. In both cases, 
our result shows the communication-computation tradeoff is affected by the second-order heterogeneity of the $f_i$s.

\subsection{Local DGT under strong convexity}\label{sec:local-DGT}

\begin{assumption}\label{scv}
The average loss function $f$ is $\mu$-strongly convex, i.e., for all $\boldsymbol x, \boldsymbol y \in \mathbb R^d :$
\begin{equation}
f(\boldsymbol y) \geq f(\boldsymbol x) + \langle \nabla f(\boldsymbol x), \boldsymbol y - \boldsymbol x\rangle + \frac{\mu}{2} \|\boldsymbol y - \boldsymbol x\|^2.
\end{equation}
\end{assumption}

To state the convergence, we introduce a potential function that measures the optimality gap at the end of round $R$:
\begin{multline}\label{def:potential-scvx}
      P^R :=  \mu \Big( f(\overline{\boldsymbol x}^{R+1}) - f^\star\Big)    \\
    + \frac{1}{\sum_{r=0}^R{z_r}} \sum_{r=0}^R{  \Bigg( \frac{z_r\sum_{i=1}^m \!{\sum_{k=0}^K{\|\nabla f(\boldsymbol x^i_{r,k})\|^2}}} {m(K+1)} \Bigg)},
\end{multline}
where {}{$\overline{\boldsymbol x}^{R+1}: =  \frac 1 m \sum_{i=1}^m \boldsymbol x_{R+1,0}^i$} and $f^\star$ denotes the minimum of $f$ and $\{z_r\}$ is an increasing geometric sequence whose formula is given in the proof.

\begin{theorem}(Strong convexity).\label{th1}
Consider problem~\eqref{problemformula} with the average loss $f$ satisfying  Assumption~\ref{sm}  and \ref{scv}. Suppose the $f_i$s satisfy Assumption~\ref{hessian}.  Let $\{ \bm{x}_{r,k}^i\}$ be the sequence generated by the local DGT algorithm under Assumption~\ref{W}. 
Then there exists step size $\eta \lesssim  \min\Big\{ \frac 1 L, \frac{1-\rho}{ K\mu}, \frac{1-\rho}{K\delta}, \frac{(1-\rho)^2}{\rho K(\delta+L)} \Big\}$ such that
\begin{align}\label{eq:thm1-rate}
     P^R \leq  {}{{\mathcal O}} \left(\exp \left[ -\frac{ {}{(K+1)}\mu R}{L+ \frac{(\delta+\mu)K}{1-\rho}+\frac{ \rho (\delta + L) K}{(1-\rho)^2}} \right] \right),     
\end{align}
where $\lesssim$ denotes inequalities up to multiplicative absolute constants that do not depend on any problem parameters. Consequently, to reach a solution satisfying $f(\overline{\boldsymbol x}^{R+1}) - f^\star \leq \varepsilon$, the communication rounds  it takes for local DGT is 
\begin{align}\label{eq:complexity-scvx}
 R = {}{{\mathcal O}} \left(\left(\frac{L}{\mu {}{(K+1)}} + \frac{\delta + {}{\mu} }{\mu (1 - \rho)} + \frac{\rho }{(1 - \rho)^2} \frac{ \delta + L}{\mu} \right) \log \frac{1}{ {}{\mu} \varepsilon} \right)  . 
\end{align}
\end{theorem}

The proof of this theorem is in Section \ref{pr-th1} and Theorem~\ref{th1} provides the following insights on the algorithmic design. 

\noindent{\textbf{Communication-computation tradeoff.}}
The expression \eqref{eq:complexity-scvx} shows when the condition number $\frac{L}{\mu}$ is relatively large, increasing the number of local updates can effectively reduce the communication cost. However, when $K$  exceeds {}{$K^\star :=  \mathcal O \left( \Big\lfloor \frac{L(1-\rho)^2}{\rho(L-\mu)+\delta+\mu} \Big\rfloor \right) $}, 
the last two terms will be dominating. As such, further increasing local updates beyond $K^\star$ offers marginal benefit in reducing communication overhead, resulting in unnecessary computation costs. When the data is significantly heterogeneous or network connectivity is weak, it is better off to communicate per iteration. 

\noindent{\textbf{Influence of data heterogeneity.}} In local DGT, the gradient tracking technique iteratively reconstructs the average gradient $\nabla f (\bm{x}^i)$ using the tracking variable $\bm{y}^i$ for each agent and thus eliminates the impact of gradient (first-order) heterogeneity on the convergence rate. Our result shows in this case, data heterogeneity affects the convergence rate through higher-order terms. Specifically, a smaller $\delta$, indicating the local $f_i$s are more similar, leads to lower communication complexity. Moreover, since $K^\star$ increases as $\delta$ decreases, this suggests local updates will be more effective in saving communication costs when the second-order heterogeneity is relatively small.

\noindent{\textbf{Influence of network connectivity.}} From the last two network dependence terms in~\eqref{eq:complexity-scvx}, the stronger the connectivity is, the fewer communication rounds it takes to reach an optimal solution. Similar to $\delta$, one finds a smaller $\rho$ yields a larger $K^\star$, and thus local updates will be more effective when the network is more connected.

\subsubsection{Comparison to existing works}\label{sec:compare-literature}
By particularizing the 
communication complexity given by~\eqref{eq:complexity-scvx}, we show in special cases that our result recovers/improves existing works.\\[1.5ex]
{\bf The FL setting.} The local DGT algorithm can be implemented over a {}{fully connected work} with $\bm{W} = (1/m) \mathbf{1}_m \mathbf{1}_m^T$, we have  $\rho = 0$ and the expression~\eqref{eq:complexity-scvx} reduces to {}{$\tilde{\mathcal{O}} \left(\frac{L}{\mu (K+1)} + \frac{\delta + {}{\mu}}{\mu} \right)$}. This matches and generalizes the communication complexity of Scaffold~\cite[Theorem IV]{karimireddy2020scaffold}, which requires that each local loss $f_i$ being strongly convex and quadratic.\\[1.5ex]
{\bf Mirror descent type methods.} 
By eliminating the $\bm{y}^i$ variable we can rewrite~\eqref{eq:DGT-comp} as
\begin{align}\label{eq:DGT-comp-reform}
    \begin{split}
       \bm{x}_{r,k+1}^i & = \bm{x}_{r,k}^i - \eta  \left( \boldsymbol y^i_{r,0} + \nabla f_i(\boldsymbol x^i_{r,k}) - \nabla f_i(\boldsymbol x^i_{r,0})\right).
\end{split}
\end{align}
By letting $K \to \infty$, Eq.~\eqref{eq:DGT-comp-reform} can be viewed as solving the following problem using gradient descent:
\begin{align}
    \min_{\bm{x}^i}~f_i(\bm{x}^i) + (\bm{y}^i_{r,0} - \nabla f_i (\bm{x}_{r,0}^i))^T (\bm{x}^i - \bm{x}_{r,0}^i)
\end{align}
with initialization $\bm{x}^i = \bm{x}_{r,0}^i$. This recovers the decentralized mirror descent type algorithms in \cite{sun2022distributed,li2020communication}.

As $K \to \infty$, the first term in~\eqref{eq:complexity-scvx} goes to zero and the expression becomes $\tilde{\mathcal{O}}\left(  \frac{\delta + {}{\mu}}{\mu (1 - \rho)} + \frac{\rho }{(1 - \rho)^2} \cdot \frac{L+ \delta}{\mu} \right)$. Furthermore, under the condition that  $\frac{\rho}{1-\rho} \lesssim \frac{\delta +{\mu} }{L+\delta}$ we obtain $\tilde{\mathcal{O}} \left( \frac{\delta + {}{\mu} }{\mu}\right)$ if $\delta>\mu$. This matches the complexity in \cite{sun2022distributed} but under condition {}{$\frac{\rho}{(1-\rho)^2} \lesssim \frac{1}{\left(1+\frac L \delta \right) \left( \frac {L+\delta} \mu \right)}$} (sufficiently strong network connectivity), and improves over that in \cite{li2020communication}. 

\noindent {\bf Gradient correction methods with local updates.} Some existing works have also analyzed the convergence of decentralized gradient correction methods with local updates. Our analysis improves the results obtained in these works~\cite{nguyen2023performance,huang2023computation,liu2024decentralized,alghunaim2024local,mishchenko2022proxskip}. In detail, the communication complexities of K-GT and Periodical GT in~\cite{liu2024decentralized} and LU-GT in~\cite{nguyen2023performance}  are all $\mathcal O\left( \frac{L}{\epsilon}\right)$ for non-convex $f$. For strongly convex case, FlexGT in \cite{huang2023computation} and LED in \cite{alghunaim2024local} have been proved {}{to converge} with $\tilde{\mathcal O}\left( \frac{L}{\mu(1-\rho)^2}\right)$ and $\tilde{\mathcal O}\left( \frac{L}{\mu(1-\rho)}\right)$ communication complexity, respectively. All these results do not reflect the usefulness of local computations. Scaffnew proposed in \cite{mishchenko2022proxskip} employs a randomized strategy that invokes network communication at each step with probability $p$ and obtains communication complexity $\Tilde{\mathcal O}\left(\frac{Lp}{\mu} + \frac{1}{p(1-\rho)}\right)$. Our result differs from Scaffnew in the following aspects. In terms of the problem class, Scaffnew requires that each local loss $f_i$ be strongly convex and smooth, while we just need a weaker condition that the average $f$ is strongly convex. In terms of the algorithm and analysis, Scaffnew is essentially a primal-dual type method, whereas our analysis for local DGT is primal. This is also the reason why our analysis can be applied to a wider range of problems. In addition, the best rate
Scaffnew can obtain is ${}{\tilde {\mathcal{O}}}\left(\sqrt{\frac L \mu} \right)$ even for $\rho =0$ whereas we can achieve $\mathcal{O}(1)$ complexity 
for small $\rho$ and $\delta$. Implementation-wise, local DGT uses fixed $K$ local updates while Scaffnew employs random local updates, which is more demanding in agent signaling. 

\subsubsection{Generalization to PL objectives}
Our proof can be generalized to weakly convex $f$ under following assumptions.

\begin{assumption}\label{wcv}
 The average loss $f$ is $\beta$-weakly convex, i.e.,  there exists a constant $\beta \geq 0$ such that $f + \frac{\beta}{2}\|\cdot\|^2$ is convex.
\end{assumption}

Note that Assumption~\ref{sm} implies $f$ must be weakly convex. However, the constant $\beta$ can be much smaller than $L$. For example, when $f$ is convex we have $\beta =0$.

\begin{assumption}\label{PL}
The average loss function $f $ satisfies the PL condition with parameter $\mu$, i.e., for all $\bm{x} \in \mathbb{R}^d$ it holds
\begin{equation}
\left\| \nabla f(\boldsymbol x) \right\|^2 \geq 2\mu(f(\boldsymbol x) - f^\star).
\end{equation}
\end{assumption}
Note that a function that is $\mu$-strongly convex will also satisfy the PL condition with parameter $\mu$.


\begin{theorem}
(The PL condition).\label{coro1}
Consider problem~\eqref{problemformula} with the average loss $f$ satisfying   Assumption~\ref{sm}, \ref{wcv} and \ref{PL}. Suppose the $f_i$s satisfy   Assumption~\ref{hessian}.
Let $\{ \bm{x}_{r,k}^i\}$ be the sequence generated by the local DGT algorithm under Assumption~\ref{W}. Then there exists step size $\eta \lesssim \min \Big\{ \frac 1 L, \frac{1-\rho}{K\mu}, \frac{1-\rho}{K\delta}, \frac{1-\rho}{K\beta},  \frac{(1-\rho)^2}{\rho K(\delta + L + \beta)}\Big\}$ such that
\begin{align}\label{GT-wcpl}
    \begin{split}
   \!\! P^R \leq  {}{{\mathcal O}} \left(\exp \left[ -\frac{{}{(K+1)} \mu R}{L+ \frac{(\delta+\mu+\beta)K}{1-\rho}+\frac{\rho (\delta+L + \beta)K}{(1-\rho)^2}} \right] \right) .    
    \end{split}
\end{align}
Consequently, to reach a solution satisfying $f(\overline{\boldsymbol x}^{R+1}) - f^\star \leq \varepsilon$, the communication rounds it takes for local DGT is 
{
\begin{align}\label{eq:complexity-PL}
 R = {}{\tilde{\mathcal O}} \left(\frac{L}{\mu {}{(K+1)}} + \frac{\delta  + \beta + {}{\mu}}{\mu (1 - \rho)} + \frac{\rho }{(1 - \rho)^2} \cdot \frac{\delta + L + \beta}{\mu}\right). 
\end{align}
}
\end{theorem}

The proof of Theorem \ref{coro1} is in  Section \ref{proof-DGT-PL}. Theorem~\ref{coro1} shows the impact of nonconvexity on the communication cost through parameter $\beta$. In the special case $\beta = 0$, implying $f$ is convex, the bound~\eqref{eq:complexity-PL} reduces to~\eqref{eq:complexity-scvx}.

\subsection{Local DGD under over-parameterization}\label{sec:local-DGD}
 In this section, we consider the over-parameterization setting where the number of model parameters is large enough to interpolate the training data of all agents \cite{koloskova2020unified}, and the local losses have shared minimizers. Formally, we assume the $f_i$s satisfy Assumption~\ref{overpara}.  

\begin{assumption} \label{overpara}
{}{There exists $\boldsymbol x^\star \in \arg\min_{\boldsymbol x \in \mathbb R^d} f(\boldsymbol x)$ such that $\nabla f_i(\boldsymbol x^\star) =0$ for all $i \in [m]$.}
\end{assumption}

In this interpolation regime, the $f_i$s enjoy the property that the gradient difference $\nabla f_i(\bm{x}^i) - \nabla f(\bm{x}^i)$ vanishes at the common minimizer $\bm{x}^\star$. Compared to the general setting where local DGT  actively aligns the local descent direction $\bm{y}^i$ to the average gradient $\nabla f(\bm{x}^i)$ using gradient tracking, the local DGD algorithm converges automatically to the exact minimizer thanks to this nice property \cite{koloskova2020unified,qin2022decentralized}. Given the similarity of the behavior of first-order difference in the two settings (both diminish as the algorithm progresses towards $\bm{x}^\star$),
the study of local DGT in Section \ref{sec:local-DGT} then motivates us to investigate whether higher-order differences of the $f_i$s  will have an analog effect on the convergence rate of local DGD in the over-parameterization setting.

\subsubsection{PL objectives}

We first consider the problem assumptions consistent with Theorem~\ref{coro1}, i.e., the average loss function $f$ is PL. Specific models satisfying this property include over-parameterized linear regression~\cite{woodworth2018graph}, deep neural network~\cite{nguyen2018optimization,nguyen2021tight,charles2018stability} and non-linear systems \cite{liu2022loss}. We further relax the uniform second-order heterogeneity Assumption~\ref{hessian} to the following weaker assumption, which only requires the boundedness with respect to the shared minimizer $\bm{x}^\star$.

\begin{assumption}\label{rehessian}
Each $f_i$ is $C^1$ and satisfies for all $\bm{x}\in \mathbb{R}^d$ and $i 
  \in [m]$:
  \begin{equation}
\|\nabla f(\boldsymbol x) - \nabla f_i(\boldsymbol x) - (\nabla f(\boldsymbol x^\star) - \nabla f_i(\boldsymbol x^\star))\|^2 \leq \delta^2 \|\boldsymbol x - \boldsymbol x^\star\|^2,
\end{equation}
{}{where $\boldsymbol x^\star$ is any  minimizer of $f$ satisfying Assumption \ref{overpara}.}
\end{assumption}
Recall the potential function $P^R$ in \eqref{def:potential-scvx} and {}{the second-order heterogeneity parameter $\zeta:= 1 - (\delta/\mu)^2 \in (0,1)$}, the result of local DGD is provided in the next Theorem~\ref{thm2}.
\begin{theorem}\label{thm2}
 Consider problem~\eqref{problemformula} in the over-parameterization setting satisfying Assumption~\ref{overpara}. Suppose that the  average loss $f$ satisfies  Assumption~\ref{sm}, \ref{wcv} and \ref{PL};  and the $f_i$s satisfy the restricted second-order heterogeneity Assumption~\ref{rehessian} with $\delta < \mu$. 
Let $\{ \bm{x}_{r,k}^i\}$ be generated by the local DGD algorithm under Assumption~\ref{W}. Then there exists a step size   $\eta \lesssim \min \Big\{\frac 1 L, \frac {1-\rho} {K\mu \zeta}, \frac{(1-\rho)\zeta}{K\beta}, \frac{(1-\rho)^2\zeta}{\rho^2K(L + \beta)}\Big\}$ 
such that
\begin{equation}
\begin{aligned}
P^R \leq {}{{\mathcal O}} \left(\exp\left[ - \frac{ {}{(K+1)} \mu R}{ \frac{L}{\zeta}  +  \frac{K\mu}{1-\rho} + \frac{\beta K}{(1-\rho) \zeta^2}+ \frac{ \rho^2 (L+\beta)  K}{(1-\rho)^2 \zeta^2}
}\right] \right).
\end{aligned}
\end{equation}
Consequently, to reach a solution satisfying $f(\overline{\boldsymbol x}^{R+1}) - f^\star \leq \varepsilon$, the communication rounds  it takes for local DGD is 
\begin{align}\label{eq:complexity-DGD-PL}
   R =  {}{{\mathcal O}} \left( \left(\frac{L}{\mu {}{(K+1)} \zeta}  + \frac{1}{1-\rho} + \frac{\beta + \rho^2 L}{\mu (1-\rho)^2 \zeta^2}  \right) \log \frac{1}{ {}{\mu}\varepsilon}\right).
\end{align}
\end{theorem}
The proof of this theorem is in Section \ref{pr-th2}. Similar to the analysis of local DGT, the communication complexity~\eqref{eq:complexity-DGD-PL} shows under the condition $\delta < \mu$, performing local updates in DGD is effective when $\beta$ and $\rho$ are small. In the special case with {}{$\beta = \rho=0$} 
and $K \to \infty$, our result implies  $\Tilde{\mathcal O}\left(1\right)$ communication complexity suffices for local DGD to achieve $\varepsilon$-optimality for convex $f$.

\begin{remark}
    The condition $\delta < \mu$ is necessary for our analysis to show the effectiveness of local updates. Notably, when $\delta \to \mu$, the complexity bound~\eqref{eq:complexity-DGD-PL} becomes loose and can be less insightful. In the next Section \ref{sec:OLS} we show that this condition can be removed for the quadratic loss $f_i$, and the communication complexity can be tightened. For a general loss $f$, investigating the possibility of removing the condition remains an interesting future work. {}{We also mention that the convergence rate of local DGT given by Theorem~\ref{coro1} can be directly applied to the over-parameterized problems,  showing the algorithm converges to the set of minimizers without requiring $\delta < \mu$. Nonetheless, it is an interesting question if one can improve upon~\eqref{eq:complexity-PL} under the extra Assumption~\ref{rehessian}.} 
\end{remark}


\subsubsection{A special case: the overparameterized linear regression}\label{sec:OLS}
We strengthen the convergence rate of local DGD for linear regression. In this problem, each agent $i$ owns local data $\mathcal{D}_i$ consists of design matrix $\bm{A}_i \in \mathbb{R}^{n \times d}$ and observation $\bm{b}_i \in \mathbb{R}^n$. The agents collaboratively find a linear model that interpolates the data by solving problem~\eqref{problemformula} with $f_i(\bm{x}) = \frac 1 {2n} \|\boldsymbol A_i \boldsymbol x - \boldsymbol b_i\|^2$. We consider the over-parameterization setting where $d > N$, i.e., the problem dimension exceeds the gross sample size of all agents, which implies Assumption~\ref{overpara} hold.

Given the problem setup, define $L$ and $\mu$ respectively as the largest and smallest \emph{nonzero} eigenvalue of the matrix $(1/m) \cdot \sum_{i= 1}^m \bm{A}_i^T \bm{A}_i$, and let $\boldsymbol x^\star$ be the minimum norm solution of~\eqref{problemformula}. Furthermore, define the following potential function 
\begin{align}\label{ols_conv}
 P^R_{\rm OLS} = \left\| \overline{\boldsymbol x}^R - \boldsymbol x^\star \right\|^2 + \frac 1 m \sum_{i=1}^m {\left\| {}{\boldsymbol x^i_{R,0}} - \overline{\boldsymbol x}^R \right\|^2}.
\end{align} 
\begin{theorem} \label{thm3}
Let $\{ \bm{x}_{r,k}^i\}$ be the sequence generated by the local DGD algorithm for the interpolation over-parameterized linear regression problem under Assumption~\ref{W}, {}{\ref{rehessian}} 
with initialization $\boldsymbol x^i_{0,0} = \boldsymbol 0, \ \forall \ i=1,\cdots,m$. Then there exists a step size $\eta \lesssim \min\Big\{\frac 1 L, \frac{1-\rho}{K\mu}, \frac{\mu(1-\rho)}{K\delta^2}\Big\}$ such that
\begin{equation}
\begin{aligned}
  P^R_{\rm OLS} \leq {}{{\mathcal O}} \left( \exp{\left[ - \frac{{}{(K+1)}R}{\frac{L}{\mu} + \left(\frac{\delta}{\mu}\right)^2\frac{K}{1-\rho}   }\right]} \right). 
\end{aligned}
\end{equation}
Consequently, to reach a solution satisfying {$\left\|\overline{\boldsymbol x}^{R+1} - \boldsymbol x^\star \right\|^2 \leq \varepsilon$}, the communication rounds  it takes for local DGD is 
\begin{align}\label{eq:complexity-OLS}
    R = {}{{\mathcal O}} \left( \left( \frac{L}{\mu {}{\left(K+1\right)} } + \frac{\delta^2}{\mu^2 (1 - \rho)} \right) \log \frac{1}{\varepsilon}\right).
\end{align}
\end{theorem}

The proof of this theorem is in Section \ref{pr-th3}. Theorem~\ref{thm3} improves the result of Theorem~\ref{thm2} (with $\beta = 0$) in the following aspects. First, it removes the condition  $\delta<\mu$ and applies to problems with any bounded $\delta$. Second, under the same condition of Theorem~\ref{thm2}, i.e., $\delta < \mu$, the expression~\eqref{eq:complexity-OLS} becomes ${\mathcal{O}} ( ( \frac{L}{\mu {}{(K+1)}} + \frac{1}{ 1 - \rho} ) \log \frac{1}{\varepsilon})$. This improves~\eqref{eq:complexity-DGD-PL} in both the first and the third term.
Two consequences follow: in the limit $\rho \to 1$, \eqref{eq:complexity-OLS} demonstrates a better dependency on the network connectivity; and when $K$ is sufficiently large, {}{terms involving $L$ in ~\eqref{eq:complexity-OLS} becomes negligible while those in} \eqref{eq:complexity-DGD-PL} additionally requires $\rho$ to be small enough. \\[2ex]
{\it When to apply gradient tracking?} Although the local DGD is convergent, one can always employ an extra gradient tracking step in this setting at twice the communication cost per round. The question is if doing so is worthwhile, in the sense that it can help reduce the overall communication complexity. To investigate the answer, we invoke Theorem~\ref{coro1} and set $\beta=0$ in \eqref{eq:complexity-PL}. Since the least squares loss function $f$ is PL, 
this gives the communication complexity of local DGT for over-parameterized linear regression as 
\begin{align}\label{eq:complexity-OLS-DGT}
   {}{{\mathcal{O}}} \left( \left(\frac{L}{\mu {}{(K+1)}} + \frac{\delta + {}{\mu} }{\mu (1 - \rho)} + \frac{\rho }{(1 - \rho)^2} \cdot \frac{ \delta + L}{\mu} \right)  \log \frac{1}{\varepsilon} \right).
\end{align}
Compared to \eqref{eq:complexity-OLS}, we observe that when $\delta$ is large
local DGT exhibits better dependence on the heterogeneity term $\frac{\delta}{\mu}$ in contrast to $\left( \frac{\delta}{\mu}\right)^2$ in local DGD. However, local DGT has an extra term $\frac{\rho}{(1-\rho)^2} \cdot \frac{\delta+L}{\mu}$ involving the smoothness parameter $L$. Therefore, for problems with small $\delta$ and large $\rho$, indicating a small degree of heterogeneity and low network connectivity, {the results suggest that local DGD may be more communication-efficient than local DGT. The conjecture is validated numerically in Section~\ref{sec:sim-OLS}.
} 
\begin{remark}
Theorem~\ref{thm3} shows when initializing $\bm{x}^i$s to be zero, the iterates generated by local DGD will be consensual, and converge to the minimum norm solution $\bm{x}^\star$. In fact, based on \cite[Proposition 2]{shamir2023implicit}, the initialization condition can be relaxed to  $\{\boldsymbol x_{0,0}^i\}_{i=1}^m \in {}{row}\left( \{\boldsymbol A_i\}_{i=1}^m\right)$. Similarly, we can conclude from Theorem~\ref{coro1} that local DGT with the same initialization condition finds $\bm{x}^\star$. In conclusion, both local DGD and -DGT induce an implicit bias akin to centralized gradient \cite{gunasekar2018characterizing} for over-parameterized linear regression. 
\end{remark}

{}{
\begin{remark}
Distributed estimation of linear models have been widely studied in early works such as \cite{cattivelli2009diffusion,lopes2008diffusion,kar2012distributed}. These works consider the setting where the sample size is larger than the problem dimension and focus on characterizing the MSE of the estimators computed by the algorithm.
In complement, our result is for the over-parameterized problems
 where the sample size \( N \) is smaller than dimension \( d \), a setting more common in modern applications such as genomics, bioinformatics, and neuroscience \cite{wainwright2019high}. We prove that among all minimizes, 
 both DGT and DGD converge to the one with smallest $\ell_2$ norm without explicit regularization, which can lead to better generalization, as supported by findings in \cite{bartlett2020benign,shamir2022implicit}. The linear model serves as a prototype for exploring more complicated models, such as neural networks in the NTK regime.
\end{remark}
}

\section{Proof of the theorems}\label{proofs}
Recall that $m$ denotes the number of agents and $K$ denotes the number of extra local updates, $\boldsymbol x^i_{r,k} \in \mathbb R^d$ and $\boldsymbol y^i_{r,k} \in \mathbb R^d$ denote optimization variable and gradient tracking variable of agent $i$ at $k$-th local update in $r$-th communication round, respectively.  To write the algorithm in a compact form, we introduce the stacked variables and their averages as follows
\begin{align}
&\boldsymbol X_{r,k} \triangleq [\boldsymbol x^1_{r,k},\cdots,\boldsymbol x^m_{r,k}] \in \mathbb R^{d\times m},\; \overline{\boldsymbol x}_{r,k} \triangleq \frac 1  m \sum_{i=1}^m{\boldsymbol x^i_{r,k}} \in \mathbb R^d; \nonumber \\ & \overline{\boldsymbol x}^{r} \triangleq \overline{\boldsymbol x}_{r,0} \in \mathbb R^d,\;\overline{\boldsymbol X} \triangleq \overline{\boldsymbol x}\boldsymbol 1^T_m \in \mathbb R^{d\times m},\quad \overline{\boldsymbol y}^r = \overline{\boldsymbol y}_{r,0} \in \mathbb R^d ; \nonumber \\
&  \boldsymbol Y_{r,k} \triangleq [\boldsymbol y^1_{r,k}, \cdots, \boldsymbol y^m_{r,k}] \in \mathbb R^{d\times m}, \quad \overline{\boldsymbol y}_{r,k} \triangleq \frac 1 m \sum_{i=1}^m {\boldsymbol y}^i_{r,k};  \nonumber \\
& \nabla \boldsymbol F(\boldsymbol X_{r,k}) \triangleq \left[ \nabla f_1(\boldsymbol x^1_{r,k}),\cdots,\nabla f_m(\boldsymbol x^m_{r,k}) \right] \in \mathbb R^{d\times m}; \nonumber \\
& \nabla \boldsymbol f(\boldsymbol X_{r,k}) \triangleq [\nabla f(\boldsymbol x^1_{r,k}),\cdots,\nabla f(\boldsymbol x^m_{r,k})]\in \mathbb R^{d\times m}. 
\end{align}
With the above notation,  we rewrite the local DGT in \eqref{eq:DGT-comp}, \eqref{stand-dgt} as follows: $\forall r = 0,1,\cdots$ and $\forall k=0,1,\cdots,K-1$
\begin{align}\label{reiter}
 &\boldsymbol X_{r,k+1} = \boldsymbol X_{r,k} - \eta \Big(\boldsymbol Y_{r} + \nabla \boldsymbol F(\boldsymbol X_{r,k}) - \nabla \boldsymbol F(\boldsymbol X_{r})\Big) \nonumber \\
 & \boldsymbol Y_{r,k+1} = \boldsymbol Y_{r} + \nabla \boldsymbol F(\boldsymbol X_{r,k+1}) -  \nabla \boldsymbol F(\boldsymbol X_{r}) \nonumber \\
& \boldsymbol X_{r+1} = \Big(\boldsymbol X_{r,K} - \eta \boldsymbol Y_{r,K}\Big) \boldsymbol W \nonumber \\
 & \boldsymbol Y_{r+1} =  \Big(\boldsymbol Y_{r,K} + \nabla \boldsymbol F(\boldsymbol X_{r+1})  - \nabla \boldsymbol F(\boldsymbol X_{r,K})\Big)\boldsymbol W,
\end{align}
where we use $\boldsymbol X_r, \boldsymbol Y_r$ to denote $\boldsymbol X_{r,0}, \boldsymbol Y_{r,0}$ for simplicity.

Further, we define $\boldsymbol X_{r,K+1} \triangleq \boldsymbol X_{r,K} - \eta \boldsymbol Y_{r,K}$ and $\overline{\boldsymbol x}_{r,K+1} \triangleq  \frac{1}{m} \sum_{i=1}^m{\boldsymbol x^i_{r,K+1}}$. Using the double stochasticity of $\bm{W}$, $\forall k = 0, 1, \cdots, K-1$, we have
\begin{align}\label{eq:avg-local-GT}
     & \overline{\boldsymbol x}_{r,k+1} = \overline{\boldsymbol x}_{r,k} - \frac{\eta}{m} \sum_{i=1}^m{ \left(\nabla f_i(\boldsymbol x_{r,k}^i) - \nabla f_i(\boldsymbol x^i_r) \right)} - \eta  \overline{\boldsymbol y}^r; \nonumber \\
    & \overline{\boldsymbol x}^{r+1} = \overline{\boldsymbol x}_{r,K+1} = \overline{\boldsymbol x}_{r,K} - \eta \overline{\boldsymbol y}_{r,K}; \nonumber \\ 
& \overline{\boldsymbol y}_{r,k+1} = \overline{\boldsymbol y}^r + \frac 1 m \sum_{i=1}^m { \left( \nabla f_i(\boldsymbol x^i_{r,k+1}) - \nabla f_i(\boldsymbol x_r^i) \right) }; \nonumber \\
& \overline{\boldsymbol y}^{r+1} = \overline{\boldsymbol y}^r + \frac{1}{m} \sum_{i=1}^m{ \left ( \nabla f_i(\boldsymbol x^i_{r+1}) - \nabla f_i(\boldsymbol x^i_{r}) \right) }.
\end{align}

\noindent All our proofs are based on the following metrics:
\begin{align}\label{more_definis}
& {s_r^i \triangleq  s^i_{r,0} \; ; \; s^i_{r,k} \triangleq \| \boldsymbol x^i_{r,k} - \overline{\boldsymbol x}_{r,k} \|^2 } \; ; \; {S_r \triangleq S_{r,0} \, ; \, S_{r,k} = \sum_{i=1}^m {s^i_{r,k}}}.\nonumber \\
& \Gamma_r \triangleq \frac{\|\boldsymbol Y_r - \nabla \boldsymbol f(\boldsymbol X_r) \|^2}{m} \; ; \; G_r \triangleq \frac { \sum_{i=1}^m\sum_{k=0}^K{\|\nabla f(\boldsymbol x^i_{r,k})\|^2}} {m(K+1)}; \nonumber \\
&F_r \triangleq f(\overline{\boldsymbol x}^{r}) - f(\boldsymbol x^\star)\; ; \; D_r \triangleq  \sum_{i=1}^m \sum_{k=0}^K \left\| \boldsymbol x^i_{r,k} - \overline{\boldsymbol x}^r \right\|^2.
\end{align}
\subsection{Proof of Theorem \ref{th1} }\label{pr-th1}
The proof outline of Theorem 1 is as follows and it is constructed by proving a series of key lemmas. In specific, we need Lemma \ref{lemma1} to bound the client drift from the average $\|\boldsymbol x^i_{r,k} - \overline{\boldsymbol x}^{r} \|^2$ in $r$-th round due to local updates. We quantify the decrease of the suboptimality gap $F_r$ for total average loss $f$ by one step of local update in Lemma \ref{lemma2}. Combining the two results, we can obtain the contraction of suboptimality gap perturbed by consensus error $\|\boldsymbol X_r - \overline{\boldsymbol X}^r\|^2$ and tracking error $\|\boldsymbol Y_r - \nabla \boldsymbol f(\boldsymbol X_r)\|$ in Lemma \ref{lemma3}. This technique differs from existing analysis \cite{huang2023computation,liu2024decentralized,alghunaim2024local} that directly bounds the reduction of the suboptimality gap over communication rounds. Then we separately bound the consensus and tracking error along the communication round in Lemma \ref{lemma4} and Lemma \ref{lemma6}. Finally, integrating these error bounds in the above lemmas can prove the final result.


The subsequent lemma provides a bound for the drift  $D_r$ accumulated over one round before consensus averaging. 

\begin{lemma}\label{lemma1}
Suppose $f_i$s satisfy Assumption \ref{hessian}, if $\eta \leq \frac{1}{4\sqrt{5}\delta K}$, then for $r$-th round in local DGT, it holds that
\begin{align}\label{finaldrift}
& D_r \leq 6mK \cdot S_r + 80mK^3\eta^2 \cdot \Gamma_r + 80mK^3\eta^2 \cdot G_r. 
\end{align}

\end{lemma}
\begin{proof}
    See Appendix.\ref{pr-lma1}.
\end{proof}

Note that this bound is independent of the smoothness parameter $L$ but relies on the second-order heterogeneity parameter $\delta$, which is a main difference from the existing analysis {}{\cite{huang2023computation,liu2024decentralized,nguyen2023performance,ge2023gradient}}.

Then we show the descent of the average loss $f$ for one local update step, which is upper bounded by the tracking error, consensus error, and client drift.

\begin{lemma}\label{lemma2}
Suppose $f$ satisfies Assumption \ref{sm} and $f_i$s satisfy  Assumption \ref{hessian}, if $\eta \leq \min \left\{\frac 1 L, \frac{1}{30\delta K} \right\}$, then local DGT has for all { $k = 0,\cdots, K$ } that
\begin{align}\label{desl}
& f(\boldsymbol x_{r,k+1}^i) \leq f(\boldsymbol x_{r,k}^i) + \frac {3\eta} 2  \left\| \nabla f(\boldsymbol x_r^i) -\boldsymbol y_r^i  \right\|^2  - \frac{\eta}{2} \left\|\nabla f(\boldsymbol x^i_{r,k}) \right\|^2 \nonumber \\
& \; + \frac{3\eta\delta^2}2 \left\|\overline{\boldsymbol x}^r - \boldsymbol x^i_r \right\|^2  + \frac \delta {10(K+1)} \left\|\boldsymbol x^i_{r,k} - \overline{\boldsymbol x}^r \right\|^2.
\end{align}
\end{lemma}
\begin{proof}
   See Appendix.\ref{pf:lem-2}.
\end{proof}

Combining Lemma \ref{lemma1} and Lemma \ref{lemma2}, we can prove the decrease of optimality gap $f(\overline{\boldsymbol x}^r) - f(\bm{x}^\star)$ along the round.  
\begin{lemma}\label{lemma3}
Under the same setting of Lemma \ref{lemma1} and \ref{lemma2}, in addition, {}{assume further that} Assumption 4 holds  for average loss $f$ and $\eta \leq \min \left\{\frac 1 L, \frac{1}{240\delta K}, \frac{1}{5\mu K} \right\}$, then local DGT has
\begin{align}\label{ddescent}
F_{r+1} &\leq \left(1-\frac{\mu {}{(K+1)} \eta}{8} \right) \cdot F_r + 4K\eta \cdot  \Gamma_r   \nonumber \\
& \quad + \Big(4\delta^2K\eta + 3(L+\delta)\Big) \cdot S_r - \frac{K\eta}{16} \cdot G_r.
\end{align}
\end{lemma}
\begin{proof}
    See Appendix.\ref{pf:lem-3}.
\end{proof}

Next, we establish the dynamic of consensus error $S_r$.
\begin{lemma}\label{lemma4}
Under the same setting of Lemma~\ref{lemma1}, suppose $f_i$s satisfy Assumption \ref{hessian}, the mixing matrix $\boldsymbol W $  satisfies Assumption \ref{W}, and $\eta \leq  \frac{1-\rho}{8\sqrt{5}\delta K}$ in local DGT, it has
\begin{align}\label{dconsensus}
S_{r+1} \leq \left( \frac{1+\rho}{2}\right) \cdot S_r + \frac{90\rho^2K^2\eta^2}{1-\rho} \cdot \left( \Gamma_r + G_r \right).
\end{align}
\end{lemma}
\begin{proof}
    See Appendix.\ref{pf:lem-4}.
\end{proof}

Our last step establishes the recursion of the tracking error term $\Gamma_r$. To this end, the following lemma is necessary.

\begin{lemma}\label{lemma5}
If $f_i$s satisfy Assumption \ref{hessian}, local DGT has
\vspace{-1mm}
\begin{align}
& \left\|\overline{\boldsymbol x}^{r+1} - \overline{\boldsymbol x}^r \right\|^2 \leq \frac{2\delta^2(K+1)\eta^2}{m} \cdot D_r + \frac{2(K+1)^2\eta^2}{m} \cdot G_r.
\end{align}
\end{lemma}
\begin{proof}
    See Appendix.\ref{pf:lem-5}.
\end{proof}

\begin{lemma}\label{lemma6}
Suppose that $f$ satisfies Assumption \ref{sm}, $f_i$s satisfy Assumption \ref{hessian}, and the mixing matrix $\boldsymbol W$  satisfies Assumption \ref{W}. If $\eta \leq \min \Big\{\frac 1 L, \frac{1}{5\mu K}, \frac{1-\rho}{8\sqrt{5}\delta K}, \frac{(1-\rho)^2}{64 \rho\sqrt{\delta^2+L^2}K} \Big\}$, then the tracking error generated by local DGT satisfies
\begin{align}\label{dtracking}
\Gamma_{r+1}  &\leq  \Big(\frac{1+\rho}{2}\Big) \cdot \Gamma_r  + \left( \frac{1188\rho^2(\delta^2+L^2)K^2\eta^2}{(1-\rho)^2}\right) \cdot G_r \nonumber \\
& \quad + \Big(\frac{48\rho(\delta^2+L^2)}{1-\rho}\Big) \cdot S_r.
\end{align}
\end{lemma}
\begin{proof}
    See Appendix.\ref{pf:lem-6}.
\end{proof}

Having introduced the necessary lemmas, we prove  Theorem \ref{th1} now. Firstly, unrolling  \eqref{dconsensus} to $r=0$ gives
\begin{align}
S_r  &\leq  \frac{180\rho^2K^2\eta^2}{1-\rho} \cdot  \sum_{j=0}^{r-1}{\left( \frac{\rho+3}{4}\right)^{r-j} \left(\Gamma_j +  R_j \right) } ,
\end{align}
where we have used the fact that initialization $\boldsymbol x^1_{0,0} = \cdots = \boldsymbol x^m_{0,0}$ implies $S_0 = 0$. Applying $\frac{4}{1-\rho}$-slow increasing sequence [cf. Definition \ref{def:slow-seq}] $z_r$ such that $z_r \leq \Big(1+\frac{1-\rho}{8}\Big)^{r-j}z_j$ we can bound the weighted sum of $S_r$ as follows
\begin{align}\label{wconsen}
\sum_{r=0}^R{z_rS_r} &\leq \frac{180\rho^2K^2}{1-\rho} \sum_{r=0}^R{\eta^2 z_r \sum_{j=0}^{r-1}{\left(\frac{3+\rho}{4}\right)^{r-j} \left(\Gamma_j + R_j\right)}}  \nonumber \\
& \overset{(a)}{\leq} \frac{180\rho^2K^2}{1-\rho} \sum_{r=0}^R{\eta^2 \sum_{j=0}^{r-1}{z_j\left(\frac{7+\rho}{8}\right)^{r-j} \left(\Gamma_j + R_j\right)}}  \nonumber \\
& {}{= \frac{180{}{\eta^2}\rho^2K^2}{1-\rho} \sum_{j=0}^{R-1} z_j(\Gamma_j + R_j) \sum_{r=j+1}^R \left( \frac{7+\rho}{8}\right)^{r-j} } \nonumber \\
& {}{ \leq \frac{180 {}{\eta^2}\rho^2K^2}{1-\rho} \sum_{j=0}^{R-1} z_j(\Gamma_j + R_j)  \sum_{i=0}^{\infty} \left( 1 - \frac{1-\rho}{8}\right)^i} \nonumber \\
& \overset{(b)}{\leq} \frac{1440\rho^2K^2\eta^2}{(1-\rho)^2} \left( \sum_{r=0}^R{z_r\Gamma_r} + \sum_{r=0}^R{z_r G_r} \right) ,
\end{align}
where (a) is due to the property of $\frac{4}{1-\rho}$-slow increasing sequence and $\left(1 -\frac{1-\rho}{4} \right) \left(1 +\frac{1-\rho}{8} \right) \leq 1 -\frac{1-\rho}{8}$. (b) is because $\sum_{r=j+1}^R{\Big(1 - \frac{1-\rho}{8}\Big)^{r-j}} \leq \sum_{i=0}^{\infty}{ \left( 1 - \frac{1-\rho}{8} \right)^i } = \frac{8}{1-\rho}$. 

\noindent Similarly, unrolling \eqref{dtracking} gives
\begin{align}
\Gamma_r & \leq \frac{24\rho(\delta^2+L^2)}{1-\rho} \sum_{j=0}^{r-1}{\left(\frac{3+\rho}{4}\right)^{r-j}S_j} + \left(\frac{1+\rho}{2} \right)^r  \Gamma_0 \nonumber \\ 
& \quad + \frac{2376\rho^2(\delta^2+L^2)K^2\eta^2}{(1-\rho)^2} \sum_{j=0}^{r-1}{\Big( \frac{3+\rho}{4}\Big)^{r-j}{R_j}},
\end{align}
where $\Gamma_0 = \frac 1 m \sum_{i=1}^m { \left\| \nabla f_i(\boldsymbol x^i_{0,0}) - \nabla f(\boldsymbol x^i_{0,0})\right\|^2}$ is the gradient heterogeneity at  initialization. Weighted sum for $\Gamma_r$ by the same $\frac{4}{1-\rho}$-slow increasing sequence $z_r$ as in \eqref{wconsen} yields
\begin{align}\label{wtrack}
 \sum_{r=0}^R {z_r\Gamma_r} & \leq \frac{768\rho(\delta^2+L^2)}{(1-\rho)^2} \sum_{r=0}^R{z_rS_r} +  \sum_{r=0}^R {z_r \left(\frac{1+\rho}{2} \right)^r} \Gamma_0 \nonumber \\
 & \quad + \frac{19008\rho^2(\delta^2+L^2)K^2\eta^2}{(1-\rho)^3} \sum_{r=0}^R{z_r G_r} .
\end{align}
 Substituting \eqref{wtrack} into \eqref{wconsen} and  $\eta \leq \frac{(1-\rho)^2}{1536\rho\sqrt{\delta^2+L^2}K}$  gives
\begin{align}\label{S_r}
\sum_{r=0}^R {z_rS_r} & \lesssim \frac{\rho^2K^2\eta^2}{(1-\rho)^2}\sum_{r=0}^R {z_r \left( \left( 1-\frac{1-\rho}{2} \right)^r\Gamma_0 + G_r \right) }. 
\end{align}
Further substituting the above inequality into  \eqref{wtrack}  gives
\begin{align}\label{T_r}
\sum_{r=0}^R {z_r\Gamma_r} &\lesssim \frac{\rho^2(\delta^2+L^2)K^2\eta^2}{(1-\rho)^4} \sum_{r=0}^R {z_rG_r} \nonumber \\
& \quad + \sum_{r=0}^R {z_r \left( \frac{1+\rho}{2} \right)^r}  \Gamma_0 
\end{align}

\noindent Finally, setting sequence $z_r = \left( 1-\frac{(K+1)\mu\eta}{8} \right)z_{r+1}$ {}{and $\eta \leq \frac{1-\rho}{2K\mu}$ such that $ \left\{z_r \right\}$ is $\frac 4 {1-\rho}$-slow increasing sequence } and conducting weighted sum for \eqref{ddescent}:
\begin{align}\label{com-weight}
& \frac{K\eta}{16} \cdot \frac{\sum_{r=0}^R{z_r G_r}}{\sum_{r=0}^R{z_r}} \leq \frac{\sum_{r=0}^R{\left( \left(1-\frac{{}{(K+1)}\mu\eta}{8} \right) z_rF_r - z_rF_{r+1}\right)}}{\sum_{r=0}^R{z_r}}   \nonumber \\
& \quad + \frac{4(L+\delta)\sum_{r=0}^R{z_rS_r}}{\sum_{r=0}^R{z_r}}  + \frac{4K\eta\sum_{r=0}^R{z_r\Gamma_r}}{\sum_{r=0}^R{z_r}}  \nonumber \\
& \lesssim  \left(  \frac{(L+\delta)\rho^2K^2\eta^2}{(1-\rho)^2} + K\eta \right) \cdot \frac{\sum_{r=0}^R{z_r \left( 1-\frac{1-\rho}{2} \right)^r  }  }{\sum_{r=0}^R {z_r}  } \Gamma_0 \nonumber \\
& \quad + \frac{1}{\sum_{r=0}^R{z_r}} \sum_{r=0}^R{\left( \left(1-\frac{{}{(K+1)}\mu\eta}{8} \right) z_rF_r - z_rF_{r+1}\right)}   \nonumber \\
& \quad +  \left( \frac {\rho^2(L+\delta)K^2\eta^2}{(1-\rho)^2} + \frac{\rho^2(L^2+\delta^2)K^3\eta^3}{(1-\rho)^4} \right) \frac{\sum_{r=0}^R{z_r G_r}}{\sum_{r=0}^R{z_r}},
\end{align}
where the last inequality is due to \eqref{S_r} and \eqref{T_r}.
Setting $\eta \lesssim \min\left\{ \frac{(1-\rho)^2}{\rho^2(L+\delta)K}, \frac{(1-\rho)^2}{\rho\sqrt{\delta^2+L^2}K} \right\}$ to ensure that $\frac{\rho^2(L+\delta)K^2\eta^2}{(1-\rho)^2} \lesssim \frac{K\eta}{64}$ and $\frac{\rho^2(\delta^2+L^2)K^3\eta^3}{(1-\rho)^4} \lesssim \frac{K\eta}{64}$, \eqref{com-weight} can be simplified as
\begin{align}\label{final_express_rate}
\frac{K\eta}{32} \frac{\sum_{r=0}^{R}{z_r G_r}}{\sum_{r=0}^R{z_r}} &\lesssim  \text{\small $\frac{\sum_{r=0}^R \left( \left(1-\frac{{}{(K+1)}\mu\eta}{8} \right) z_rF_r  - z_rF_{r+1}\right) + \Gamma_0 }{\sum_{r=0}^R{z_r}}$}.
\end{align}

Collecting all conditions on $\eta$ that $\eta \lesssim  \min\Big\{ \frac 1 L, \frac{1-\rho}{ K\mu}, \frac{1-\rho}{K\delta}, \frac{(1-\rho)^2}{\rho K(\delta+L)} \Big\}
$ and {applying Lemma \ref{linearcon} by setting the quantities in this lemma as 
\(
    e_t = R_t, 
    r_{T+1} = F_{T+1}, 
    a = \frac{(K+1)\mu}{8}, 
    b = \frac{K}{16}, 
    c = 0,
    A = 0, 
    B =0, 
    r_0 = F_0 + \Gamma_0  
\)
, then can obtain \eqref{final_express_rate}. Based on all conditions on $\eta$, we can set the \(d\) in the Lemma \ref{linearcon} as
\begin{align}
d =2 \left( L + \frac{K\mu}{1-\rho} + \frac{K\delta}{1-\rho} + \frac{\rho K(L+\delta)}{(1-\rho)^2} \right),  \label{set_value-d} 
\end{align} based on the fact that for any positive $x_1, \cdots, x_n$, there is
\begin{align}
 \frac{1}{\sum_{i=1}^n {x_i}} \leq \min \left\{ \frac 1 {x_1}, \cdots, \frac 1 {x_n} \right\}.   
\end{align}
Based on \(d\) in \eqref{set_value-d}, the explicit expression of 
\(\eta\) is
\begin{align}
\eta = \frac{1}{ 2 \left(  L + \frac{K\mu}{1-\rho} + \frac{K\delta}{1-\rho} + \frac{\rho K(L+\delta)}{(1-\rho)^2} \right)  }, \label{formula_eta}
\end{align}
which is based on the step size selection as the second case in the proof of Lemma 15 in \cite{koloskova2020unified}, leading to the convergence rate in \eqref{eq:thm1-rate}.}
To achieve $f(\overline{\boldsymbol x}^{R+1}) - f^\star \leq \epsilon$, the communication rounds the local DGT takes is
$
  R = \Tilde{\mathcal O} \left( \left(\frac{L}{\mu K} + \frac{\delta }{\mu(1-\rho)} + \frac{1}{1-\rho} + \frac{\rho^2(\delta+L)}{\mu(1-\rho)^2} + \frac{\rho\sqrt{\delta^2+L^2}}{\mu(1-\rho)^2} \right )\log \frac{1}{\epsilon} \right). 
$ 
Thus, the complexity is reduced to \eqref{eq:complexity-scvx} based on $\rho^2(\delta+L) \leq \rho(\delta+L)$, $\sqrt{\delta^2+L^2}\leq \delta+L$.

\subsection{Proof of Theorem \ref{coro1} } \label{proof-DGT-PL}
The sketch of Proof of Theorem \ref{th1} could be extended to  Theorem \ref{coro1}. The strong convexity of average loss $f$ is used twice in the proof of Theorem \ref{th1}. One utilizes strong convexity of $f$ to obtain the contraction of suboptimality gap and another is to apply the Jensen's inequality. Thus, we just need to modify the proof for Lemma \ref{lemma3}. Due to the PL condition of $f$, the contraction of suboptimality gap can still be achieved. The main difference between the strongly convex case and weakly convex case is that there is an additional consensus error term, the following lemma bounds this extra consensus error at the last local update. 
\begin{lemma}\label{lemma15}
Suppose $f_i$s satisfy Assumption \ref{hessian}, if $\eta \leq \frac{1}{4\sqrt{5}\delta K}$, then for $r$-th communication round of local DGT, it holds that 
\begin{align}
\frac{1}{m}  \sum_{i=1}^m\| \boldsymbol x^i_{r,K+1}   & - \overline{\boldsymbol x}_{r,K+1} \|^2 \leq  \left( 6 + 48K^2\delta^2\eta^2 \right) \cdot S_r \nonumber \\
& \quad + \left( 336K^2\eta^2 + 640 \delta^2 K^4 \eta^4 \right) \cdot (\Gamma_r+G_r). 
\end{align}
\end{lemma}
\begin{proof}
See Appendix.\ref{pf:lem-15}.
\end{proof}
The following lemma modifies the Lemma \ref{lemma3} for strongly convex $f$ to $\beta$-weakly convex $f$ with PL condition based on Lemma \ref
{weakly-con} and Lemma \ref{lemma15}.
\begin{lemma}\label{lemma16}
Suppose that non-convex average loss $f$ satisfies Assumption \ref{sm}, \ref{wcv}, \ref{PL}, local loss $f_i$s satisfy  Assumption \ref{hessian} and $\eta \lesssim  \left\{ \frac{1}{L}, \frac{1-\rho}{K\mu}, \frac{1-\rho}{K\delta}, \frac{1}{K\beta} \right\}$, then local DGT holds that
\begin{align}
F_{r+1} & \leq \left( 1 -\frac{\mu (K+1)\eta}{8}\right) \cdot F_r + 8K\eta \cdot  \Gamma_r - \frac{\eta K}{16} \cdot G_r \nonumber \\
& \quad \left( 8\delta^2K\eta + 6 \left( L+\delta+\beta \right) \right) \cdot S_r. 
\end{align}
\end{lemma}
\begin{proof}
See Appendix.\ref{pf:lem-16}. 
\end{proof}
We can compare the upper bounds for the decrease of the suboptimality gap $F_{r+1}$ in Lemma \ref{lemma16} with that in Lemma \ref{lemma3}. The main difference is that there is an additional coefficient $\beta$ in front of consensus error $S_r$, which is the weakly convex parameter. Thus, we can apply the same technique in proving Theorem \ref{th1} to finish the proof for Theorem \ref{coro1}. 
\subsection{Proof of Theorem \ref{thm2}}\label{pr-th2}
Recall the algorithm convergence metrics given in \eqref{ols_conv}
and let $\mathcal{X}^\star$ be the set of minimizers of $f$ and $f^\star$ be the corresponding minimum value. 

The proof of Theorem \ref{thm2} involves bounding the consensus error and decrease of suboptimality gap. In specific , Lemma \ref{lemma11} bounds the consensus error $\left \| \boldsymbol X_r - \overline{\boldsymbol X}^r \right \|$. Lemma \ref{lemma12} presents the descent of average loss $f$ by one communication round. Finally, integrating these two bounds to obtain the final result.

\begin{lemma}\label{lemma11}
Consider problem~\eqref{problemformula} in the over-parameterized setting satisfying Assumption~\ref{overpara}. Suppose that $f$ satisfies PL condition as Assumption~\ref{PL} and $f_i$s satisfy Assumption~\ref{rehessian}.  Let $\{ \bm{x}_{r,k}^i\}$ be the sequence generated by the local DGD algorithm under Assumption~\ref{W}, then it holds
\begin{align}\label{dgd-con}
S_{r+1} &\leq  \left( \frac{1+\rho} 2 \right) \cdot  S_r + \frac{64\rho^2K^2(\delta^2+\mu^2)\eta^2}{(1-\rho)\mu^2} \cdot G_r .  
\end{align}
\end{lemma}
\begin{proof}
    See Appendix.\ref{pf:lem-7}.
\end{proof}
Next, we derive the decrease of the suboptimality gap $f(\overline{\boldsymbol x}) - f^\star$ in one communication round. 
\begin{lemma}\label{lemma12}
Under the same setting of Lemma \ref{lemma11} and $\delta<\mu$, if the $\beta$-weak convexity in Assumption~\ref{wcv} holds for average loss $f$ holds, and $\eta \leq \min \left\{\frac 1 L, \frac 1 {2K \mu \zeta}, \frac{(1-\rho)(\mu^2 - \delta^2)}{512K\beta(\delta^2 + \mu^2)} \right\}$, then the iterates {}{generated} by local DGD satisfy
\begin{align}\label{dgd-dess}
F_{r+1} &\leq \left( 1 -\frac{{}{(K+1)}\eta \mu \zeta }{4}\right) F_r + 2(L+\beta)  S_r  - \frac {K\eta\zeta} {8}  G_r.
\end{align}
\end{lemma}
\begin{proof}
    See Appendix.\ref{pf:lem-8}.
\end{proof}

The proof of Theorem \ref{thm2} combines Lemma \ref{lemma11} and Lemma \ref{lemma12}. In specific, unrolling the $S_r$ based on \eqref{dgd-con} we get
\begin{align}
S_r  \leq \frac{64\rho^2(\delta^2+\mu^2)K^2\eta^2}{(1-\rho)\mu^2} \cdot \sum_{j=0}^{r-1} {\Big(1 -\frac{1-\rho}{4} \Big)^{r-j}R_j},
\end{align}
Applying a $\frac{4}{1-\rho}$-slow increasing sequence $z_r$ satisfying  $z_r \leq \left(1+\frac{1-\rho}{8} \right)^{r-j}z_j$ to compute th weighted sum of $S_r$  in the same way as \eqref{wconsen} yields
\begin{align}\label{w-con}
& \sum_{r=0}^R {z_r S_r}  \leq \frac{512\rho^2(\delta^2+\mu^2)K^2\eta^2}{(1-\rho)^2\mu^2} \sum_{j=0}^R{z_jR_j}.
\end{align}

\noindent Applying the same weight sequence $z_r$ to \eqref{dgd-dess} further gives
\begin{align}
&\frac{\zeta K\eta}{8} \frac{\sum_{r}{z_r G_r}}{\sum_{r}{z_r}} \leq \frac{\sum_{r}{\left( \left( 1- \frac{{}{(K+1)}\eta\mu \zeta}{4} \right)z_rF_r -z_r F_{r+1}\right)}} {\sum_r{z_r}} \nonumber \\
& \quad + \frac{2(L+\beta)}{ \sum_r{z_r}} {\sum_r{z_rS_r}}  \nonumber \\
&{\leq } \frac{\sum_{r}{\left( \left( 1- \frac{{}{(K+1)}\eta\mu \zeta}{4} \right)z_rF_r - z_r F_{r+1}\right)}} {\sum_r{z_r}} + \frac{\zeta K\eta}{16} \frac{\sum_{r}{z_r G_r}}{\sum_{r}{z_r}}, 
\end{align}
where the last inequality is based on \eqref{w-con} and $\eta \leq \frac{(1-\rho)^2(\mu^2-\delta^2)}{16384\rho^2K(L+\beta)(\delta^2+\mu^2)}$. Choosing $z_r$ to satisfy $\frac{z_{r-1}}{z_r} = 1-\frac{{}{(K+1)}\eta\mu \zeta}{4}$ and guarantee that $z_r$ satisfies $z_r\leq \left(1+\frac{1-\rho}{8} \right)^{r-j}z_j$, we further require the step size  $\eta \leq \frac{1-\rho}{4K\mu\zeta}$. Finally, telescoping summation gives
\begin{align}
\frac{\sum_{r}{z_r G_r}}{\sum_{r}{z_r}} + \frac{16z_RF_{R+1}}{K\eta \zeta \sum_r{z_r}} \leq \frac{16}{K\eta \zeta \sum_r{z_r}}F_0.
\end{align}
 Tuning $\eta$ in the same way as Theorem \ref{th1} finishes the proof.


\subsection{Proof of Theorem \ref{thm3}}\label{pr-th3}

In this proof, we use $e_{r,k} \triangleq \left\| \overline{\boldsymbol x}_{r,k} - \boldsymbol x^\star \right\|^2$ as the optimization error metric, where  $\boldsymbol x^\star$ is the minimum $\ell_2$-norm solution. Denote for short $ e_r \triangleq e_{r,0}$. The technique of proving Theorem \ref{thm3} is similar to Theorem \ref{thm2}, consisting of three steps: establishing bound for $\left \|\overline{\boldsymbol x}^r - \boldsymbol x^\star \right\|^2$ in Lemma \ref{lemma13}, recursion of consensus error in Lemma \ref{lemma14}. Finally, integrating these two bounds to finish the final proof. A key distinction in  over-parameterized linear regression lies in the exploitation of the linearity of the gradient and the fact that the Hessian matrix is a constant independent of $\bm{x}$.

\begin{lemma}\label{lemma13}
Let $\{\boldsymbol x_{r,k}^i\}$ be generated by the local DGD algorithm for the overparameterized linear regression problem under Assumption \ref{W} with initialization $\boldsymbol x^i_{0,0} = \boldsymbol 0, \forall i=1,\cdots, m$. If  $\eta \leq \frac 1 L$, then  $\forall \epsilon_1 >0$ it holds 
\begin{align}\label{opterr}
e_{r,k+1} &\leq (1-\eta \mu)^2(1+\epsilon_1)e_{r,k} + \eta^2\delta^2\left( 1+\frac{1}{\epsilon_1}\right) S_{r,k},
\end{align}
{}{where the  \( S_{r,k} : = \frac 1 m \sum_{i=1}^m \left\| \overline{\boldsymbol x}_{r,k} - \boldsymbol x^i_{r,k} \right\|^2 \).}
\end{lemma}
\begin{proof}
    See Appendix.\ref{pf:lem-9}.
\end{proof}

\begin{lemma}\label{lemma14}
In the same setting as Lemma \ref{lemma13}, if $\eta \leq \min \left\{ \frac 1 L, \frac{\mu(1-\rho)}{12K\delta^2}\right\}$, then the consensus error has 
\begin{align}\label{recon-2}
\sum_{k=0}^K {}{S_{r,k}} & \leq {}{2(K+1)S_r + \frac{32(K+1)K\delta^2\eta^2}{1-\rho}} \nonumber \sum_{k=0}^K {}{e_{r,k}}, \\
S_{r+1} & \leq \left(\frac{1+\rho}{2}\right) S_r + \frac{32K\delta^2\eta^2}{1-\rho} \sum_{k=0}^K e_{r,k}.
\end{align}

\end{lemma}

\begin{proof}
    See Appendix.\ref{pf:lem-10}.
\end{proof}
Setting $\epsilon_1 = \frac{\eta\mu}{2}$ in \eqref{opterr} and unrolling along the local updates yields $\forall k = 0, 1, \cdots, K$
\begin{align} \label{onestepopt}
e_{r,k} \leq e_r + \left( \eta^2 \delta^2 + \eta \frac{2\delta^2}{\mu} \right) \sum_{i=0}^K S_{r,i}.
\end{align}
Summing  from $k=0$ to $K$ and substituting \eqref{onestepopt} into \eqref{recon-2}:
\begin{align}\label{ols_consensus_error}
S_{r+1} &\leq \left(\frac{1+\rho}{2} \right) \cdot S_r  + \frac{128K^3\eta^3\delta^3\left(\eta+\frac{2\delta}{\mu}\right)}{1-\rho} \cdot \frac{\sum_{i=0}^K S_{r,i}}{K+1}\nonumber \\
& \quad + \frac{64K^2\delta^2\eta^2}{1-\rho} \cdot  e_r
\end{align}
Then unrolling \eqref{onestepopt} until $k=K+1$ has
\begin{align}\label{ols_opt_error}
&  e_{r+1}  =e_{r,K+1} \leq \left( 1- \frac{\eta\mu}{2} \right)^{K+1} e_r + \eta\delta^2\left( \eta + \frac 2 \mu \right) \sum_{i=0}^K { S_{r,i} } \nonumber \\
& \quad \leq \left( 1 -\frac{(K+1)\eta\mu}{4} \right) e_r + 2K\delta^2\eta\left(\eta+\frac 2 \mu \right)\frac{\sum_{i=0}^K S_{r,i}}{K+1},
\end{align}
{}{
\noindent where the last line is due to the Bernoulli inequality \((1+x)^r \geq 1+rx\), for \(x \geq -1\) and \(r \in \mathbb{R} \setminus (0,1)\). Here, \(r = K+1\) and \(x = -\frac{\eta\mu}{2} \geq -1\) due to \(\eta \leq \frac{1-\rho}{K\mu}\). Based on the Bernoulli inequality, we have \(\frac{1}{(1-\frac{\eta\mu}{2})^{-(K+1)}} \leq \frac{1}{1+\frac{\eta\mu(K+1)}{2}} \leq 1-\frac{(K+1)\eta\mu}{4}\), where the last inequality is because \(\eta \leq \frac{1-\rho}{K\mu}\).
}
{}{Then substituting \eqref{onestepopt} into the first inequality in Lemma \ref{lemma14} gives}
\begin{align}
\sum_{k=0}^K S_{r,k} \leq 4 S_r + \frac{128K^2\delta^2\eta^2}{1-\rho} e_r.
\end{align}

\noindent Combining \eqref{ols_consensus_error} and \eqref{ols_opt_error} gives 
\begin{align}
& e_{r+1}  \leq  \left( 1 -\frac{(K+1)\eta\mu}{4} \right)  e_r + 8K\eta\frac{\delta^2}{\mu}  \frac{\sum_{i=0}^K {   S_{r,i}  } }{K+1},  \label{ols_fina_1}\\
& S_{r+1} \leq \left(\frac{1+\rho}{2} \right)  S_r + \frac{64K^2\delta^2\eta^2}{1-\rho} \left( e_r + \frac{8K\delta^2\eta \sum_{i=0}^K S_{r,i}}{\mu(K+1)} \right),  \label{ols_fina_2} \\
& \frac{\sum_{i=0}^K S_{r,i} }{K+1} \leq 4 S_r + \frac{128K^2\delta^2\eta^2}{1-\rho} e_r,  \label{ols_fina_3}
\end{align}
\noindent where \eqref{ols_fina_1} is based on \eqref{ols_opt_error} and $\eta \leq \frac{1-\rho}{K\mu} \leq \frac{2}{\mu}$; and \eqref{ols_fina_2} is based on \eqref{ols_opt_error}. Substituting \eqref{ols_fina_3} into \eqref{ols_fina_1} and \eqref{ols_fina_2} has
\begin{align}
& e_{r+1}  \leq \left( 1 - \frac{(K+1)\mu\eta}{4} + \frac{1024K^3\delta^4\eta^3}{\mu(1-\rho)} \right) e_r + 32K \frac{\delta^2}\mu \eta S_r, \label{ols-finals_1}\\
 & S_{r+1} \leq \left( \frac{1+\rho}{2} + \frac{2048K^3\delta^4\eta^3}{\mu(1-\rho)} \right)  S_r  + \frac{64K^2\delta^2\eta^2}{1-\rho} \nonumber \\
 & \quad \quad \quad \quad \cdot \left( 1 + \frac{K\delta^2\eta}{\mu} \right) e_r, \label{ols-finals-2}
\end{align}
where the last inequality is due to $\eta \leq \min \left \{\frac{1-\rho}{32K\mu}, \frac{\mu(1-\rho)}{32K\delta^2} \right\}$ such that $\frac{128K^2\delta^2\eta^2}{1-\rho} \leq \frac{1}{8}$. 

Finally, summing \eqref{ols-finals_1}  and \eqref{ols-finals-2} gives
\begin{align}\label{ols_final_conver}
e_{r+1}  + S_{r+1}  &\leq \text{\small $\left(  1 - \frac{(K+1)\mu\eta}{4} + \frac{64K^2\delta^2\eta^2}{1-\rho}  + \frac{1088K^3\delta^4\eta^3}{\mu(1-\rho)}    \right)$ }\nonumber \\
& \quad \cdot e_r  + \left(\frac{1+\rho}{2} + 32K \eta \frac{\delta^2}\mu  + \frac{2048K^3\delta^4\eta^3}{\mu(1-\rho)} \right)  S_r \nonumber  \\
& {\leq } \left( 1 - \frac{(K+1)\mu\eta}{8} \right)  \left(  e_r + S_r \right),
\end{align}
where the last inequality is due to $\eta \leq \min\left\{ \frac{\mu(1-\rho)}{1024K\delta^2}, \frac{1-\rho}{32K\mu}  \right\}$ that $\max \left\{\frac{64K^2\delta^2\eta^2}{1-\rho}, \frac{1088K^3\delta^4\eta^3}{\mu(1-\rho)}\right\}  \leq \frac{(K+1)\mu\eta}{16}$, $\max\left\{\frac{2048K^3\delta^4\eta^3}{\mu(1-\rho)}, 32K\frac{\delta^2}{\mu}\eta \right\} \leq \frac{1-\rho}{8}$ and $\frac{(K+1)\mu\eta}{8} \leq \frac{1-\rho}{4}$. The linear convergence rate {}{follows directly, given} the step size condition $\eta \lesssim \min\left\{\frac 1 L, \frac{1-\rho}{K\mu}, \frac{\mu(1-\rho)}{K\delta^2}\right\}$.

\section{Numerical experiments}\label{num_res}
This section validates our theoretical results by experiments. We test the performance of local DGT on the distributed ridge logistic regression (DRLR) problem, which is strongly convex. The results on synthetic data and real-world datasets are displayed in Section \ref{sec:sim-syn-logit}, Section \ref{sec:sim-rea-logit}, and Section \ref{nn_test}. In Section \ref{sec:sim-OLS}, we simulate the over-parameterized linear regression for local DGD and DGT. For all methods with different numbers of local updates, step size $\eta$ is chosen to be the one that yields the fastest linear rate. 

\subsection{Local DGT for DRLR on synthetic data}\label{sec:sim-syn-logit}
We conduct ridge logistic regression for local DGT with each agent $i$'s local loss being  $f_i(\boldsymbol x) = \sum_{j=1}^n {\ln\left(1+\exp \left( -y_{i,j}\boldsymbol a_{i,j}^T \boldsymbol x \right)\right)} + \mu \|\boldsymbol x\|^2$, where $\boldsymbol a_{i,j} \in \mathbb R^d$ and $y_{i,j} \in \{-1,+1\}$ represent the $j$-th sample feature and the corresponding label at agent $i$, respectively. We devise two experimental scenarios to illustrate the influence of network connectivity and heterogeneity on the local updates. The optimality measure is chosen as $\|\nabla f(\overline{\boldsymbol x}^r)\|$ and the regularization parameter $\mu$ is set as $10^{-4}$ for synthetic data. We set the number of agents $m = 20$ and split the dataset uniformly at random to all of the agents.

\textbf{Influence of network connectivity}. 
\emph{Data generation.} Each agent has $n = 1000$ local samples and the model
dimension is set as $d = 5$. Local model parameters are generated as $\boldsymbol x_i = \boldsymbol x_b + \boldsymbol v_i, \forall i \in [m]$, where $\boldsymbol x_b \sim \mathcal N(\boldsymbol 0, \boldsymbol I_d)$ and $\boldsymbol v_i \sim \mathcal N(\bm{0},\boldsymbol I_d)$ are normal random vectors. The local  feature vectors for each agent $i$  are i.i.d. generated as $\boldsymbol a_{i,j} \sim \mathcal N(0.2 i \times \mathbf{1}_d,0.55\times \boldsymbol I_d), \forall i \in [m]$. Note that the expectation of the normal distributions depends on $i$, thus yielding heterogeneity in the features across the agents. Labels are generated as $s_{i,j} \sim 1+ \mathcal U(0,1)$ with $\mathcal U(0,1)$ being the standard uniform distribution, then if $s_{i,j}\leq 1+\exp\left( - \boldsymbol a_{i,j}^T \boldsymbol x_i\right)$, we set $y_{i,j} = 1$; otherwise, $y_{i,j} = -1$.

\emph{Network setting.} We simulate local DGT over three networks having $m = 20$ agents but with different connectivities. Each network is generated from random Erd\H{o}s R\'enyi (ER) graphs, where each edge between two distinct agents is included independently with probability $p$. We set edge activation probabilities as  $p = 0.04, 0.7, 1$ with corresponding network connectivity $\rho = 0.9664, 0.5946, 0$ (fully connected). The results are shown in the first row of Fig.\ref{fig1}, which indicates that when network connectivity is good, increasing the number of local updates can significantly reduce the communication cost. On the contrary, poorly connected networks result in useless local updates.

\begin{figure}[htbp]
\centering
\subfigure[$\rho=0.9664$]{
\begin{minipage}[t]{0.32\linewidth}
\centering
\includegraphics[scale=0.2]{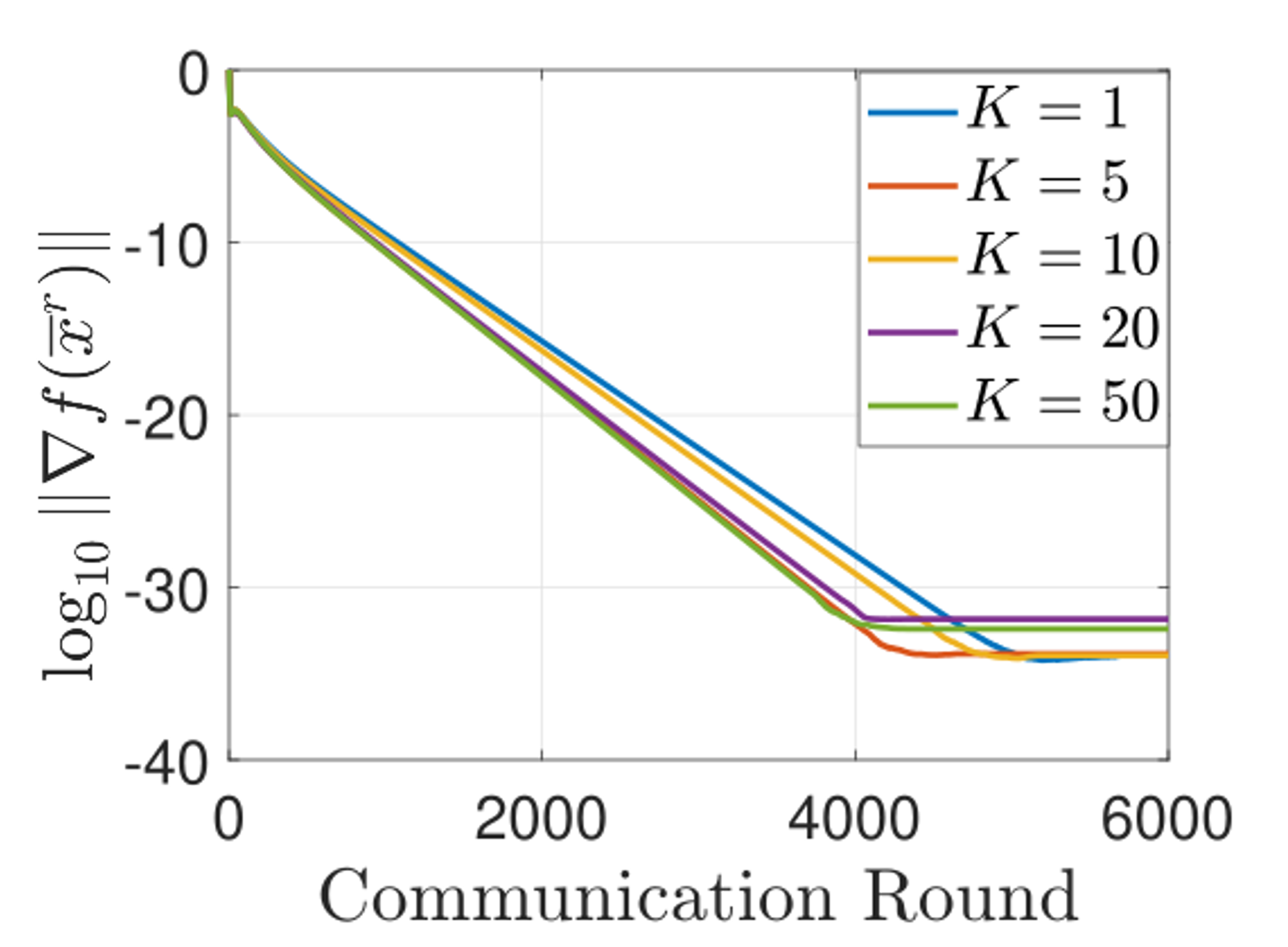}
\end{minipage}%
}%
\subfigure[$\rho = 0.5946$]{
\begin{minipage}[t]{0.32\linewidth}
\centering
\includegraphics[scale=0.2]{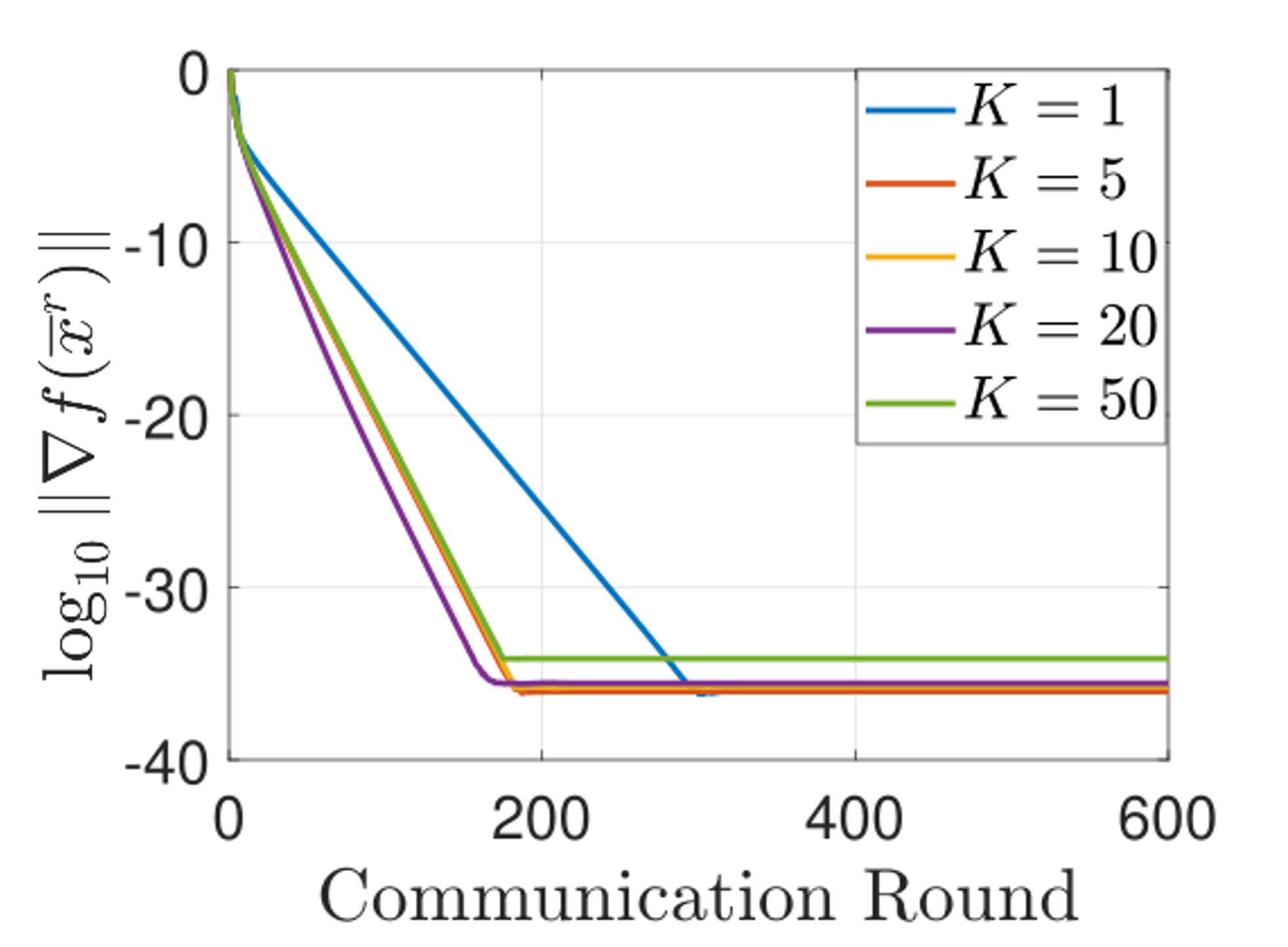}
\end{minipage}%
}%
\subfigure[$\rho=0$]{
\begin{minipage}[t]{0.32\linewidth}
\centering
\includegraphics[scale=0.2]{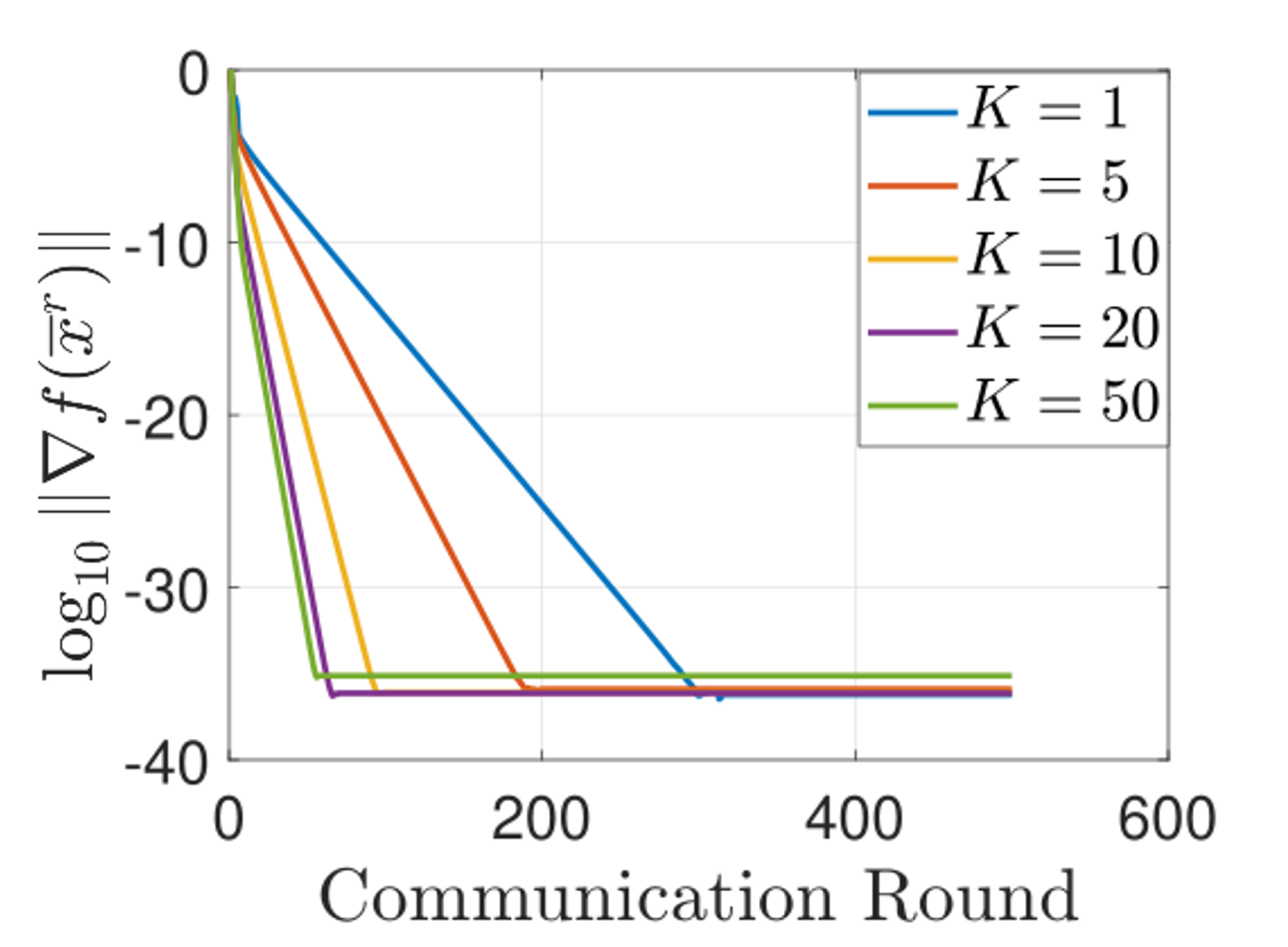}
\end{minipage}%
}%

\subfigure[Low]{
\begin{minipage}[t]{0.3\linewidth}
\centering
\includegraphics[scale=0.2]{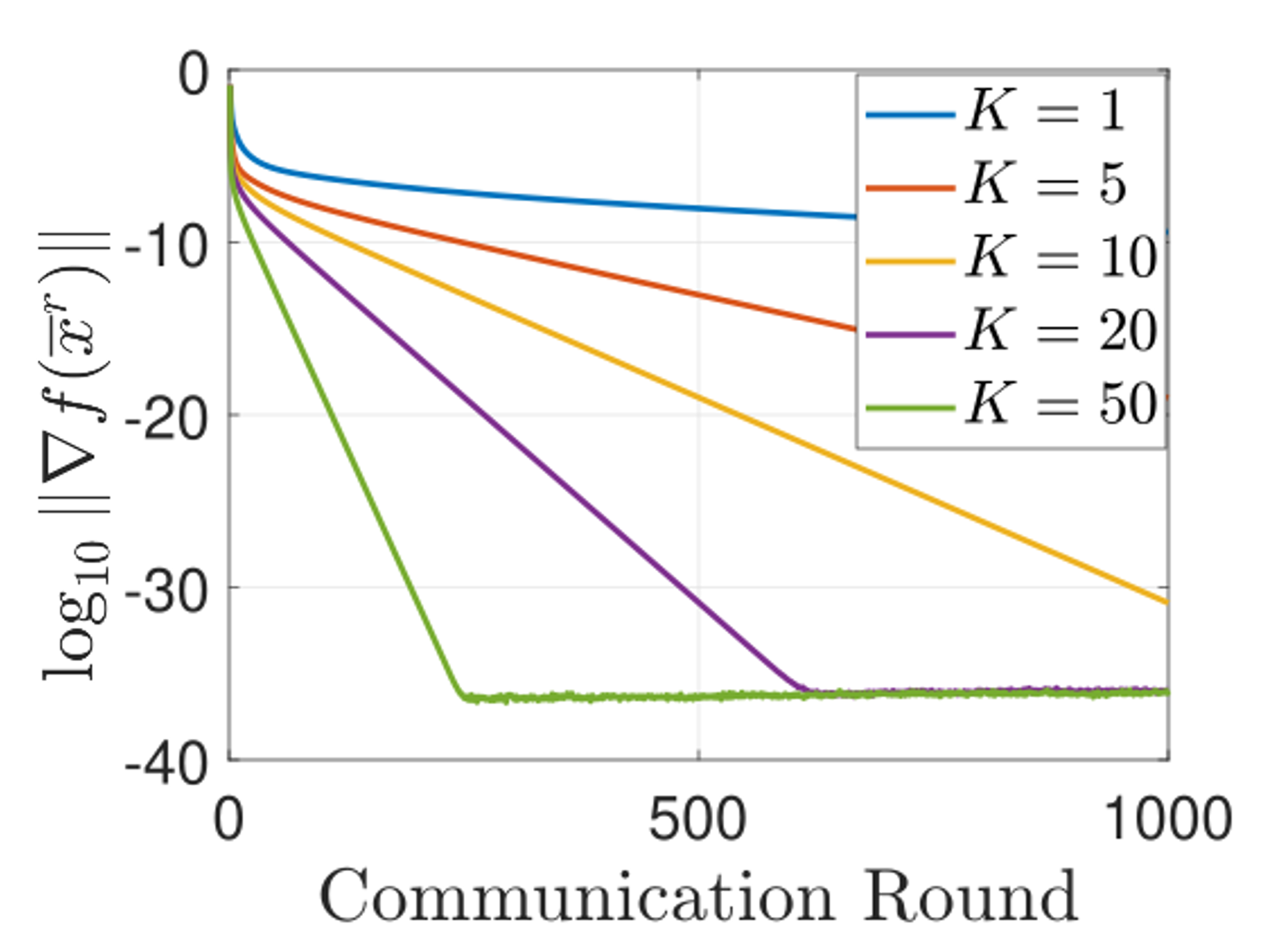}
\end{minipage}
}%
\subfigure[Moderate]{
\begin{minipage}[t]{0.3\linewidth}
\centering
\includegraphics[scale=0.2]{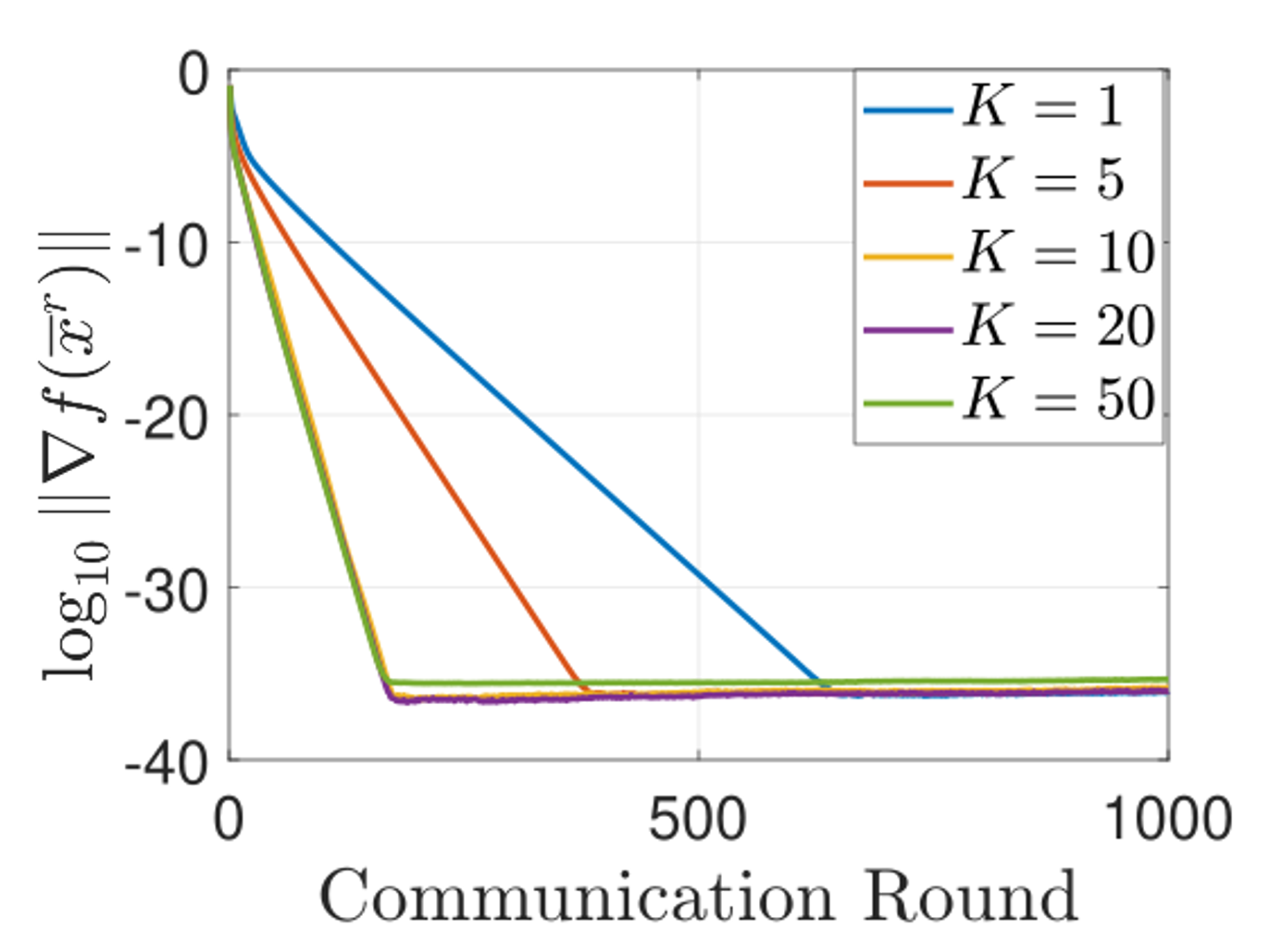}
\end{minipage}
}%
\subfigure[High]{
\begin{minipage}[t]{0.3\linewidth}
\centering
\includegraphics[scale=0.2]{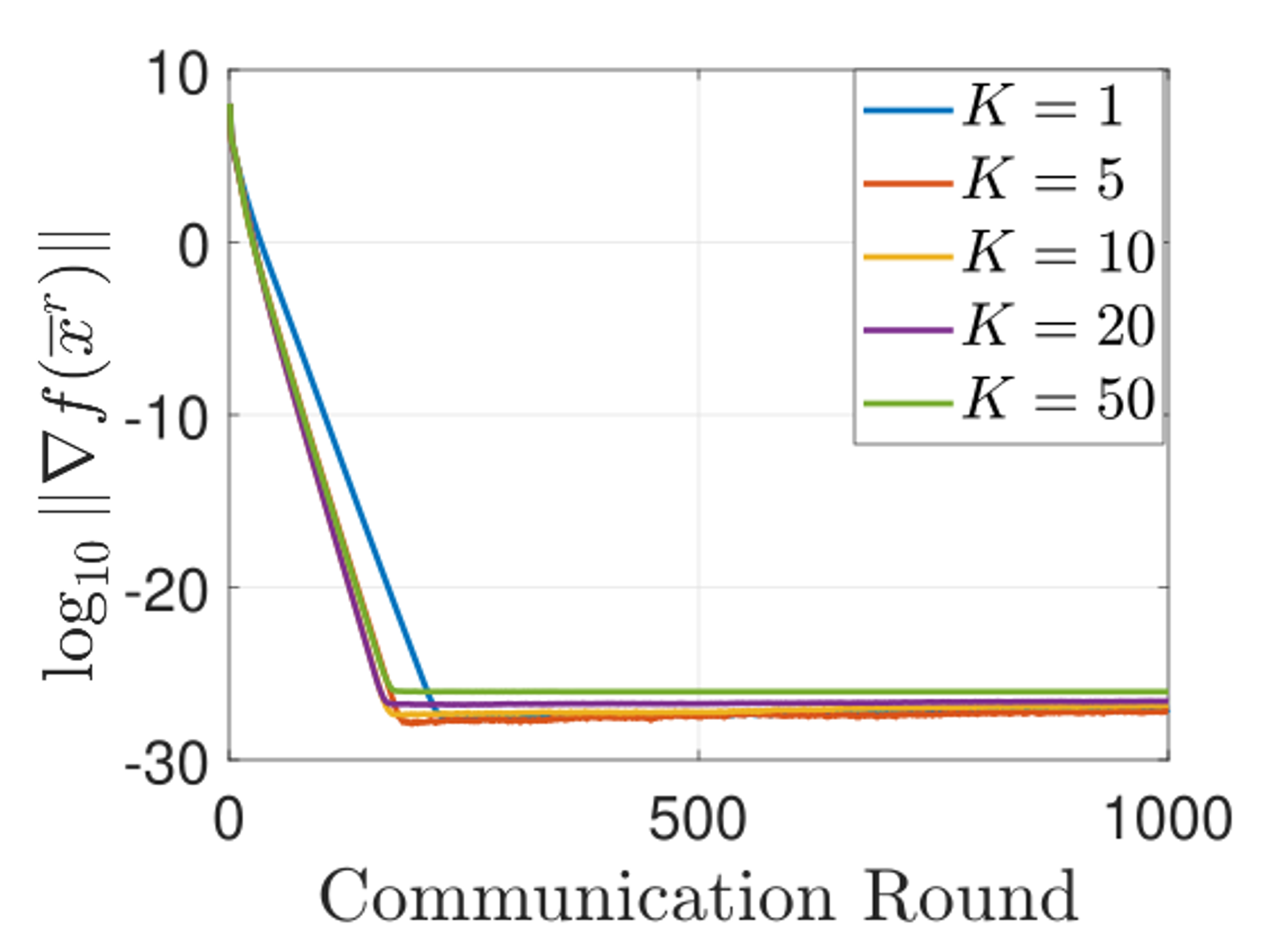}
\end{minipage}
}%
\centering
\caption{Local DGT applied to the ridge logistic regression. First row: influence of network connectivity. Second row: influence of  heterogeneity degree.}
\label{fig1}
\end{figure}

\textbf{Influence of heterogeneity.} {In this experiment, we fix a communication network with $\rho = 0.5946$ and run local DGT on three distinct datasets, with the goal of testing the influence of heterogeneity. These three groups of data are generated with varying degrees of heterogeneity.}

\emph{Data generalization}. The local sample size is $n=100$ and the model dimension is $d=80$. All three groups adopt the same generation approach that $\boldsymbol x_i = \boldsymbol x_b + \boldsymbol v_i, \forall i \in [m]$, where $\boldsymbol x_b \sim \mathcal N(\boldsymbol 0, \boldsymbol I_d)$ and $\boldsymbol v_i \sim \mathcal N(0,\boldsymbol I_d)$ are normal random vectors. For the $p$-th group ($p=1, 2, 3$), the samples are generated as follows to control the heterogeneity degree. For the first agent, its local features are i.i.d. generated as $\boldsymbol a_{1,j} \sim \mathcal{N} (\mathbf{0}, \bm{I}_d)$. Then for the $i$-th agent, we generate its feature as $\bm{a}_{i,j} = \bm{a}_{1,j} + \delta_p \cdot \mathcal{N} (\bm{0}, \bm{I}_d)$, $\forall i = 2, \cdots, m$.  The first group has low heterogeneity ($\delta_1 = 0.99$), the second group has moderate heterogeneity ($\delta_2 = 4$), and the last group exhibits high heterogeneity ($\delta_3 = 10^4$). We generate the label $y_{i,j}$ in the same way as the previous setting.

\emph{Network setting.} The three groups of data are tested on the same network, which is an Erd\H{o}s R\'{e}nyi (ER) graph with $\rho = 0.5946$. The results are given in the second row of Fig. \ref{fig1}.
(d), (e), and (f) show datasets with low, moderate, and high degrees of heterogeneity, respectively. These results show that less heterogeneity can improve the effectiveness of local updates by increasing $K$.

\begin{figure}[htbp]
\centering
\subfigure[$\rho=0.8500$]{
\begin{minipage}[t]{0.3\linewidth}
\centering
\includegraphics[scale=0.19]{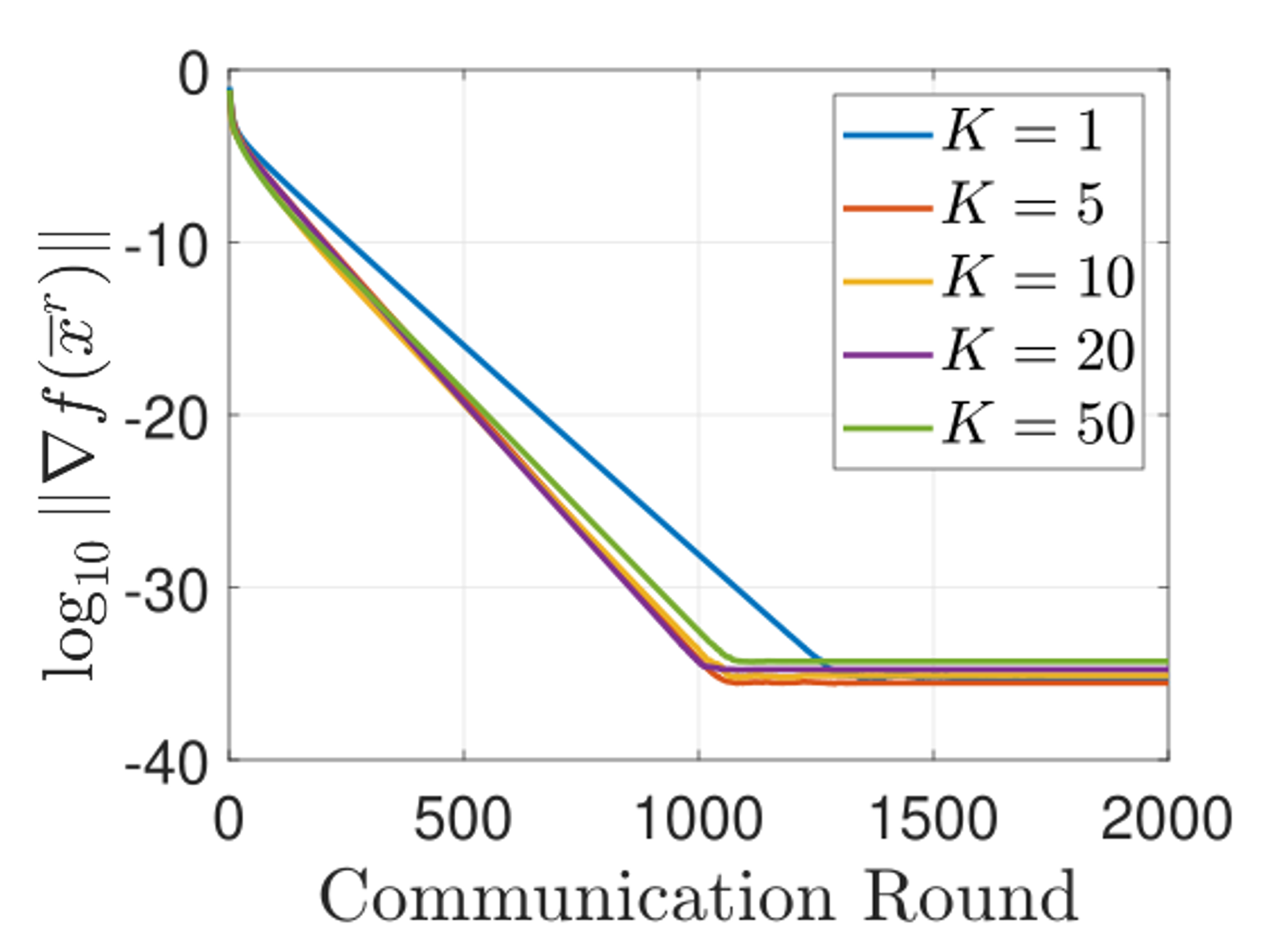}
\end{minipage}%
}%
\subfigure[$\rho=0.3164$]{
\begin{minipage}[t]{0.3\linewidth}
\centering
\includegraphics[scale=0.18]{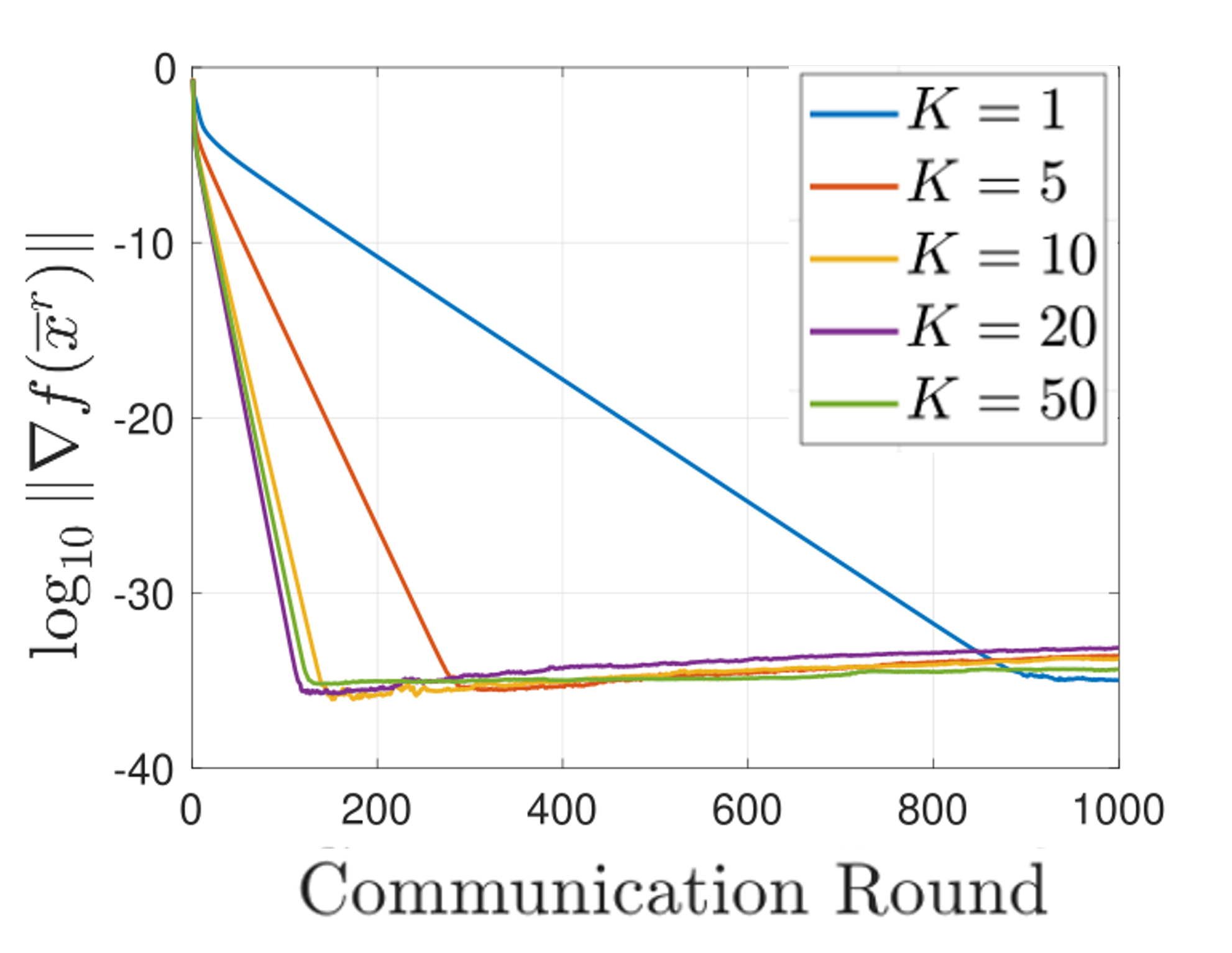}
\end{minipage}%
}%
\subfigure[$\rho=0$]{
\begin{minipage}[t]{0.3\linewidth}
\centering
\includegraphics[scale=0.19]{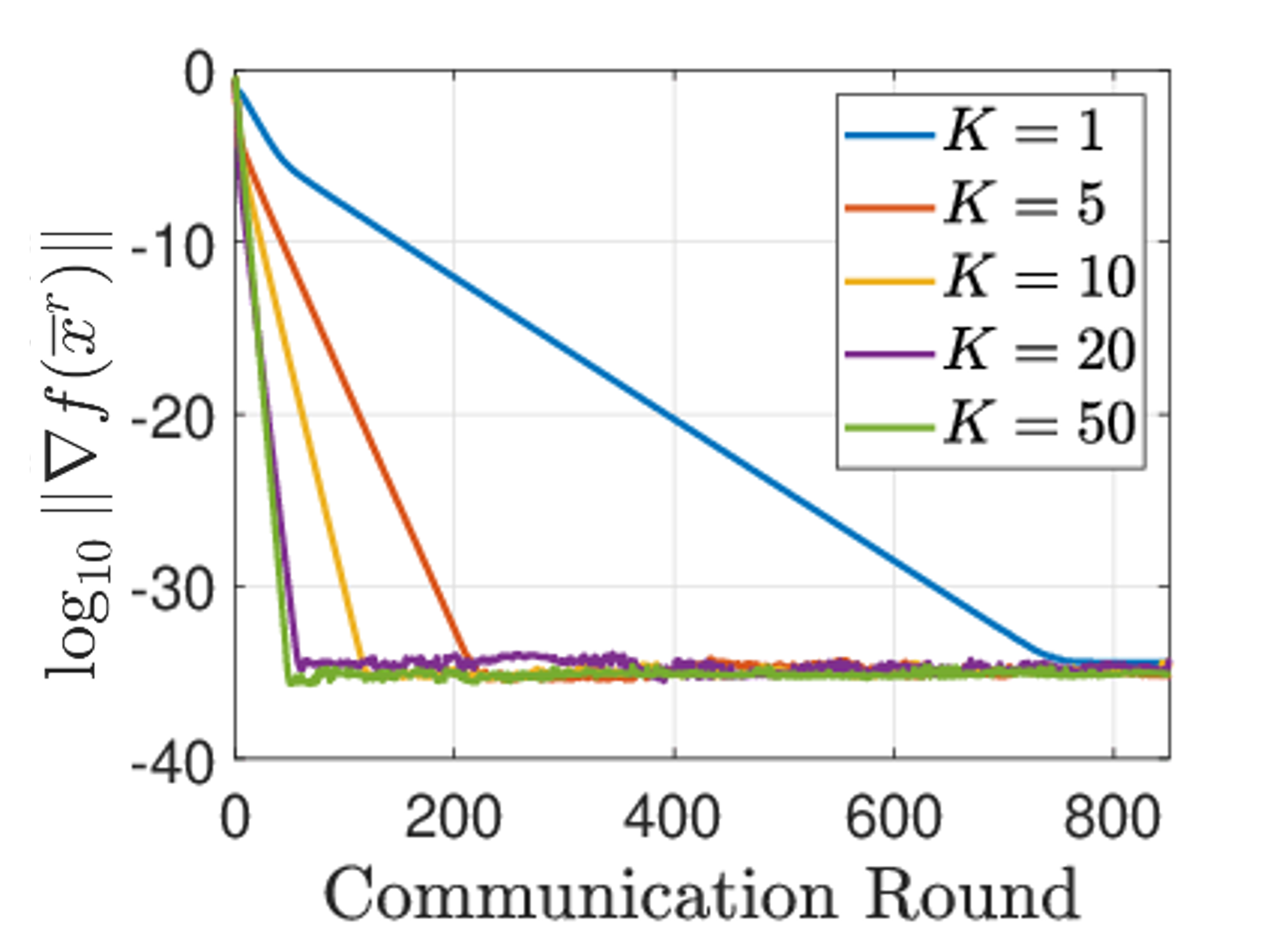}
\end{minipage}%
}

\subfigure[$\rho=0.8500$]{
\begin{minipage}[t]{0.3\linewidth}
\centering
\includegraphics[scale=0.19]{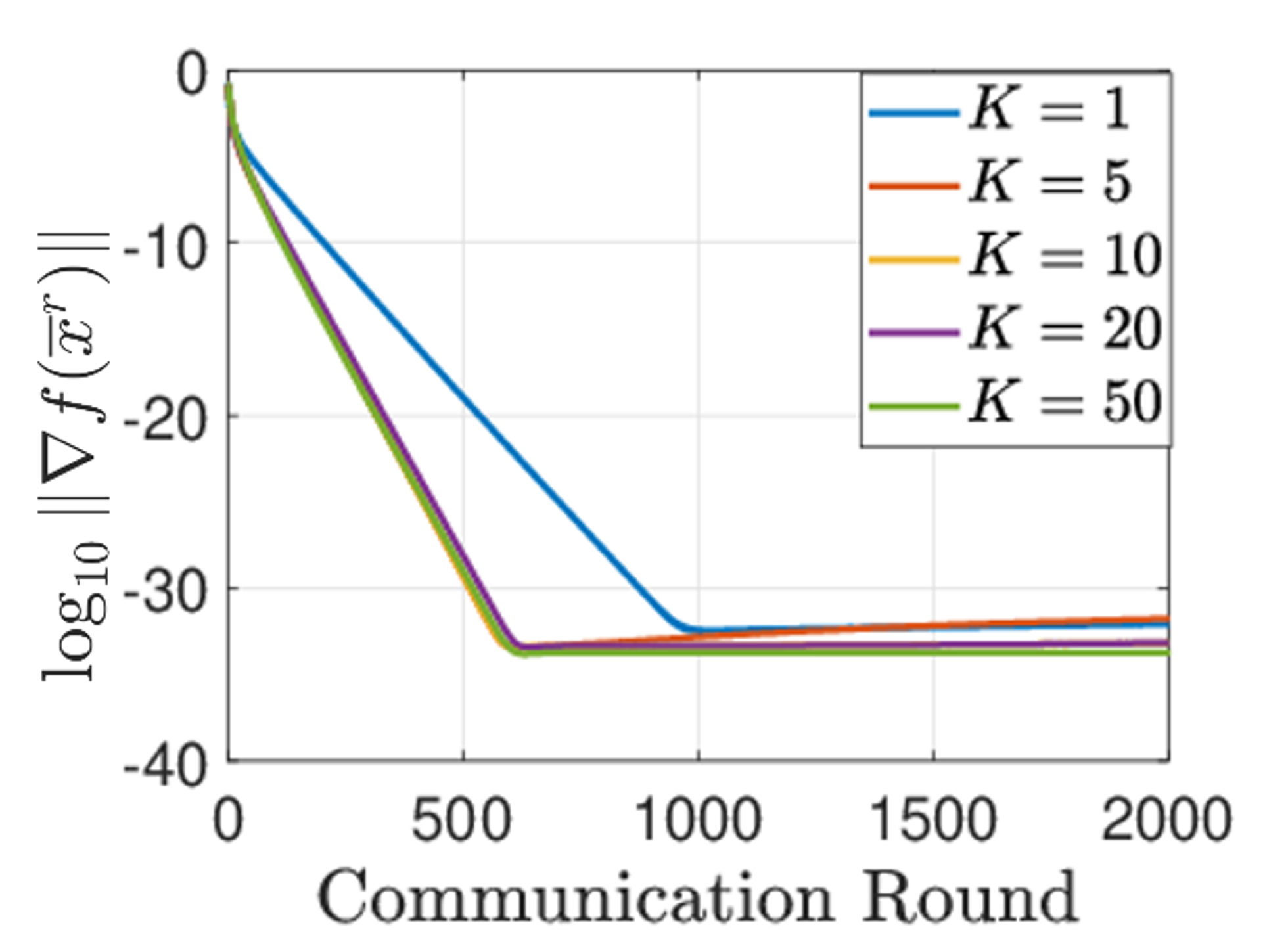}
\end{minipage}%
}%
\subfigure[$\rho=0.3164$]{
\begin{minipage}[t]{0.3\linewidth}
\centering
\includegraphics[scale=0.18]{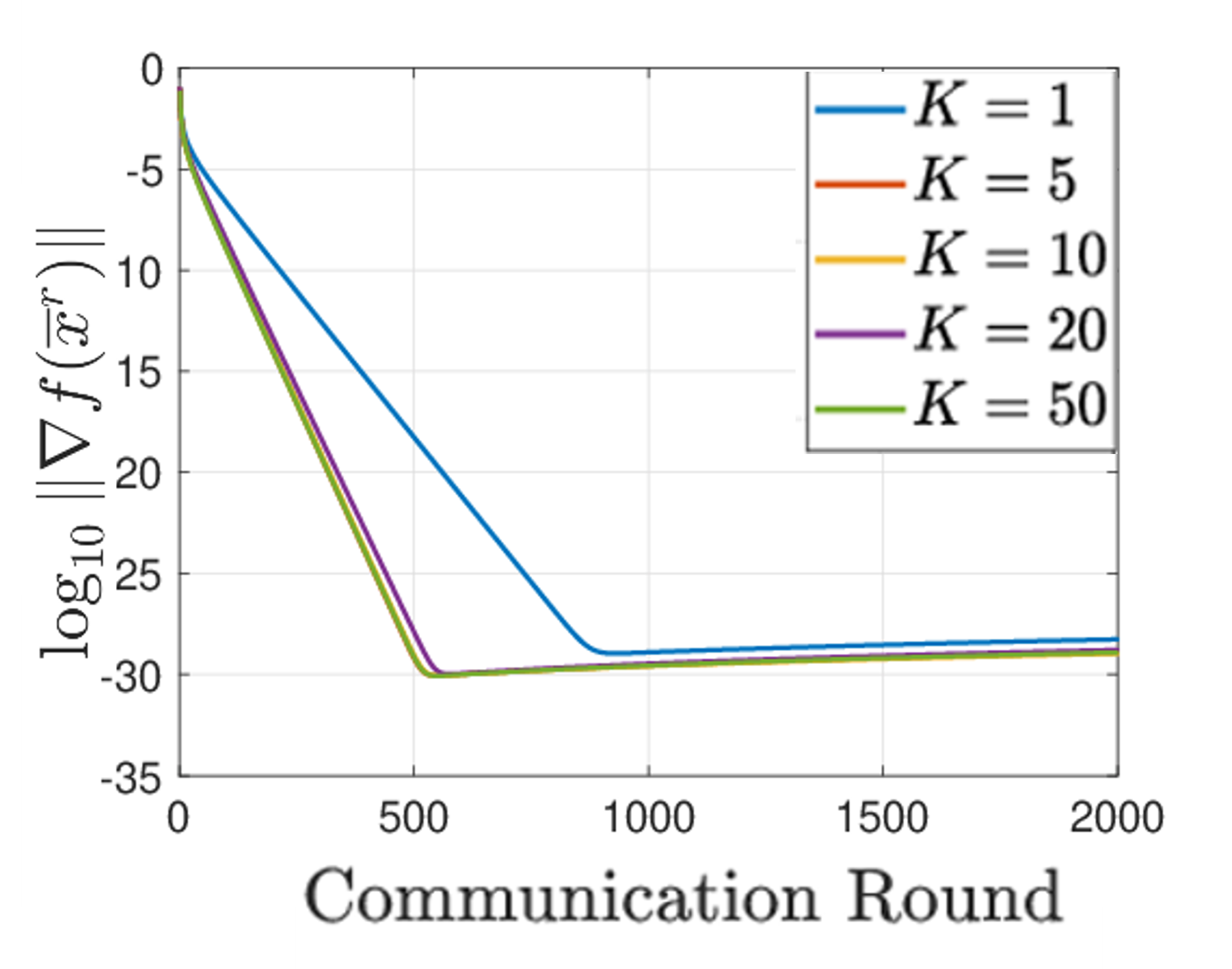}
\end{minipage}%
}%
\subfigure[$\rho=0$]{
\begin{minipage}[t]{0.3\linewidth}
\centering
\includegraphics[scale=0.19]{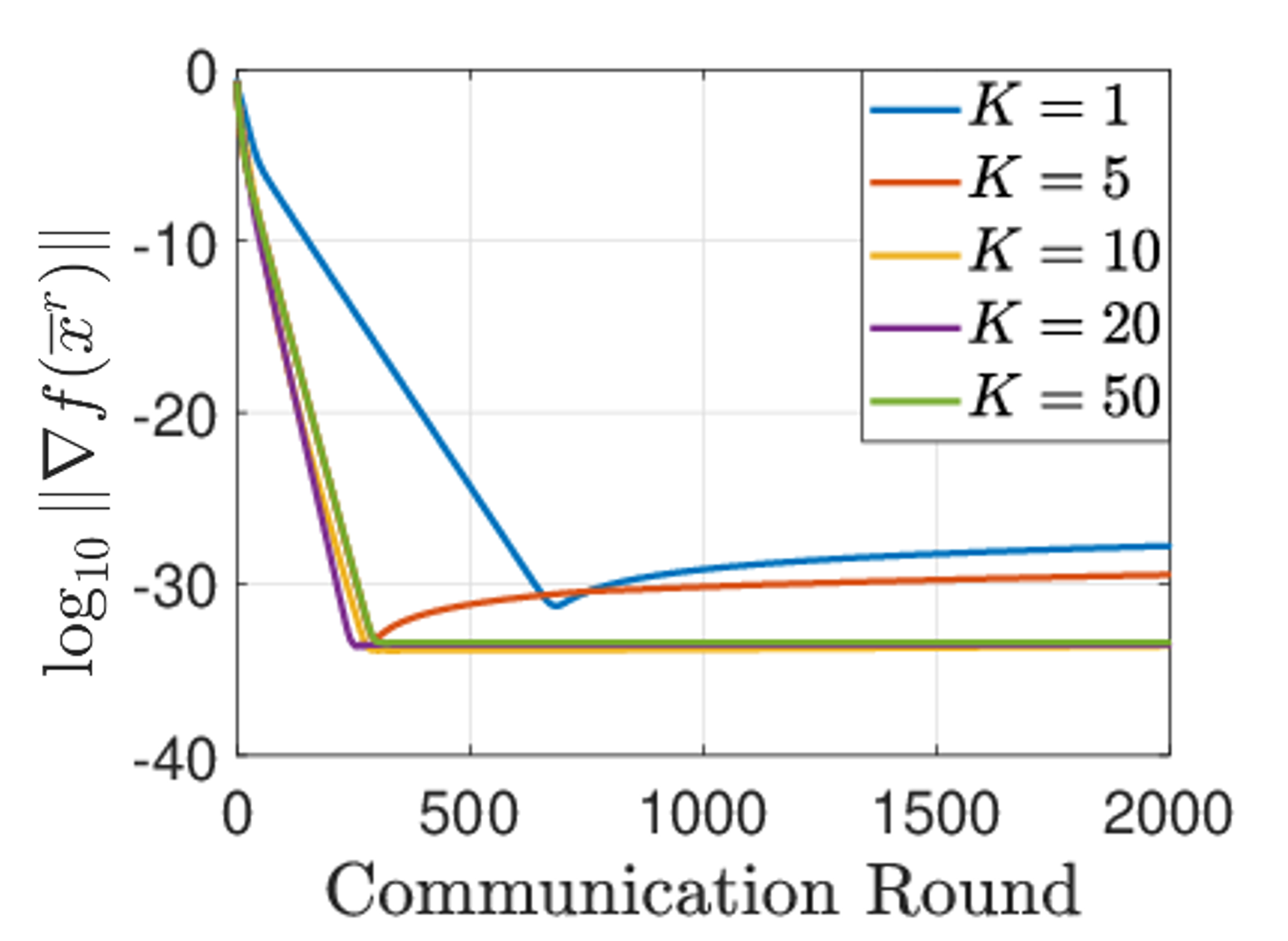}
\end{minipage}%
}%
\caption{ {}{Convergence for local DGT (top) and LED (bottom) for solving DRLR on \enquote{a9a} dataset under moderate heterogeneity degree. The values of $\rho$ are set $\rho=0.85, 0.3164, 0$. } }
\label{fig2}
\end{figure}

\begin{figure}[htbp]
\centering
\subfigure[$\rho=0.8500$]{
\begin{minipage}[t]{0.3\linewidth}
\centering
\includegraphics[scale=0.2]{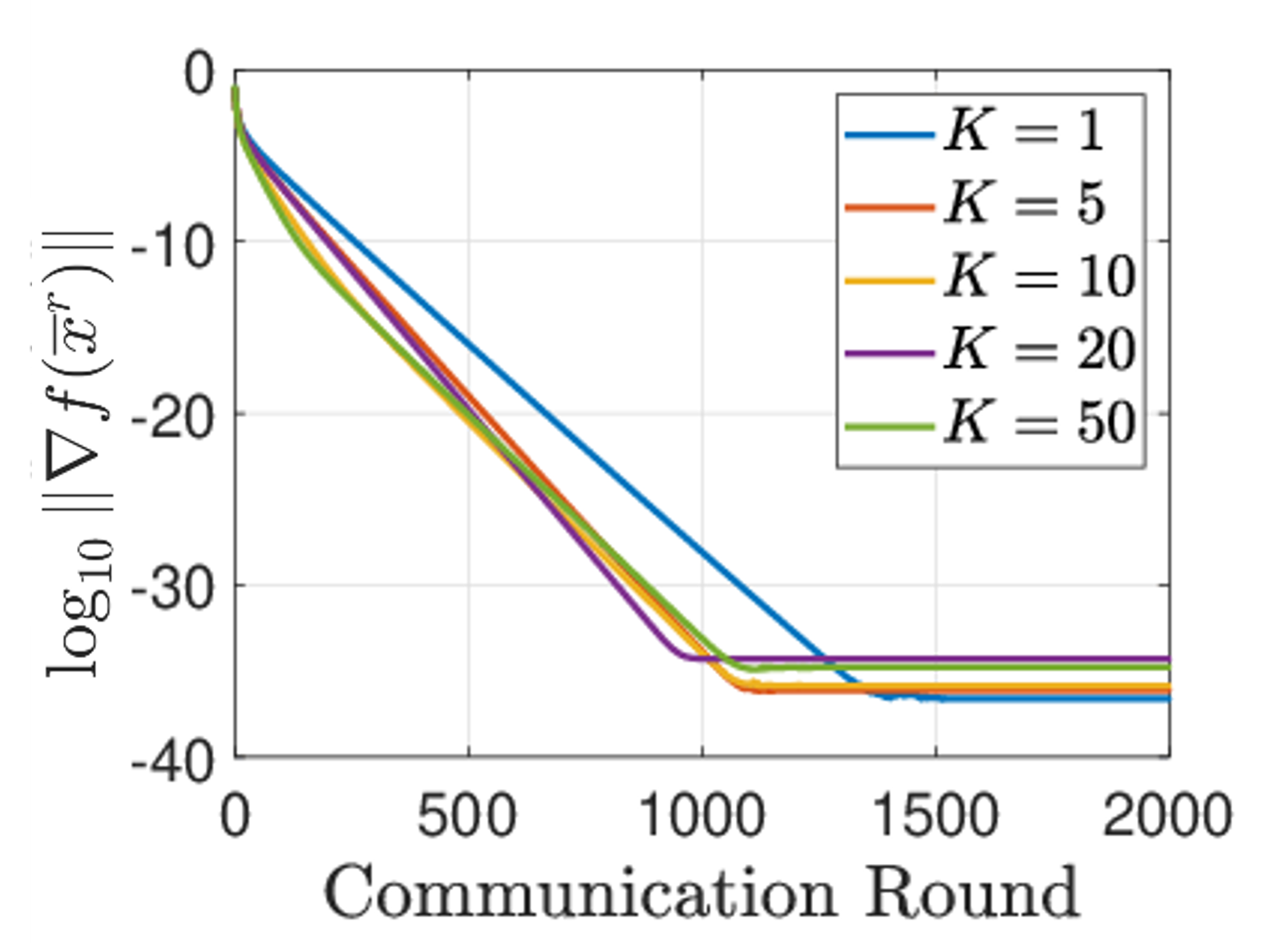}
\end{minipage}%
}%
\subfigure[$\rho=0.3164$]{
\begin{minipage}[t]{0.3\linewidth}
\centering
\includegraphics[scale=0.18]{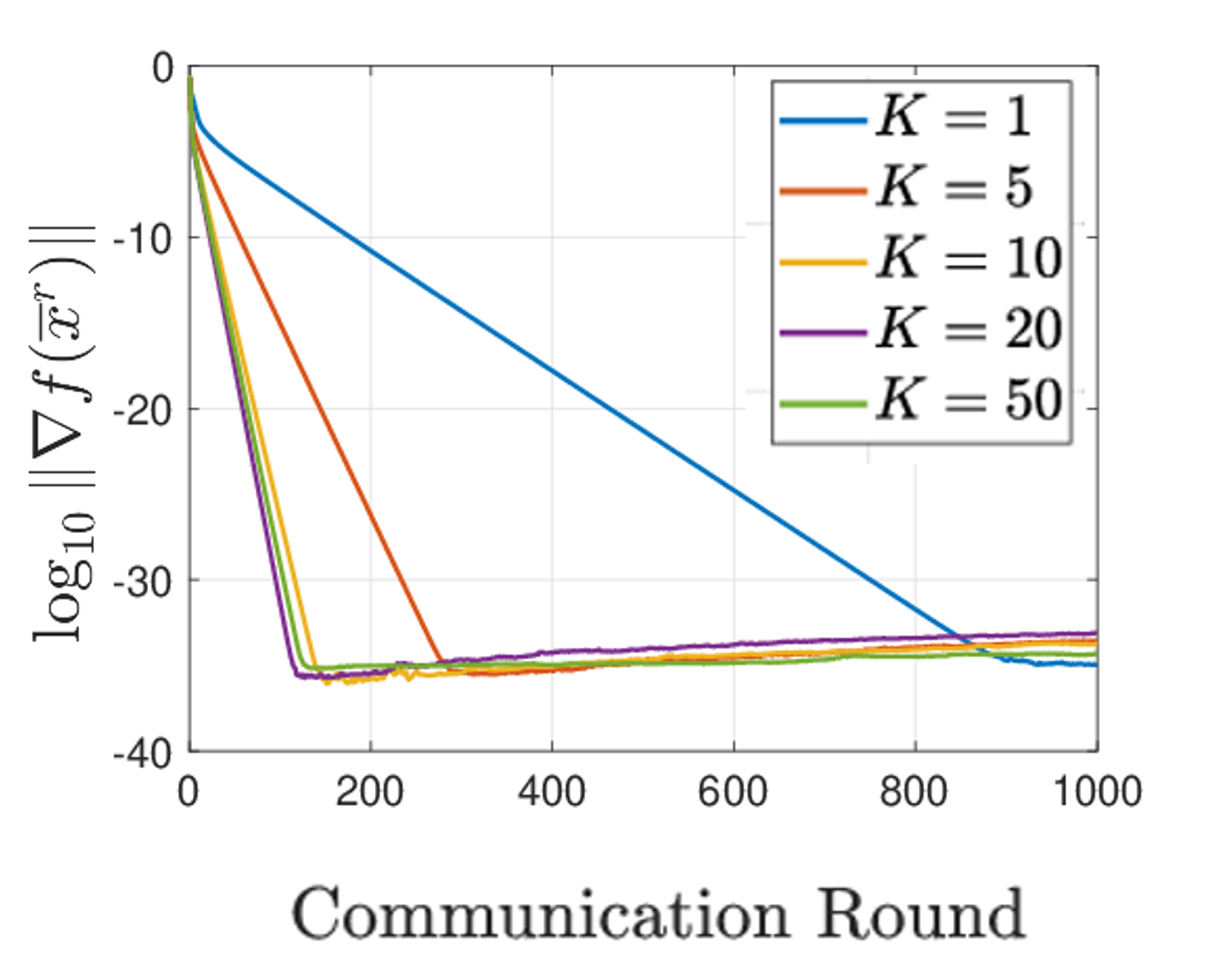}
\end{minipage}%
}%
\subfigure[$\rho=0$]{
\begin{minipage}[t]{0.3\linewidth}
\centering
\includegraphics[scale=0.2]{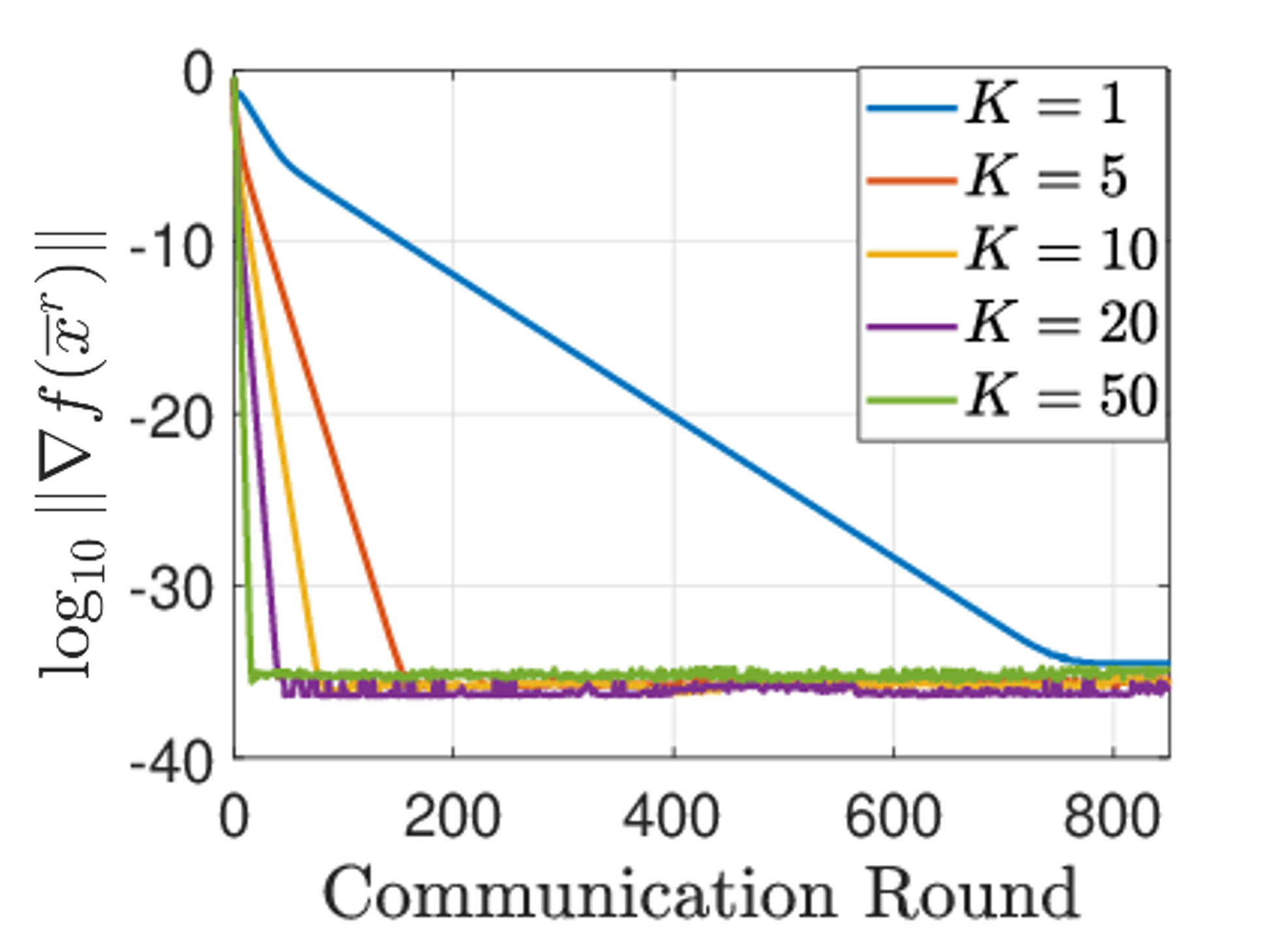}
\end{minipage}%
}

\subfigure[$\rho=0.8500$]{
\begin{minipage}[t]{0.3\linewidth}
\centering
\includegraphics[scale=0.2]{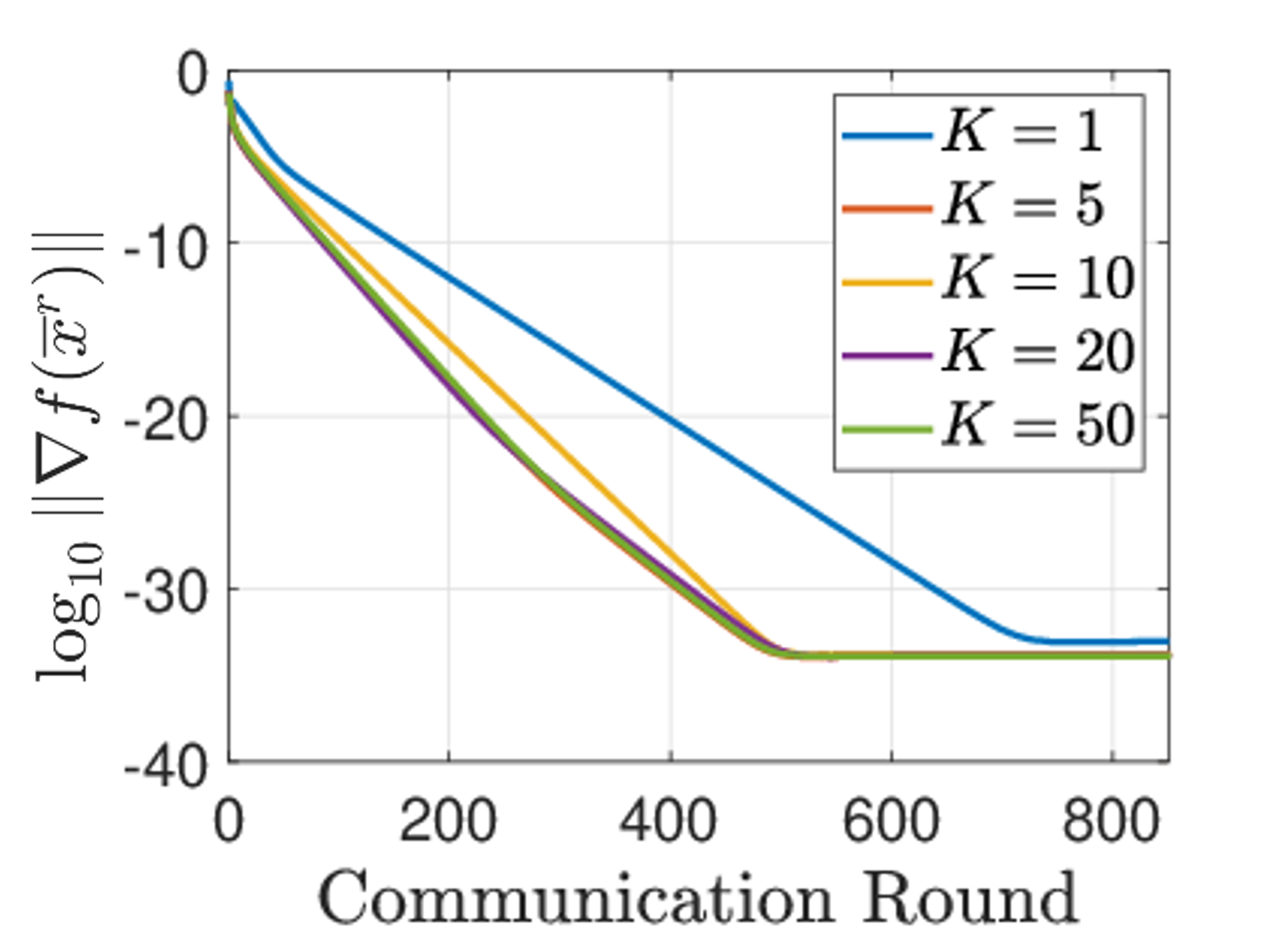}
\end{minipage}%
}%
\subfigure[$\rho=0.3164$]{
\begin{minipage}[t]{0.3\linewidth}
\centering
\includegraphics[scale=0.19]{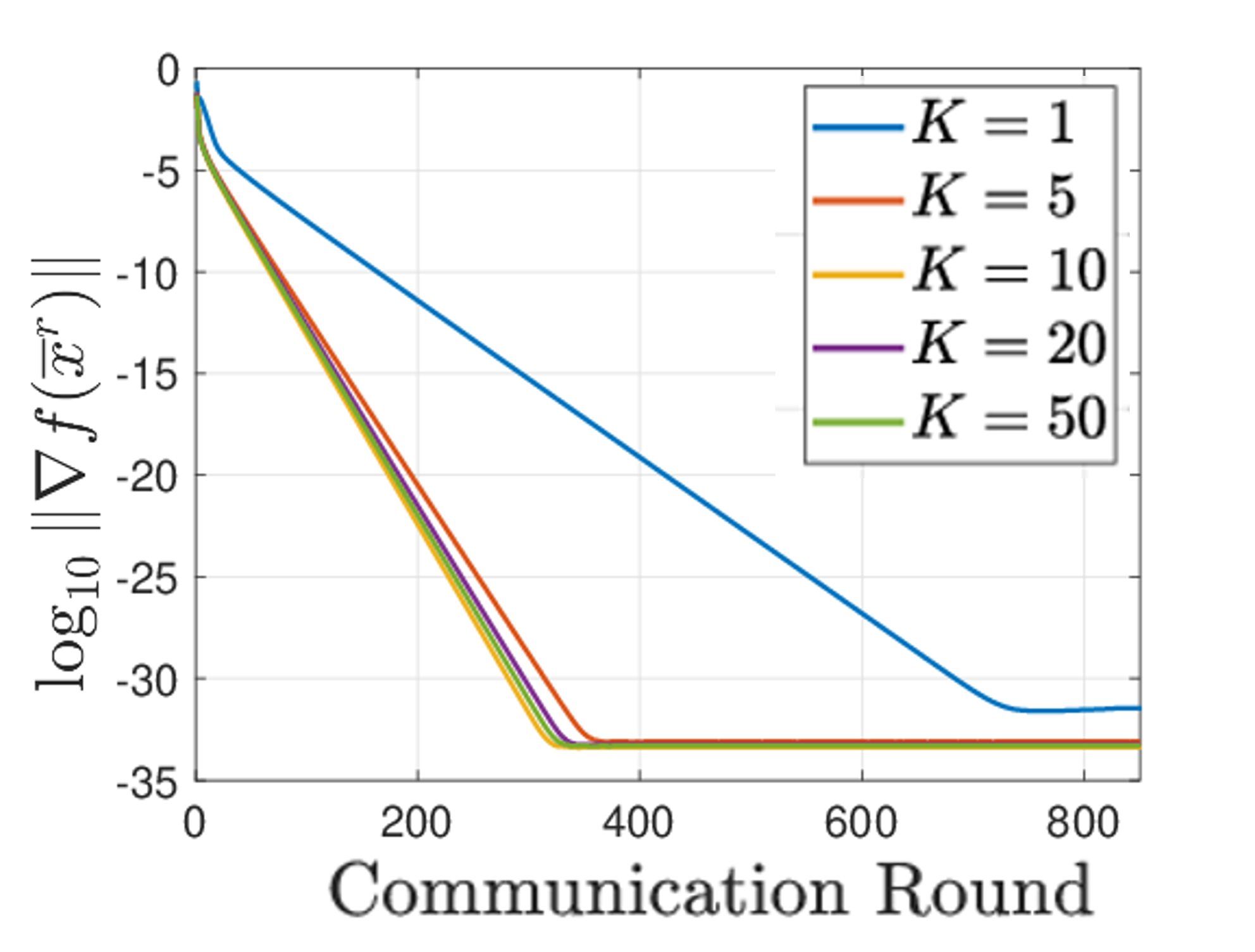}
\end{minipage}%
}%
\subfigure[$\rho=0$]{
\begin{minipage}[t]{0.3\linewidth}
\centering
\includegraphics[scale=0.2]{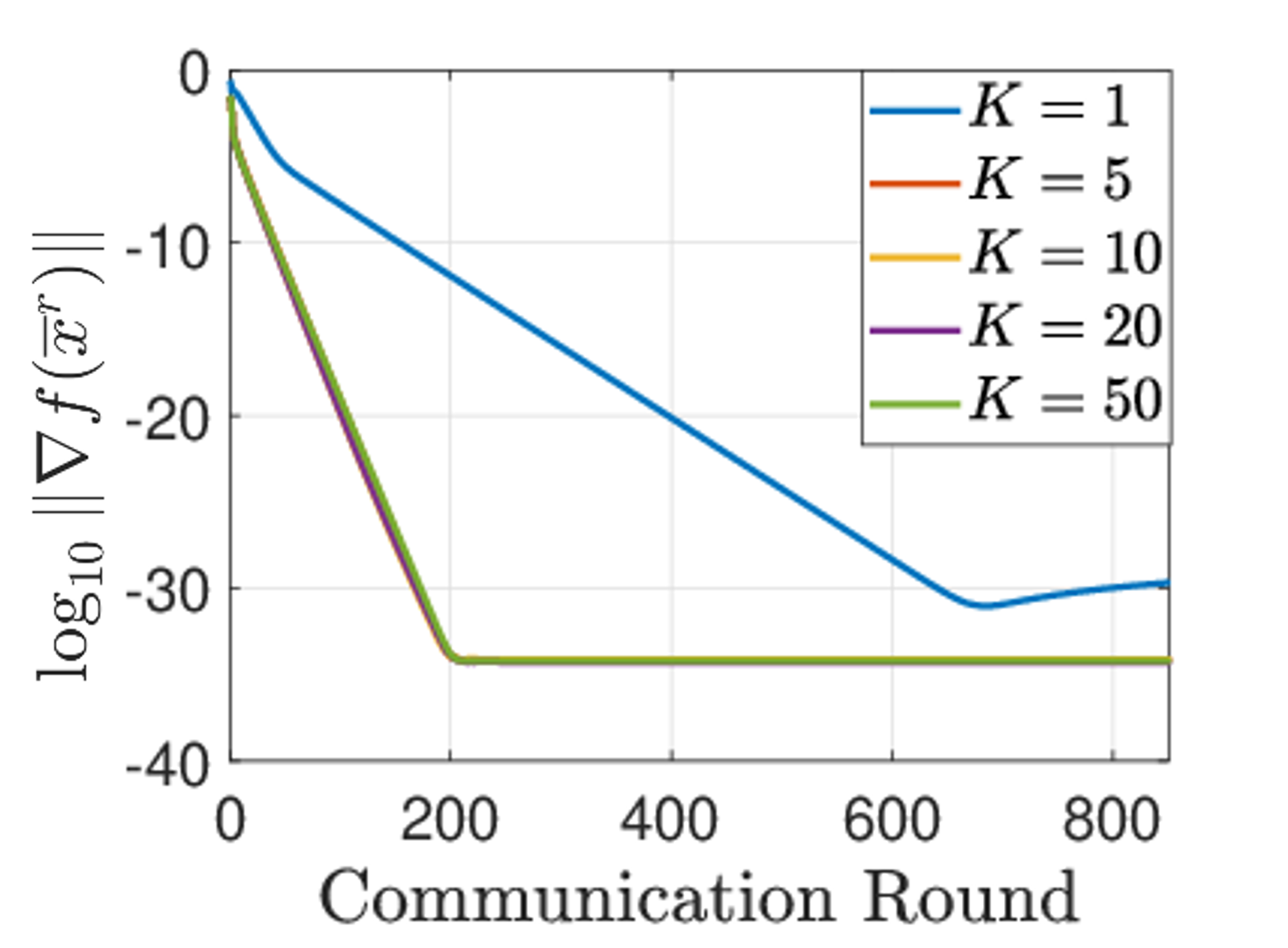}
\end{minipage}%
}%
\caption{ {}{Convergence for local DGT (top) and LED (bottom) for solving DRLR on \enquote{a9a} dataset under low heterogeneity degree. The values of $\rho$ are set $\rho=0.85, 0.3164, 0$. } }
\label{fig3}
\end{figure}

{}{
\begin{remark}
Careful readers may note that the curves in Fig. \ref{fig1}. (f) under the high heterogeneity setting converge faster than that of Fig. \ref{fig1}. (d) under the low heterogeneity setting. This is due to the discrepancy in the datasets generated from two different scenarios.
\end{remark}
}

The complete experimental results illustrate local updates can effectively reduce communication complexity when the network connectivity is relatively good or data heterogeneity is mild. However, in cases where local updates are useful, increasing $K$ beyond certain thresholds offers marginal benefit. These observations corroborate the result in Theorem \ref{th1}.

\subsection{Local DGT for DRLR on real-world dataset} \label{sec:sim-rea-logit} In this part, we use  DRLR loss for binary classification on the real-world \enquote{a9a} dataset, which could be downloaded from LIBSVM repository \cite{chang2011libsvm}. We set the number of agents $m = 50$ and assign the same number of training samples to agents.

\emph{Data distribution}. The feature dimension is $d=122$ and the number of training samples is 15000, with 11500 and 3500 for class 1 and class 2, respectively. The number of test samples is 1281, with 935 and 346 for class 1 and class 2, respectively. We properly allocate samples among agents to construct two different degrees of heterogeneity. For moderate heterogeneity, we allocate the training samples of the first 38 agents to be class 1, while the training samples of the remaining agents are class 2. For low heterogeneity, we allocate data uniformly that there are 230 samples with class 1 and 70 samples with class 2 in each agent.

\emph{Network setting}. We conduct the local DGT on networks with three different connectivities $\rho = 0.85, 0.3164, 0$ (fully connected). We set the regularization parameter $\mu=1$ in all cases and compare the performance of local DGT with the method LED \cite{alghunaim2024local}, which is the most recent decentralized method with local updates.

It can be observed that when network connectivity is sufficient, local DGT is more communication efficient than LED. This conclusion is based on the following phenomenons. (1) Under the moderate heterogeneity case, Fig. \ref{fig2}. (b) and Fig. \ref{fig2}. (e) shows the comparison of local DGT and LED under network $\rho=0.3164$. When $K=1$ and $K=5$, the performance of both methods is almost identical. However, as $K$ increases to 10, the communication overhead of local DGT continues to decrease, whereas the number of communication rounds for LED does not show a similar reduction.
(2) In the case of low heterogeneity, Fig. \ref{fig3}. (b) and Fig. \ref{fig3}. (e) present the comparison results of local DGT and LED under network setting $\rho=0.3164$. When $K=1$, the efficiency of LED is slightly higher than that of local DGT. However, with the increase of  $K$, the number of communication rounds required by local DGT becomes less than that of LED.
\begin{figure}[htbp]
\centering
\subfigure[Low heterogeneity ]{
\begin{minipage}[t]{0.45\linewidth}
\centering
\includegraphics[scale=0.38]{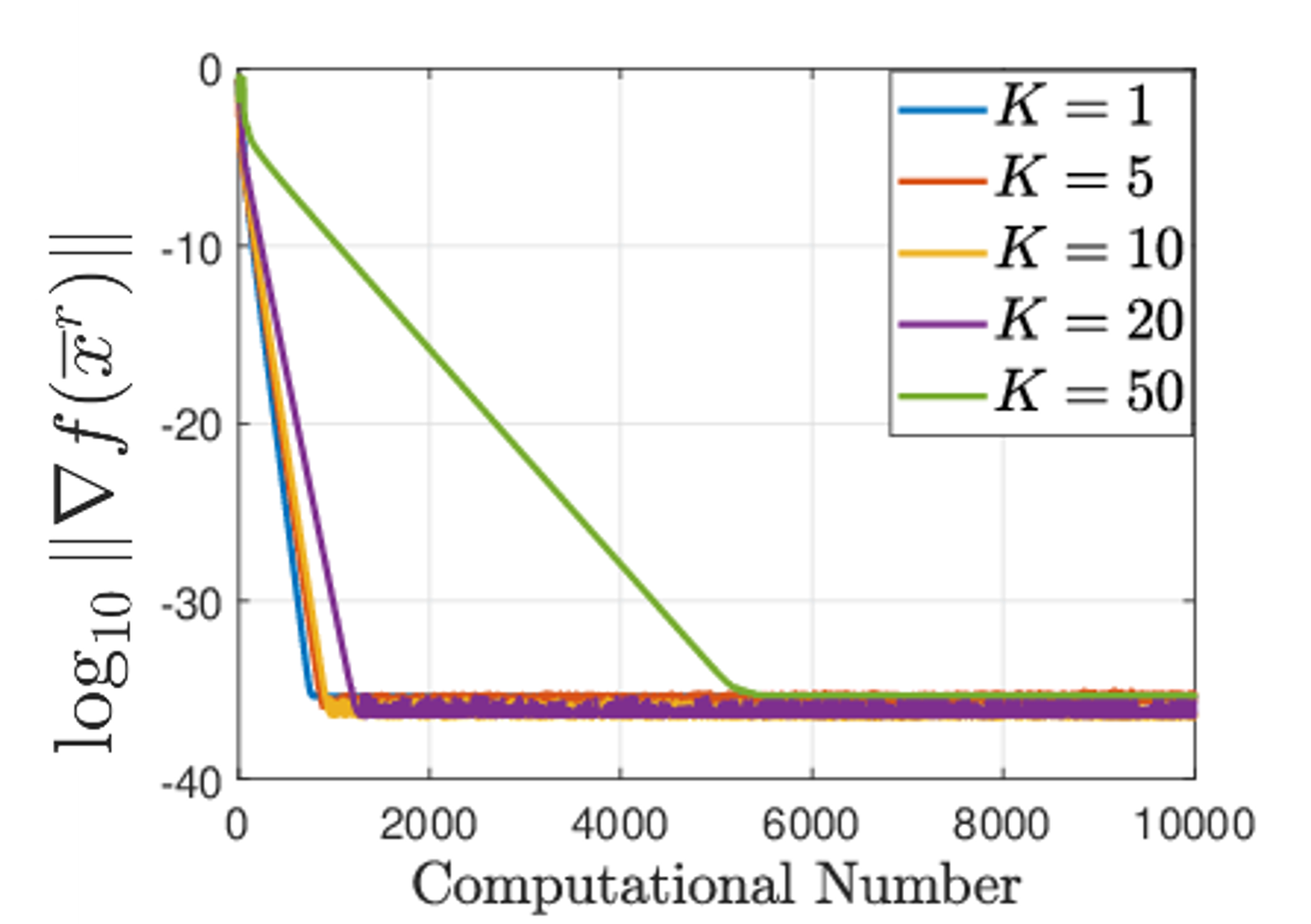}
\end{minipage}%
}%
\subfigure[Moderate heterogeneity]{
\begin{minipage}[t]{0.45\linewidth}
\centering
\includegraphics[scale=0.38]{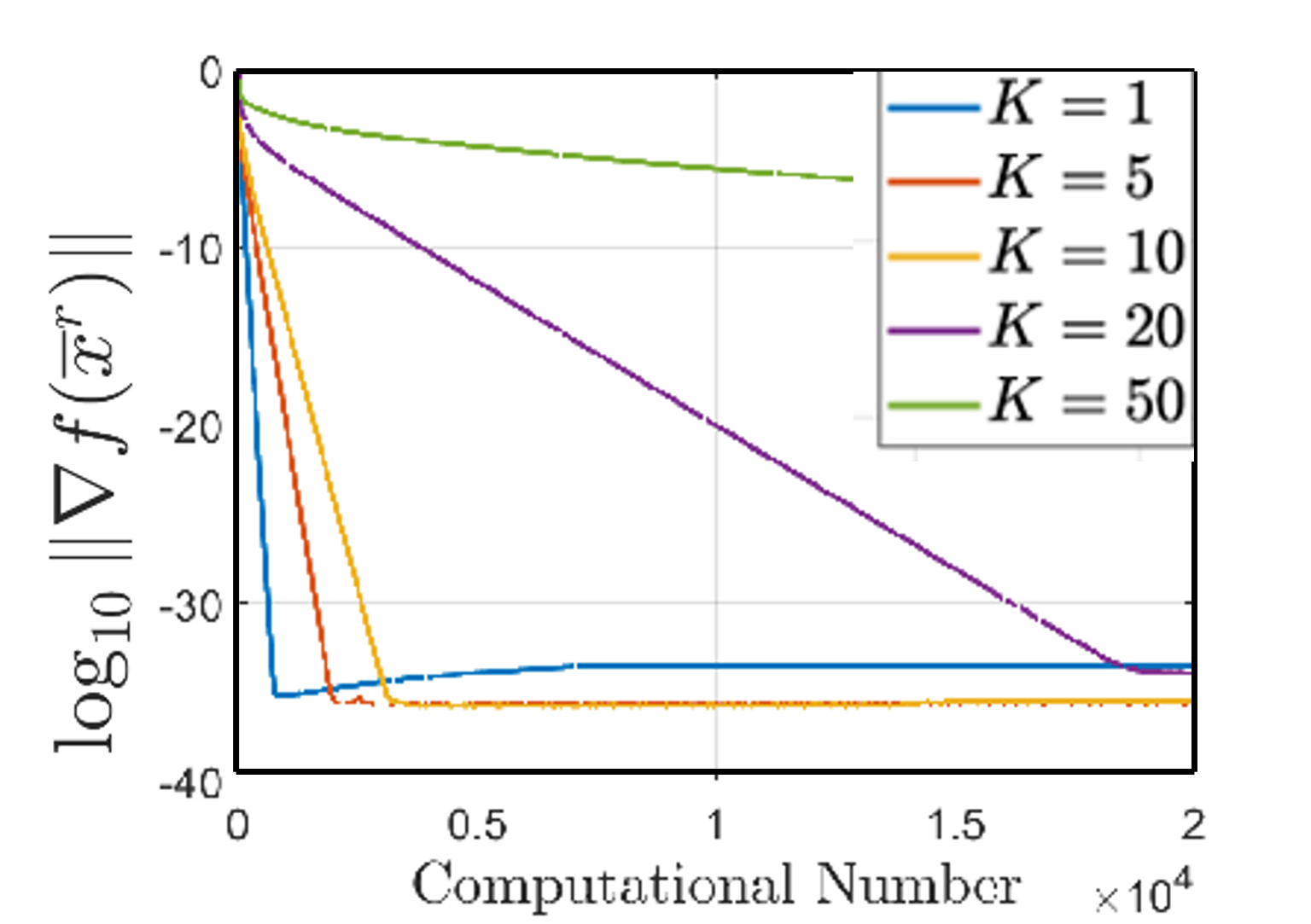}
\end{minipage}%
}%

\subfigure[Low heterogeneity]{
\begin{minipage}[t]{0.45\linewidth}
\centering
\includegraphics[scale=0.38]{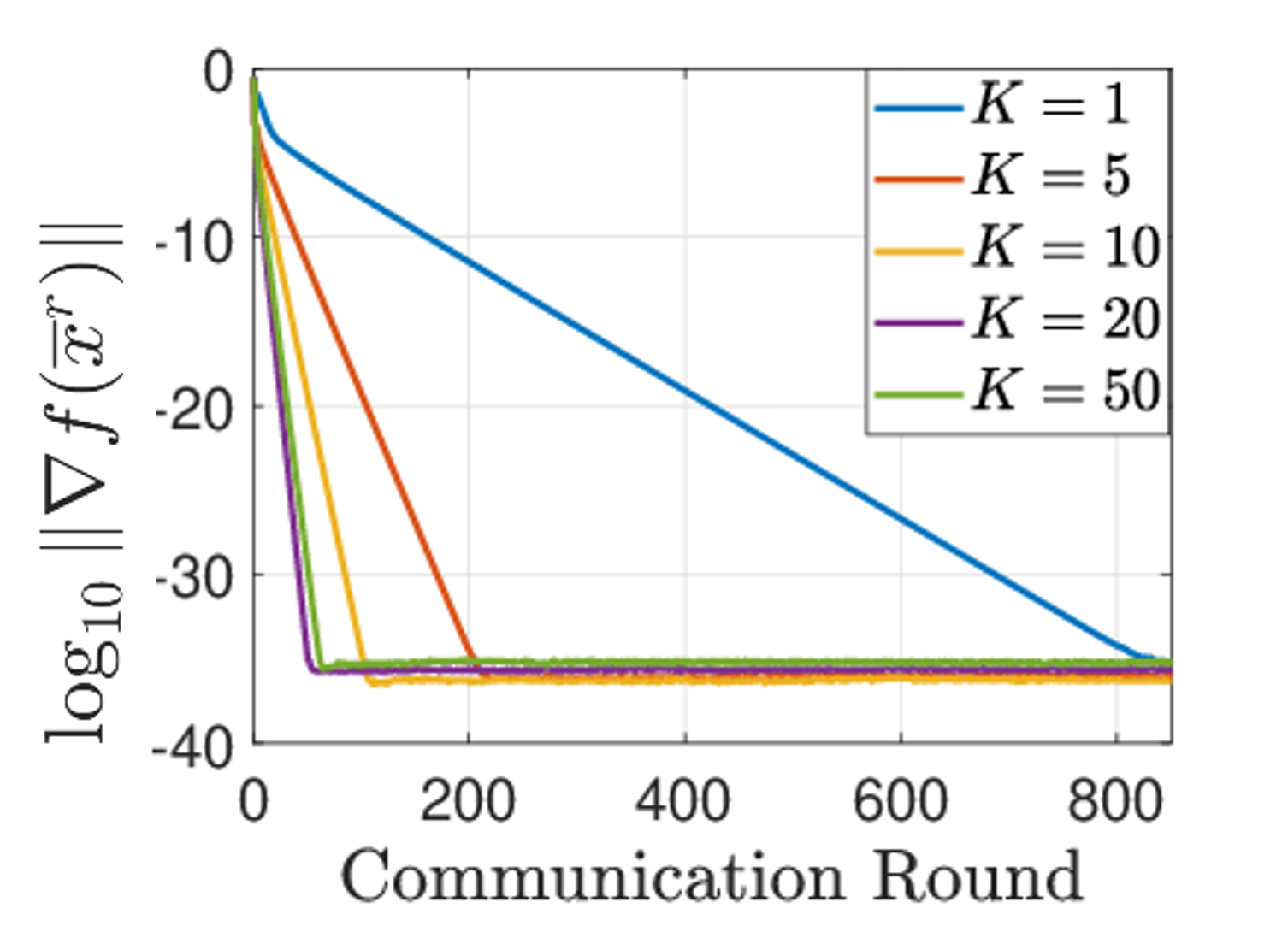}
\end{minipage}%
}%
\subfigure[Moderate heterogeneity]{
\begin{minipage}[t]{0.45\linewidth}
\centering
\includegraphics[scale=0.38]{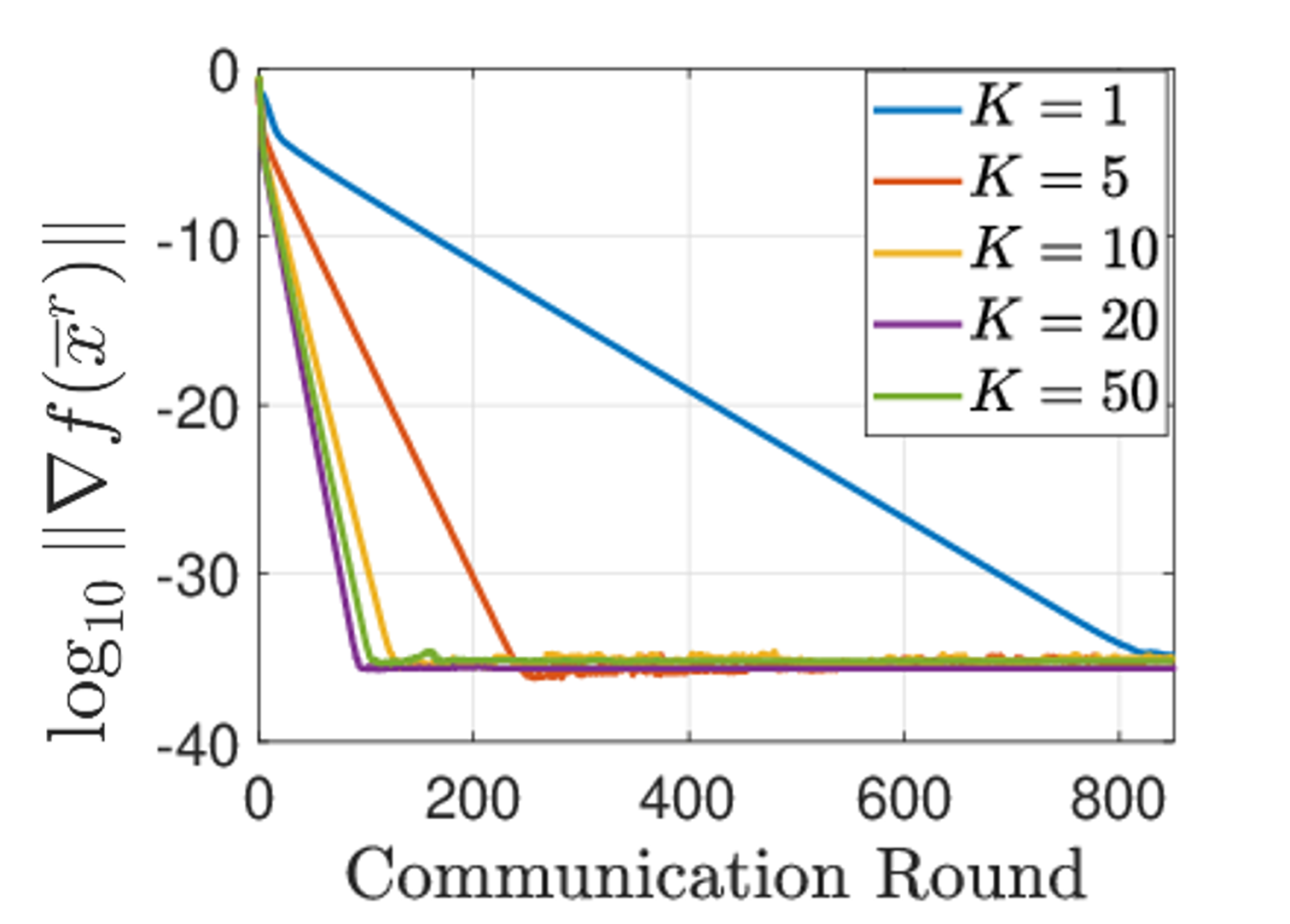}
\end{minipage}%
}%
\caption{ Convergence with respect to computational number and communication round for local DGT for solving DRLR on “a9a” dataset under low and moderate heterogeneity with $\rho=0.1095$. }
\label{fig5}
\end{figure}

\textbf{Trade-off between computation and communication.} {}{Fig. \ref{fig5}} considers both computation and communication costs of local DGT for solving DRLR on ``a9a” dataset.

Fig. \ref{fig5}. (a) shows that in the low heterogeneity when \(K \leq 20\), increasing the number of local updates (significantly reducing the communication costs) will only sacrifice negligible computational overhead. However, when \(K=50\), although it can spend minimal communication rounds, it requires the most computational cost. Thus, there is a trade-off between computation and communication for the number of local updates. In this setting, \(K=20\) can achieve the best trade-off, which is consistent with the conclusion obtained from Fig. \ref{fig5}. (c) that plots the convergence with respect to the communication round.

Fig. \ref{fig5}. (b) shows that in the moderate heterogeneity when \(K \leq 10\), increasing the number of local updates (significantly reducing the communication costs) will only sacrifice negligible computational overhead. However, when \(K \geq 20\), although it can spend fewer communication rounds, it requires significantly more computational cost. Thus, in this setting, \(K=10\) can achieve the best trade-off, which is consistent with the conclusion obtained from Fig. \ref{fig5}. (d) that plots the convergence with respect to the communication round.

\subsection{Decentralized training for deep neural networks}\label{nn_test}
Besides above strongly convex setting, we also have conducted decentralized training highly non-convex VGG-like deep convolutional neural network (CNN) on MNIST \cite{lecun1998gradient}, a widely used benchmark for image classification tasks. The MNIST dataset comprises 60,000 training images and 10,000 testing images of handwritten digits, each represented as a $28\times 28$ grayscale image. The images are labeled with digits from 0 to 9, making this a 10-class classification problem.

We also compare local DGD and local DGT with the three latest decentralized methods with local updates, which are decentralized Scaffnew \cite{mishchenko2022proxskip}, local exact diffusion (LED) \cite{alghunaim2024local} and K-GT \cite{liu2024decentralized} on two data distributions which are corresponding to different heterogeneity settings. The following are details of our experiments.

The communication network has 10 agents and is generated by {}{Erd\H{o}s R\'enyi (ER)} graph with $\rho = 0.1778$. All methods are conducted on the same network. We normalized the pixel values to the range $[-1,1]$ to preprocess the data.

We considered two distinct data distribution settings. \emph{Uniform distribution:} In this setting, the dataset was uniformly split among the agents, ensuring that each agent received an approximately equal portion of the data to simulate homogeneous conditions. \emph{Heterogeneous distribution:} Each agent was assigned data from a subset of randomly selected classes. Specifically, each agent received data from 5 classes, chosen randomly from the 10 available classes. This setup introduced variability in the data distribution across agents, reflecting practical challenges in federated and decentralized learning.

 The performance of each method was evaluated based on the test accuracy over multiple communication rounds. For each method, we conducted experiments with different numbers of local updates per communication round, specifically $K = 1, 5, 10, 20$. For each method and each $K$, we use the best constant learning rate tuned from $\left\{0.0001,0.0005,0.001,0.005,0.01,0.05,0.1,0.5\right\}$.
\begin{figure}[htbp]
\centering
\subfigure[$K=1$]{
\begin{minipage}[t]{0.45\linewidth}
\centering
\includegraphics[scale=0.25]{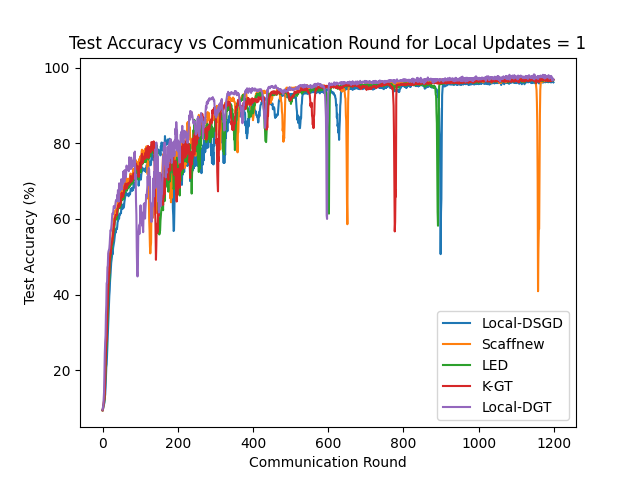}
\end{minipage}%
}%
\subfigure[$K=5$]{
\begin{minipage}[t]{0.45\linewidth}
\centering
\includegraphics[scale=0.25]{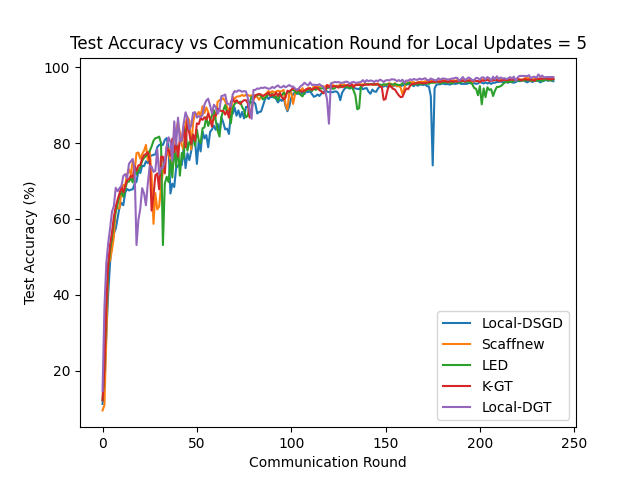}
\end{minipage}%
}%

\subfigure[$K=10$]{
\begin{minipage}[t]{0.45\linewidth}
\centering
\includegraphics[scale=0.25]{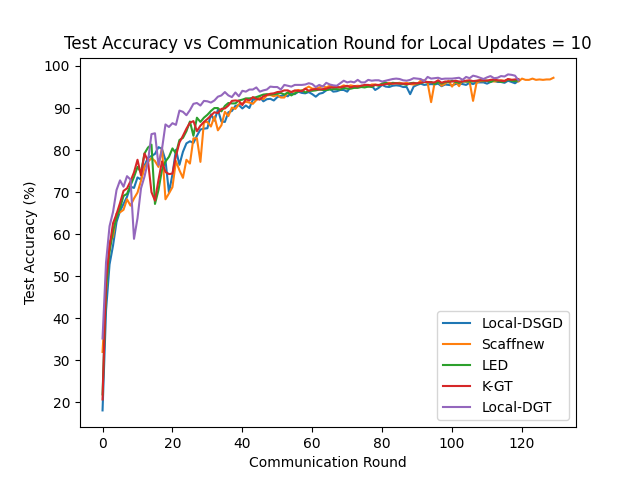}
\end{minipage}%
}%
\subfigure[$K=20$]{
\begin{minipage}[t]{0.45\linewidth}
\centering
\includegraphics[scale=0.25]{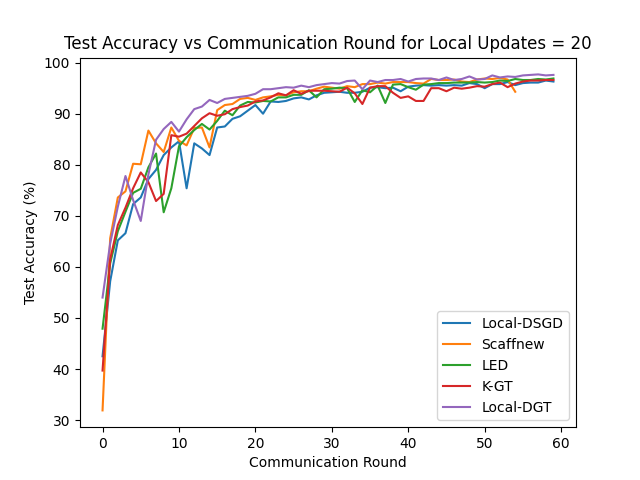}
\end{minipage}%
}%
\caption{Uniform distribution: generalization performance on MNIST for all methods under different numbers of local updates. }
\label{fig6}
\end{figure}

 Fig. \ref{fig6} and Fig. \ref{fig7} show compared results under the uniform distribution setting and heterogeneous distribution, respectively. It can be observed that in the uniform data distribution setting, all methods have almost the same competitive performance that can achieve high test accuracy even for training the highly non-convex deep CNN model. In addition, with the increasing $K$, all methods use the decreasing number of communication rounds to achieve final convergence. Specifically, the number of communication rounds to achieve the same accuracy is reduced from 600 to 30 when we increase $K$ from 1 to 20 for local DGT. This significant reduction in communication overhead benefits from the local updates. When comes to the heterogeneous setting, local DSGD has slower convergence with increasing $K$. This is because local DSGD suffers more severe client drift in this setting due to a lack of gradient corrections, which are used in the other four methods. Thus, the performance of these four methods is almost unaffected in this setting. The local update is also effective in this case, which shows that the communication rounds for local DGT decrease from 800 to 60 when $K$ increases from 1 to 20  for achieving the same test accuracy. 

\begin{figure}[htbp]
\centering
\subfigure[$K=1$]{
\begin{minipage}[t]{0.45\linewidth}
\centering
\includegraphics[scale=0.25]{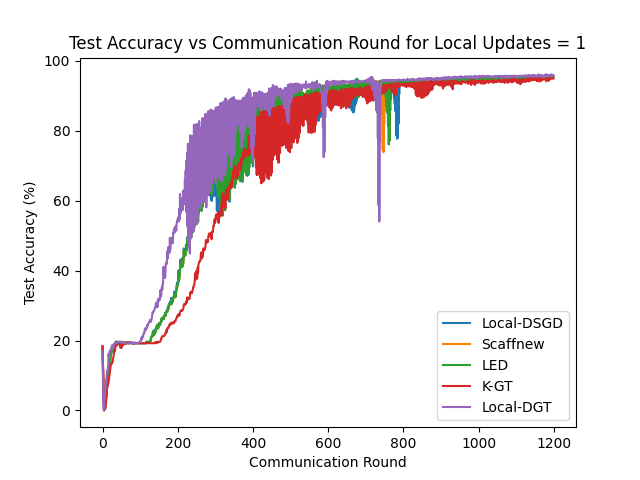}
\end{minipage}%
}%
\subfigure[$K=5$]{
\begin{minipage}[t]{0.45\linewidth}
\centering
\includegraphics[scale=0.25]{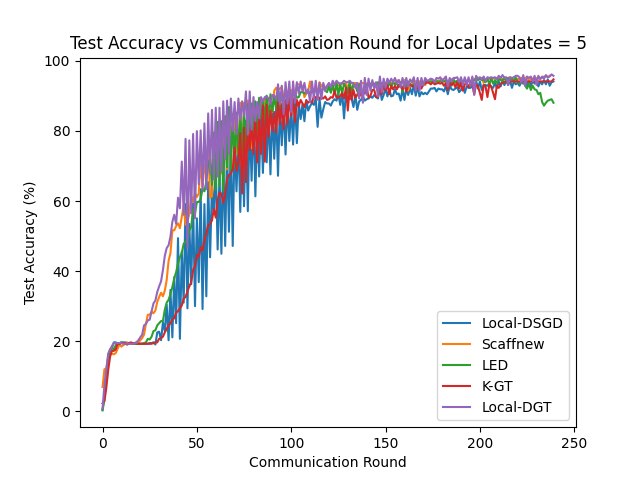}
\end{minipage}%
}%

\subfigure[$K=10$]{
\begin{minipage}[t]{0.45\linewidth}
\centering
\includegraphics[scale=0.25]{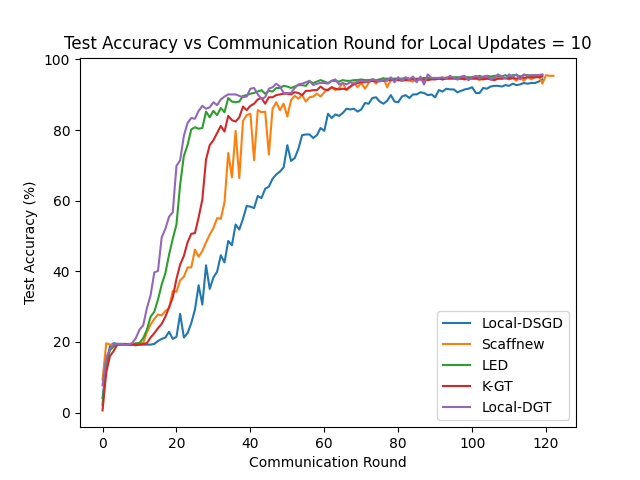}
\end{minipage}%
}%
\subfigure[$K=20$]{
\begin{minipage}[t]{0.45\linewidth}
\centering
\includegraphics[scale=0.25]{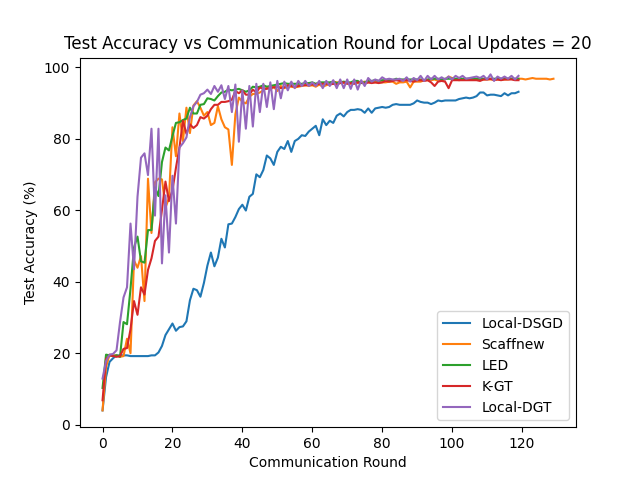}
\end{minipage}%
}%
\caption{Heterogeneous distribution: generalization performance on MNIST for all methods under different numbers of local updates. }
\label{fig7}
\end{figure}

Based on the above results, we can observe that even for solving highly non-convex deep neural networks, gradient correction-based methods with local updates can effectively reduce communication costs. Local DGT can achieve competitive performance as existing methods under both settings while local DGD suffers degradation with increasing local updates under heterogeneous settings.

\subsection{Over-parameterized linear regression}\label{sec:sim-OLS}
In this section, we conduct the over-parameterized least squares simulation to explore the usefulness of local updates for local DGT and local DGD. The local loss for the $i$-th agent is $f_i (\bm{x}) = \frac 1 {2n}\|\boldsymbol A_i\boldsymbol x - \boldsymbol b_i\|^2$. We  compare two methods in two settings with different degrees of data heterogeneity. The optimality measure is  set as $\|\overline{\bm{x}}^r - \bm{x}^\star\|$, i.e., the distance to the minimum norm solution $\bm x^\star$.
\begin{figure}[htbp]
\centering
\subfigure[$\rho=0.9676$]{
\begin{minipage}[t]{0.31\linewidth}
\centering
\includegraphics[scale=0.25]{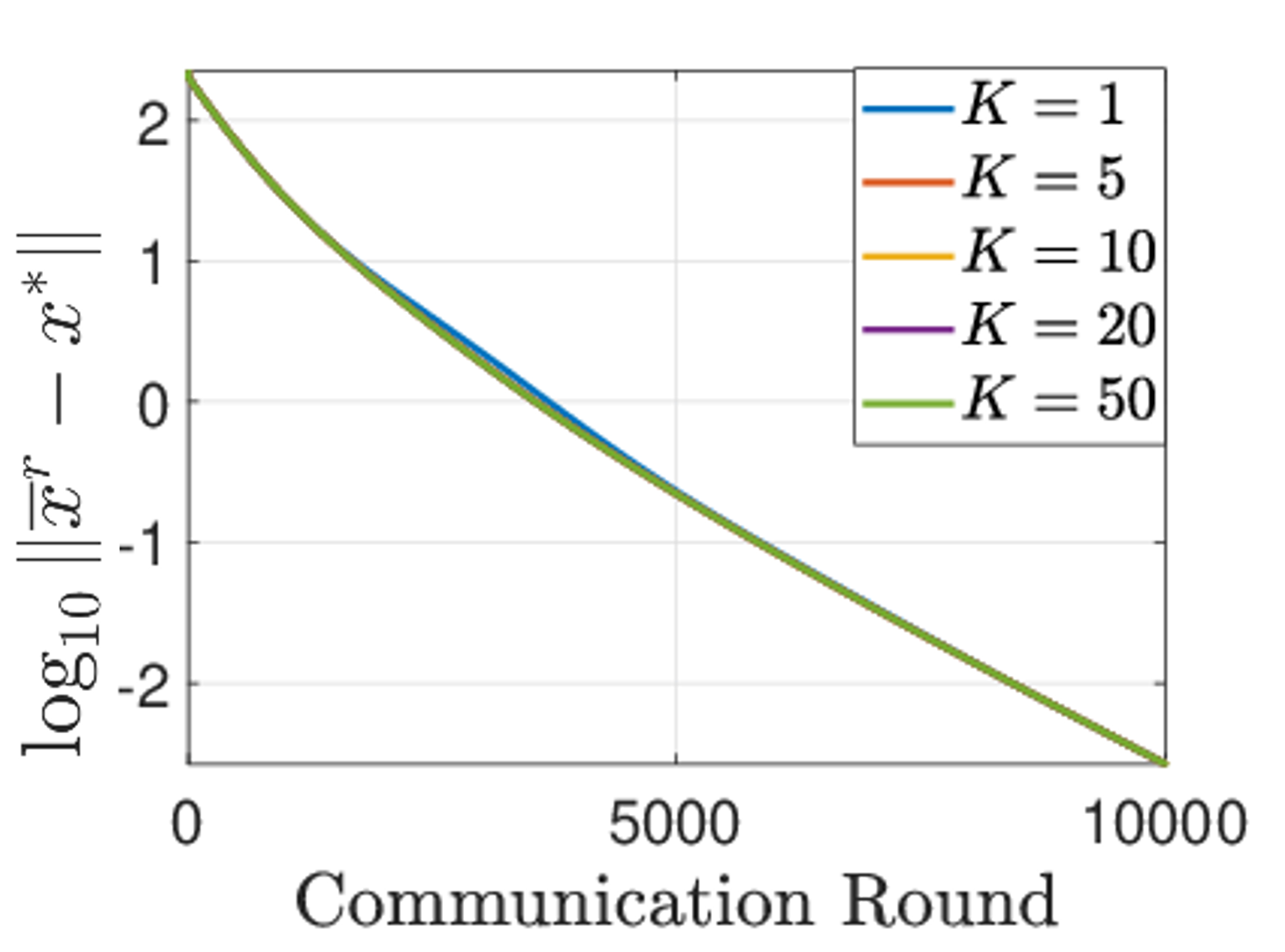}
\end{minipage}%
}%
\subfigure[$\rho=0.8924$]{
\begin{minipage}[t]{0.31\linewidth}
\centering
\includegraphics[scale=0.25]{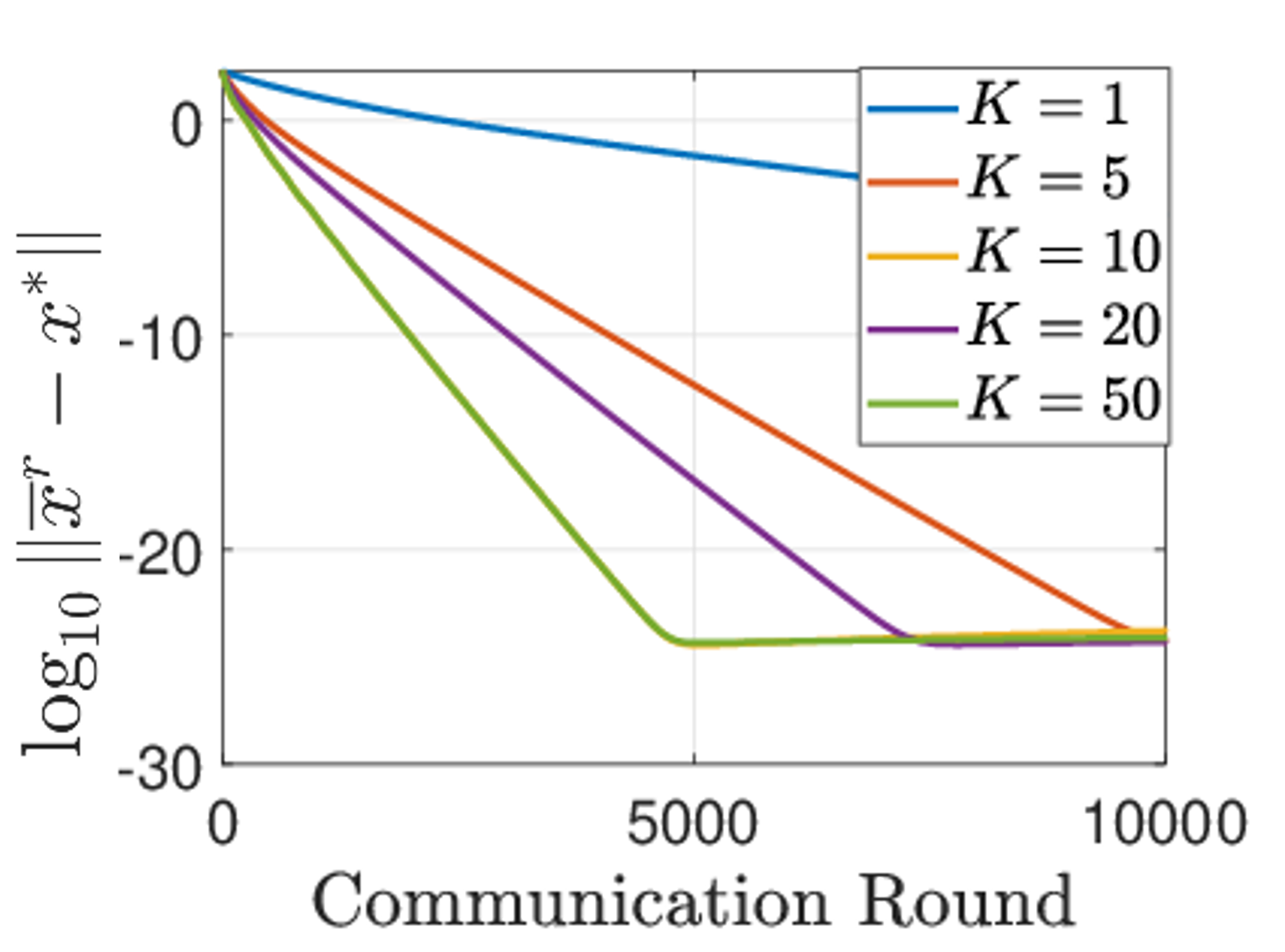}
\end{minipage}%
}%
\subfigure[$\rho=0$]{
\begin{minipage}[t]{0.31\linewidth}
\centering
\includegraphics[scale=0.25]{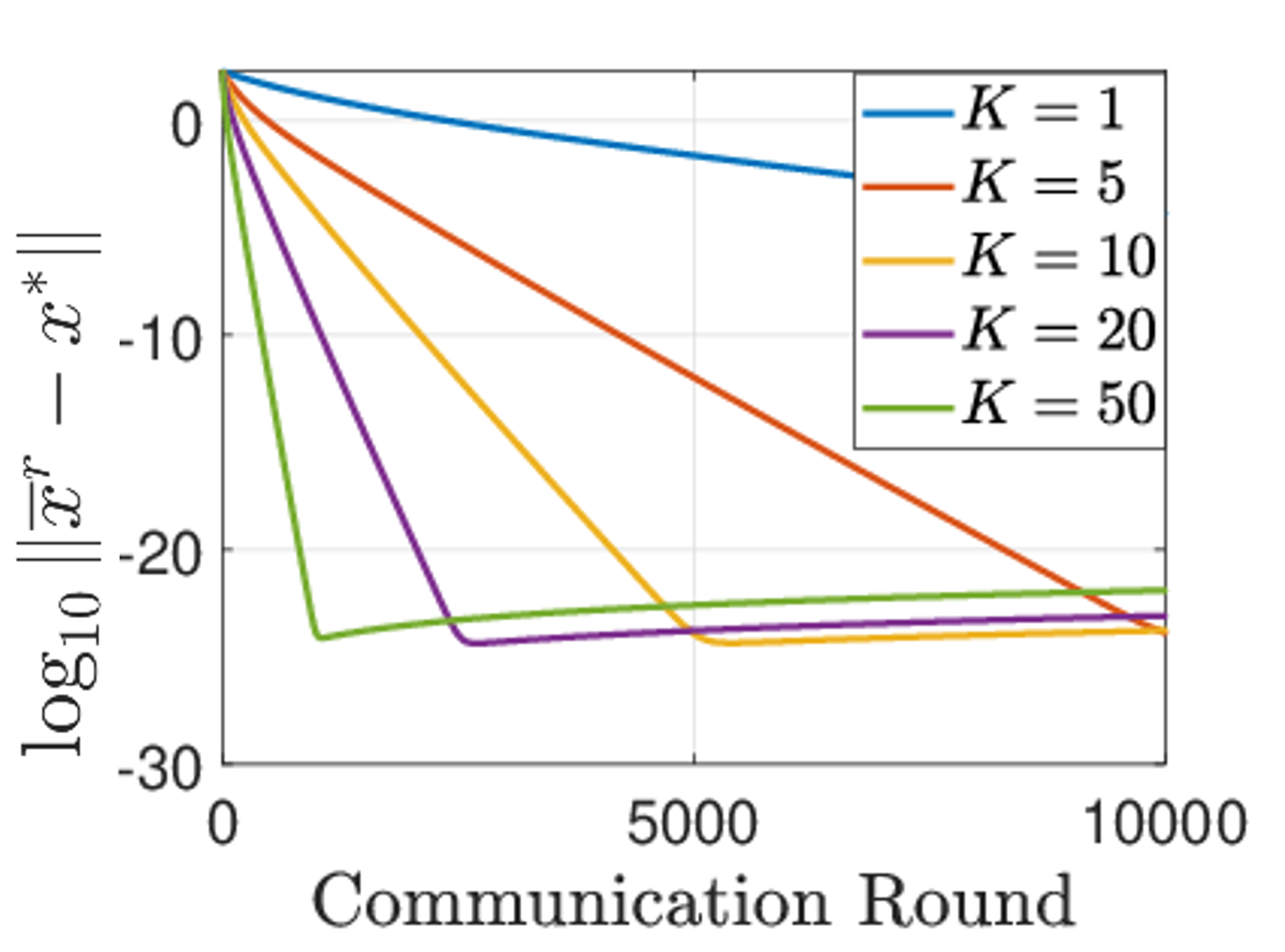}
\end{minipage}%
}%

\subfigure[$\rho=0.9676$]{
\begin{minipage}[t]{0.295\linewidth}
\centering
\includegraphics[scale=0.25]{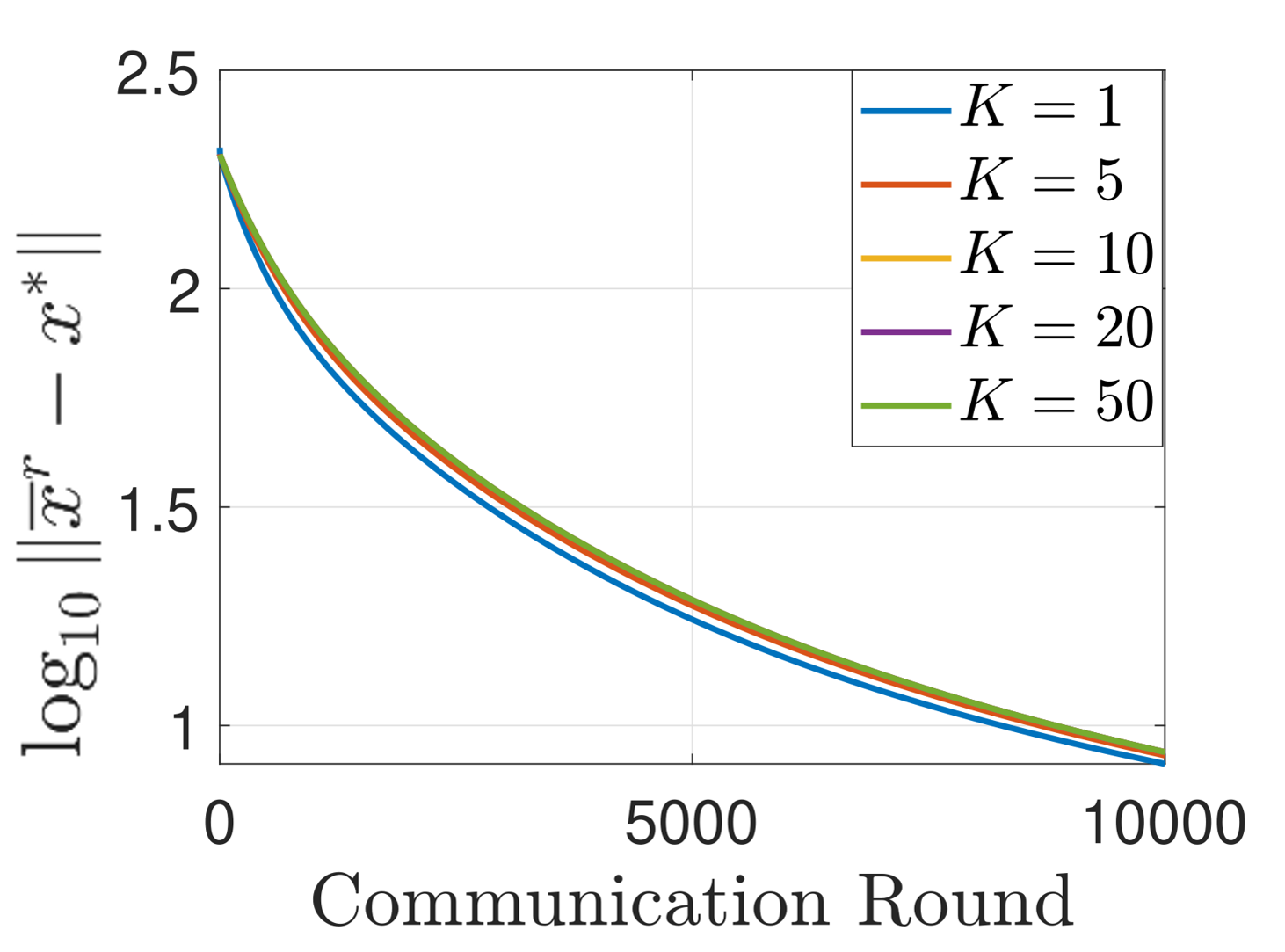}
\end{minipage}
}%
\subfigure[$\rho=0.8924$]{
\begin{minipage}[t]{0.295\linewidth}
\centering
\includegraphics[scale=0.25]{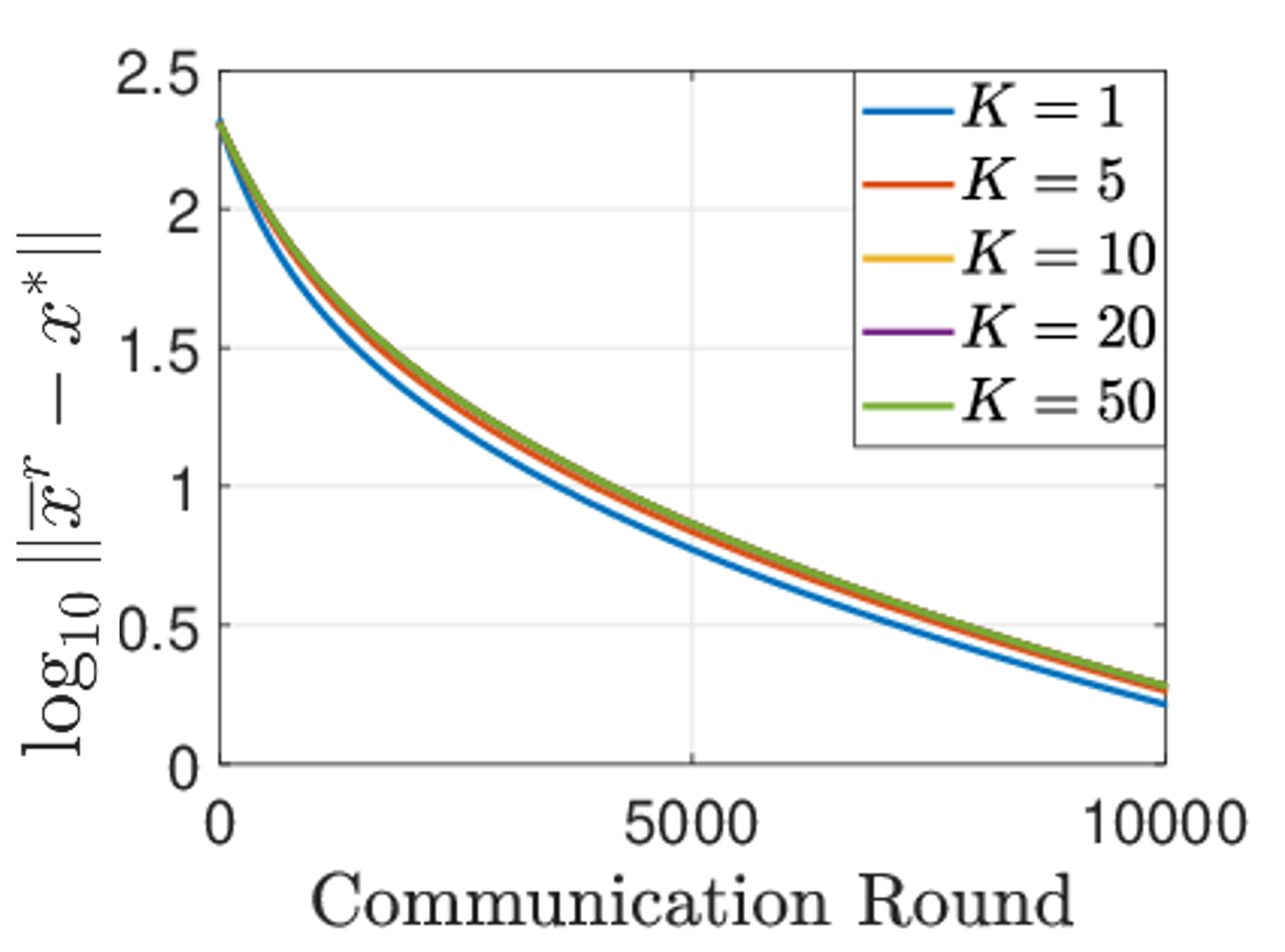}
\end{minipage}
}%
\subfigure[$\rho=0$]{
\begin{minipage}[t]{0.295\linewidth}
\centering
\includegraphics[scale=0.25]{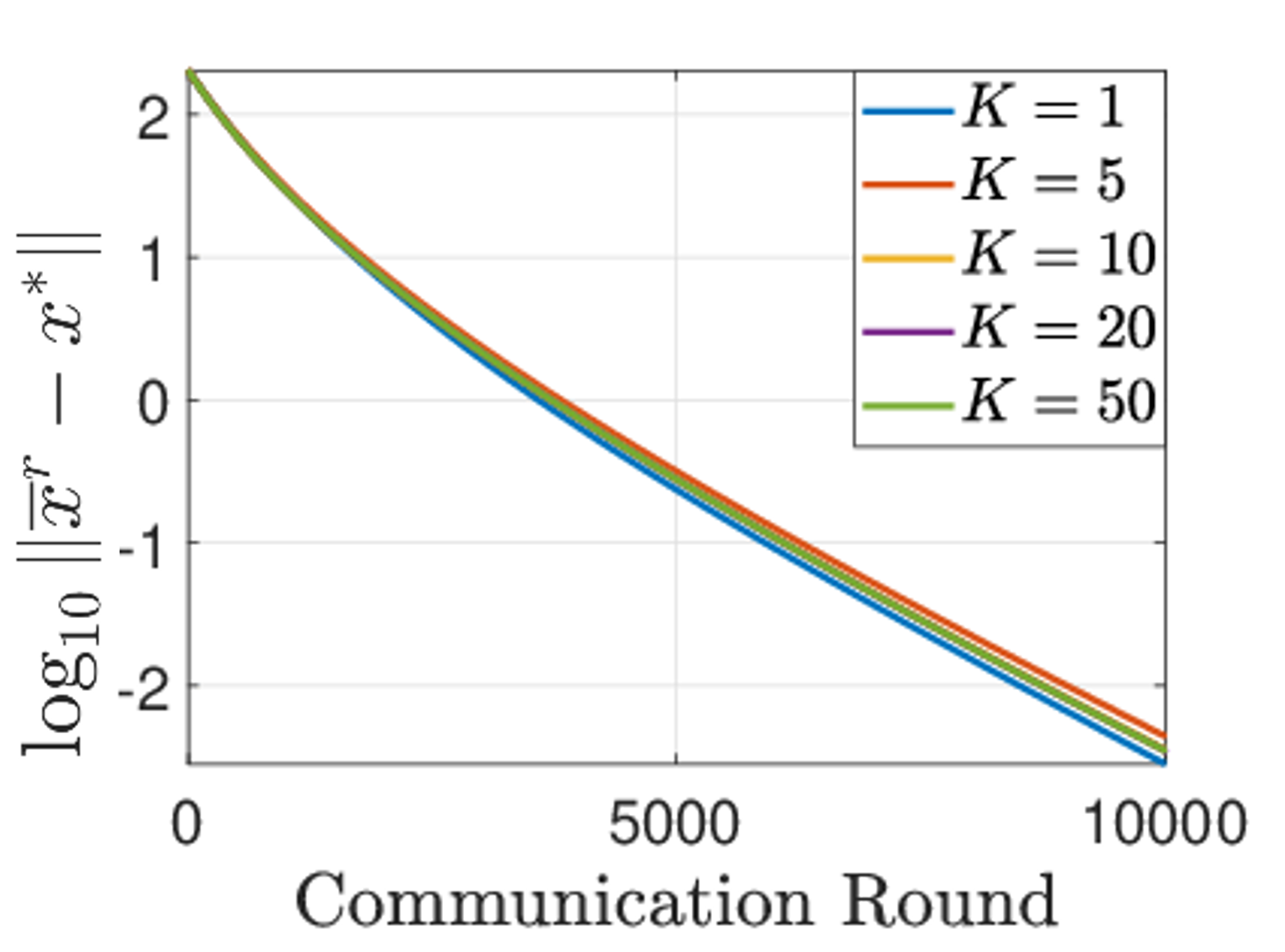}
\end{minipage}
}%
\centering
\caption{Influence of network connectivity under moderate data heterogeneity. First row: local DGT. Second row:
 local DGD.}
\label{fig8}
\end{figure}

\textbf{Moderate heterogeneity.} 
The convergence results are provided in Fig. \ref{fig8}. As $\rho$ decreases, the first row shows increasing local updates in local DGT become more effective in reducing communication. In contrast, the second row reveals increasing local updates in local DGD is less useful. This can be explained by Theorem \ref{thm3}. When $\frac{\delta}{\mu}$ is large, Eq.~\eqref{eq:complexity-OLS} shows the communication complexity will be dominated by the second term that scales with $(\frac{\delta}{\mu})^2$. As such, increasing $K$ will not significantly reduce communication. However, the complexity of local DGT only scales with $\frac{\delta}{\mu}$, and consequently increasing $K$ is more effective compared to local DGD. 

\begin{figure}[htbp]
\centering
\subfigure[$\rho=0.9664$]{
\begin{minipage}[t]{0.315\linewidth}
\centering
\includegraphics[scale=0.25]{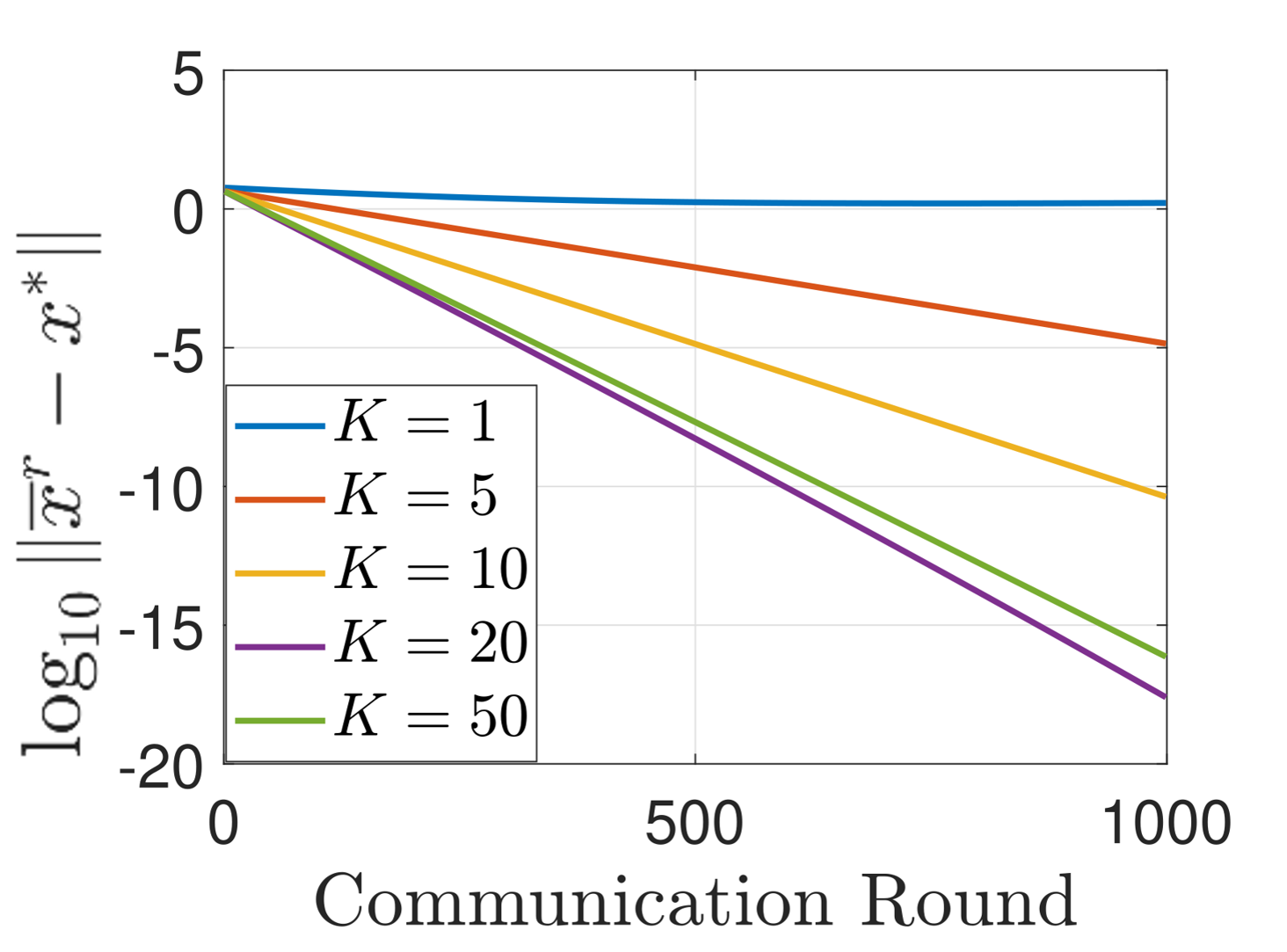}
\end{minipage}%
}%
\subfigure[$\rho=0.8924$]{
\begin{minipage}[t]{0.315\linewidth}
\centering
\includegraphics[scale=0.25]{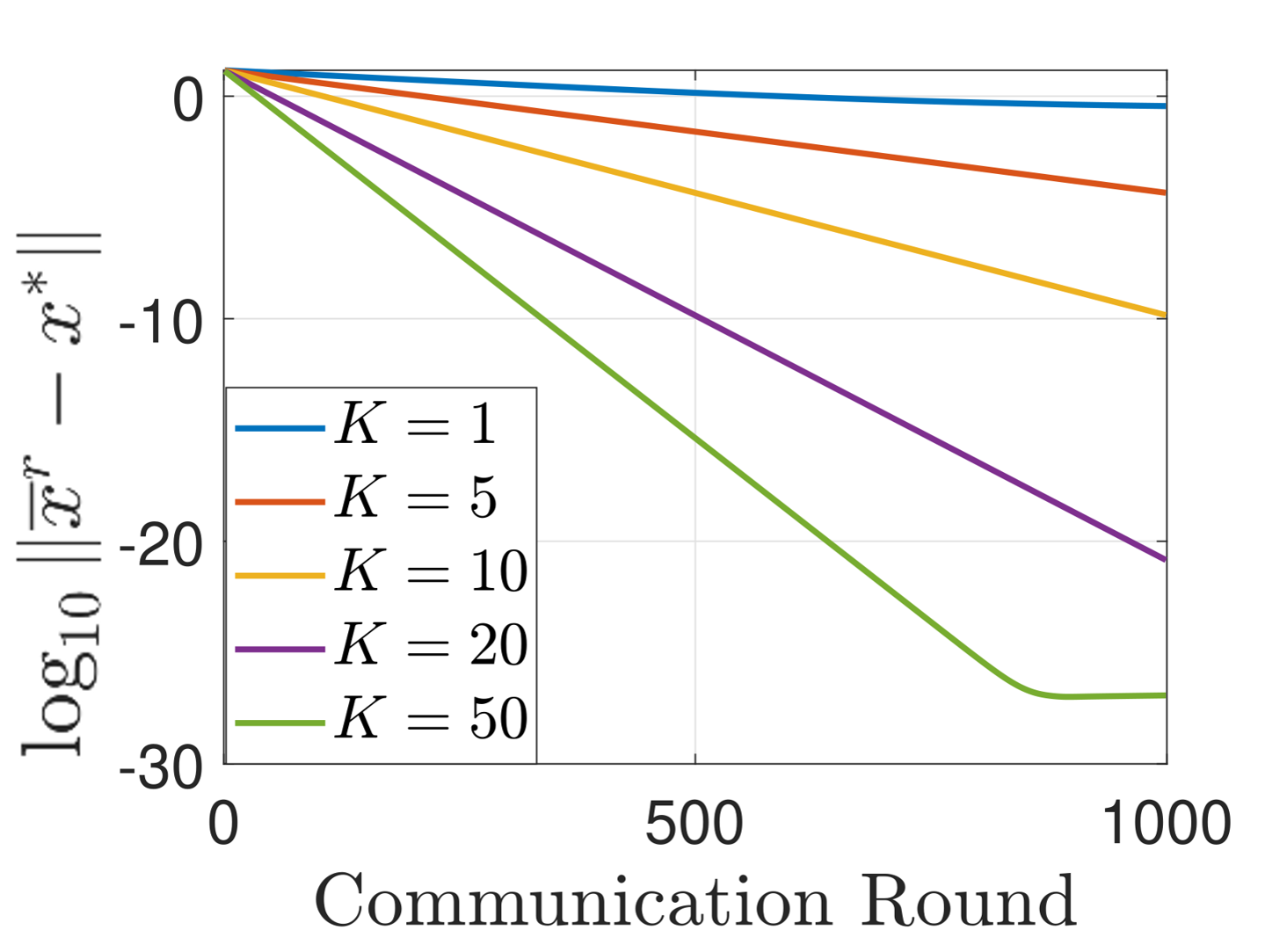}
\end{minipage}%
}%
\subfigure[$\rho=0$]{
\begin{minipage}[t]{0.315\linewidth}
\centering
\includegraphics[scale=0.25]{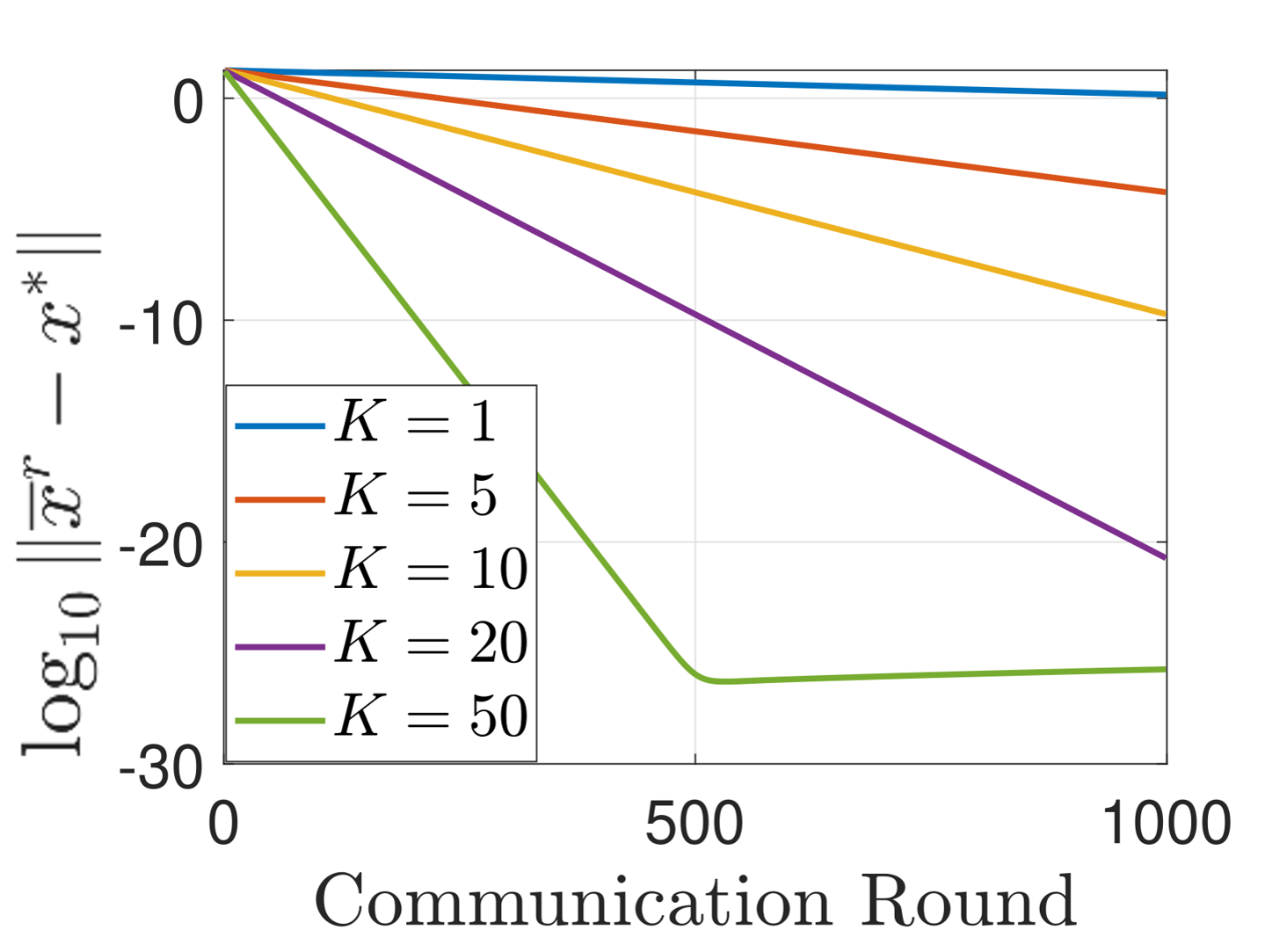}
\end{minipage}%
}%

\subfigure[$\rho=0.9664$]{
\begin{minipage}[t]{0.3\linewidth}
\centering
\includegraphics[scale=0.25]{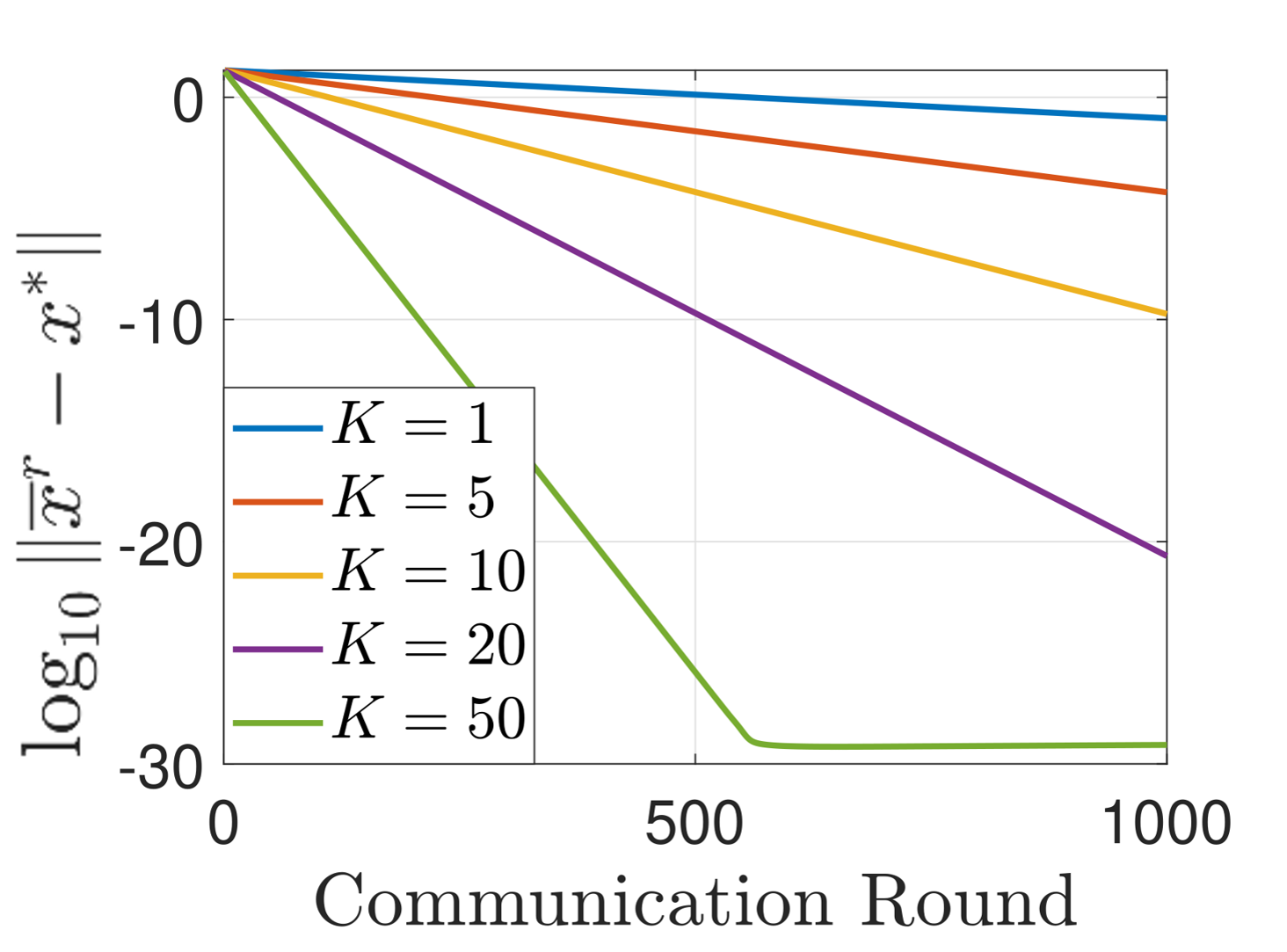}
\end{minipage}
}%
\subfigure[$\rho=0.8924$]{
\begin{minipage}[t]{0.3\linewidth}
\centering
\includegraphics[scale=0.25]{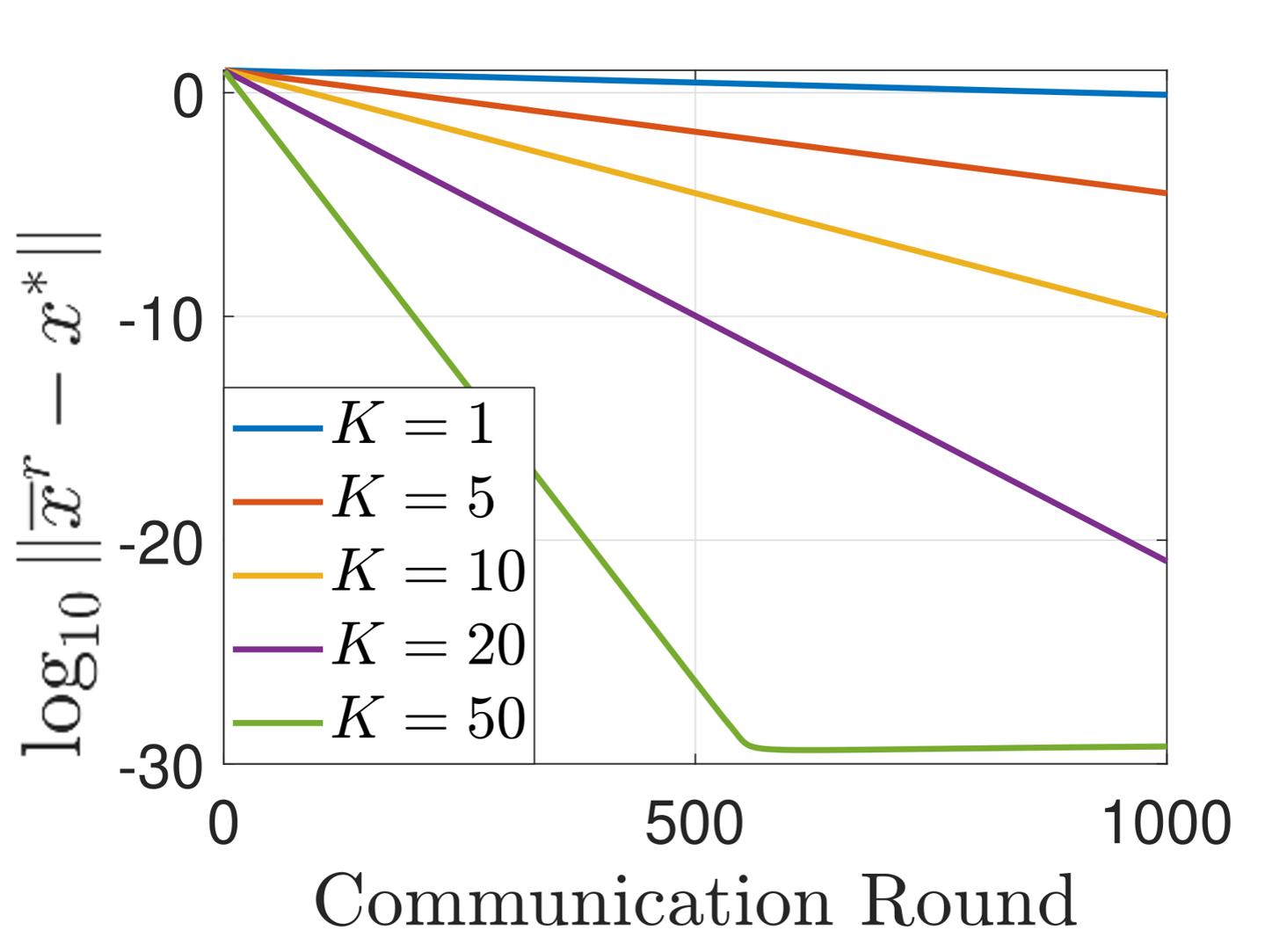}
\end{minipage}
}%
\subfigure[$\rho=0$]{
\begin{minipage}[t]{0.3\linewidth}
\centering
\includegraphics[scale=0.25]{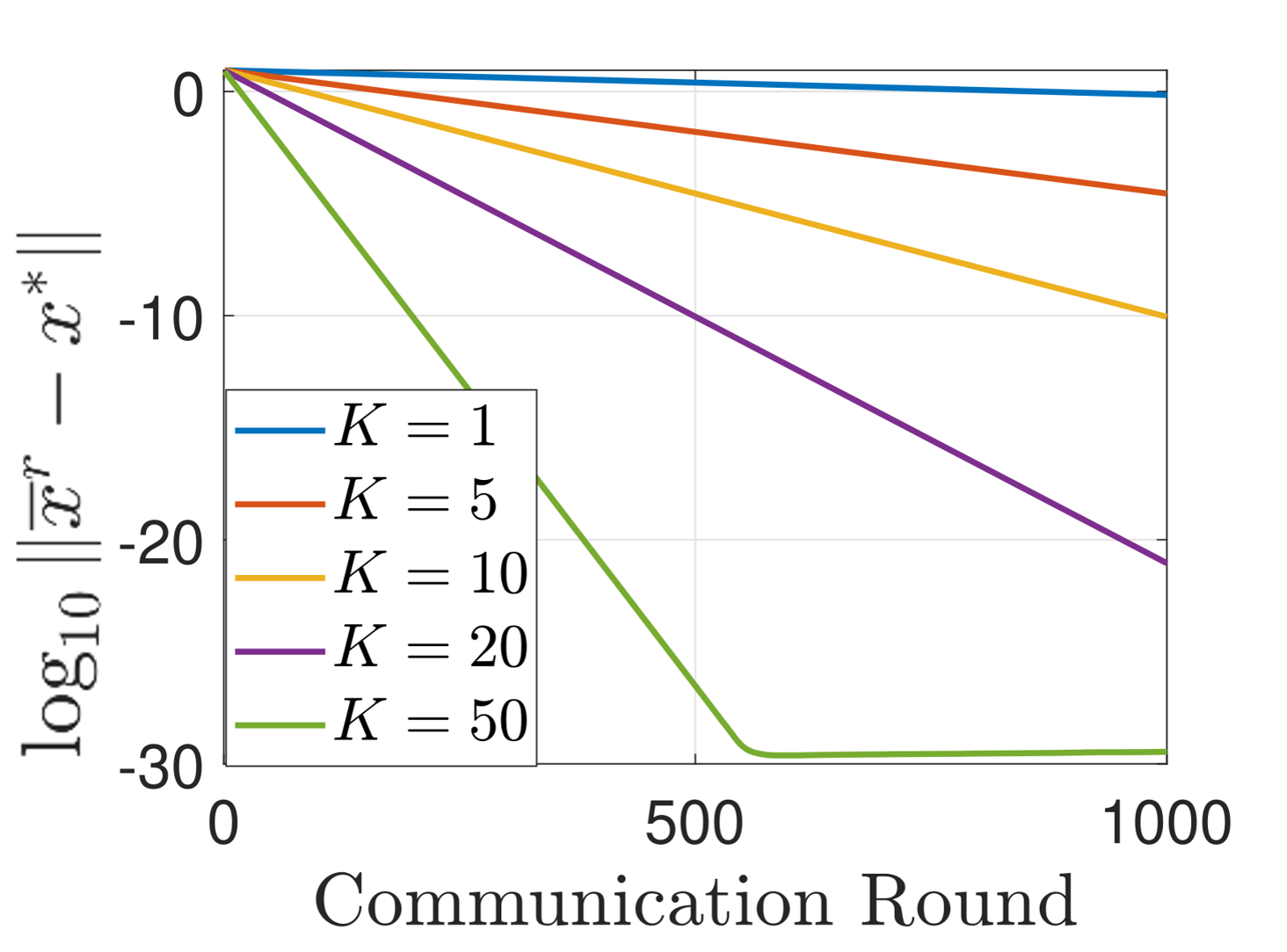}
\end{minipage}
}%
\centering
\caption{Influence of network connectivity under low data heterogeneity. First row: local DGT. Second row:
 local DGD.}
\label{fig9}
\end{figure}

Fig. \ref{fig8} also shows the impact of network connectivity. For local DGT, Eq.~\eqref{eq:complexity-OLS-DGT} shows decreasing $\rho$ reduces the last two terms, thus increasing $K$ would be more efficient (first row). On the other hand, Eq.~\eqref{eq:complexity-OLS} shows {}{decreasing} $\rho$ reduces the second term and improves the communication complexity for all choices of $K$ when it is dominating (second row).

\textbf{Low heterogeneity.}
The communication complexity is shown in Fig. \ref{fig9}. 
For local DGT, either increasing $K$ or decreasing $\rho$ helps reduce communication. For small $\rho$, the effectiveness of increasing $K$ is more significant. This can be explained by~\eqref{eq:complexity-OLS-DGT} for small $\delta$. For local DGD, increasing $K$ helps reduce communication, but reducing $\rho$ does not yield significant improvement. {}{This is because the first term \(\frac{L}{\mu}\) dominates when \(\delta\) is small in equation \eqref{eq:complexity-OLS}.}

Comparing two methods, we see local DGT does not outperform  DGD, and is even worse for poorly connected networks. This numerically reveals that local DGD may be more communication efficient than DGT for problems with small $\delta$ and large $\rho$.

\section{Conclusions}\label{conclu}
We considered two decentralized optimization methods, DGT and DGD, and studied the effectiveness of employing multiple local updates in reducing communication costs. By studying local DGT for strongly convex problems and local DGD in the over-parameterization regime, our result reveals an interesting yet intuitive finding that the convergence rate of the algorithms will be affected by data heterogeneity through the second-order terms when the discrepancy of the first-order terms is eliminated. Through the expressions of the communication complexity, we demonstrated increasing local updates will be advantageous when the datasets are similar and network connectivity is sufficiently good.

\ifCLASSOPTIONcaptionsoff
  \newpage
\fi

\bibliographystyle{IEEEtran}
\bibliography{revised}

\newpage

\section{Appendix}\label{appen}
\subsection{Proof of Lemma \ref{lemma1} }\label{pr-lma1}
Denote $d^i_{r,k} \triangleq \| \boldsymbol x^i_{r,k} - \overline{\boldsymbol x}^r\|^2$ and $F^i_{r,k} \triangleq f(\boldsymbol x^i_{r,k}) - f^\star$.
Based on the iteration of local DGT algorithm, we have
{
\begin{align}
\boldsymbol x_{r,k+1}^i - \overline{\boldsymbol x}^{r} &= \boldsymbol x_{r,k}^i - \overline{\boldsymbol x}^{r} -\eta \Big( \boldsymbol y_{r}^i - \nabla f(\boldsymbol x_{r}^i)\Big)  \nonumber \\
& \quad - \eta  \Big(\nabla f_i(\boldsymbol x_{r,k}^i) - \nabla f(\boldsymbol x_{r,k}^i) \Big)  \nonumber \\
& \quad - \eta \Big( \nabla f_i(\overline{\boldsymbol x}^{r}) - \nabla f(\overline{\boldsymbol x}^{r}) \Big)  \nonumber \\
 & \quad - \eta   \nabla f(\boldsymbol x_{r,k}^i) - \eta  \Big(\nabla f_i(\overline{ \boldsymbol x}^{r})  -  \nabla f(\overline{ \boldsymbol x}^{r})\Big)  \nonumber \\
 & \quad -  \Big( \nabla f_i(\boldsymbol x_{r}^i)  - \nabla f(\boldsymbol x_{r}^i) \Big), 
\end{align}
}
which gives
\begin{align}\label{onestepdrift-lem5}
d^i_{r,k+1} & {\leq} \left(1+\frac 1 {2K} \right)   \left\| \boldsymbol x_{r,k}^i - \overline{\boldsymbol x}^r \right\|^2 + 20K\eta^2  \left\|\boldsymbol y_{r}^i - \nabla f(\boldsymbol x_{r}^i) \right\|^2  \nonumber \\
& \quad + 20K\delta^2\eta^2  \left\| \overline{\boldsymbol x}^r - \boldsymbol x_r^i \right\|^2 + 20K\eta^2 \left\|\nabla f(\boldsymbol  x_{r,k}^i) \right\|^2,
\end{align}
where the inequality uses (i) and (ii) with $a=4K$ in Lemma \ref{trieq}, Assumption~\ref{hessian} and $\eta \leq \frac{1}{4\sqrt{5}\delta K}$. Unrolling the above inequality to the beginning of $r$-th round  yields
\begin{align}\label{onedrift}
d^i_{r,k}  &\leq \left(1+\frac 1 {2K}\right)^k s^i_r + 20K\eta^2 \sum_{j=0}^{k-1} \left( 1+ \frac 1 {2K} \right)^{k-1-j}   \nonumber \\
& \quad \cdot \Big( \left\|\boldsymbol y_r^i   - \nabla f(\boldsymbol x_{r}^i)\right\|^2 + \delta^2 s^i_r  + 20K\eta^2 + \left\|\nabla f(\boldsymbol x_{r,j}^i)\right\|^2 \Big)  \nonumber \\
& \leq 3 s^i_r + 40K\eta^2 (K\left\|\boldsymbol y_r^i - \nabla f(\boldsymbol x_r^i) \right\|^2 + \sum_{j=0}^K\left\|\nabla f(\boldsymbol x_{r,j}^i)\right\|^2),
\end{align}
where the last inequality is due to $(1+\frac 1 {2K})^K \leq \sqrt{e} \leq 2$ and $\eta \leq \frac{1}{4\sqrt{5}\delta K}$. Summing up the drift over $i \in [M]$ and  $k \in [K]$ completes the proof.


\subsection{Proof of Lemma \ref{lemma2} }\label{pf:lem-2}
Applying the descent lemma along the local updates gives
\begin{align}
& f(\boldsymbol x^i_{r,k+1}) = f\Big(\boldsymbol x^i_{r,k} - \eta \Big( \boldsymbol y^i_r + \nabla f_i(\boldsymbol x^i_{r,k}) - \nabla f_i(\boldsymbol x^i_r) \Big)\Big) \nonumber \\
& \leq  f(\boldsymbol x^i_{r,k}) - \eta \Big\langle \nabla f(\boldsymbol x^i_{r,k}), \boldsymbol y^i_r + \nabla f_i(\boldsymbol x^i_{r,k}) - \nabla f_i(\boldsymbol x^i_r) \Big\rangle \nonumber \\
& \quad  \; + \frac L 2 \eta^2  \left\|  \boldsymbol y^i_r + \nabla f_i(\boldsymbol x^i_{r,k}) - \nabla f_i(\boldsymbol x^i_r) \right\|^2 \nonumber \\
& =  f(\boldsymbol x^i_{r,k}) + \frac \eta 2 \left\| \nabla f(\boldsymbol x^i_{r,k}) -\boldsymbol y^i_r   - \nabla f_i(\boldsymbol x^i_{r,k}) + \nabla f_i(\boldsymbol x^i_r) \right \|^2 \nonumber \\
& \quad  - \frac \eta 2 \left( \left\|\nabla f(\boldsymbol x^i_{r,k}) \right\|^2 + \left\| \boldsymbol y^i_r + \nabla f_i(\boldsymbol x^i_{r,k}) - \nabla f_i(\boldsymbol x^i_r) \right\|^2  \right) \nonumber \\
 & \quad  + \frac L 2 \eta^2 \left \|  \boldsymbol y^i_r + \nabla f_i(\boldsymbol x^i_{r,k})   - \nabla f_i(\boldsymbol x^i_r) \right\|^2 \nonumber \nonumber \\
 & =  f(\boldsymbol x^i_{r,k}) - \frac \eta 2 (1-L\eta) \left\|  \boldsymbol y^i_r + \nabla f_i(\boldsymbol x^i_{r,k}) - \nabla f_i(\boldsymbol x^i_r) \right \|^2  \nonumber \\
 & \quad  + \frac \eta 2 \left \| \nabla f(\boldsymbol x^i_{r,k}) -\boldsymbol y^i_r - \nabla f_i(\boldsymbol x^i_{r,k})  + \nabla f_i(\boldsymbol x^i_r) \right \|^2    \nonumber \\
 & \quad   -\frac \eta 2  \left\|\nabla f(\boldsymbol x^i_{r,k}) \right\|^2  \nonumber \\  
 &{=}  f(\boldsymbol x^i_{r,k}) + \frac{\eta}{2} \left\|  \nabla f(\boldsymbol x^i_r) -\boldsymbol y^i_r + \Big( \nabla f(\boldsymbol x^i_{r,k}) - \nabla f_i(\boldsymbol x^i_{r,k}) \Big)  \right. \nonumber \\
 & \quad \;\left. -  \Big( \nabla f(\boldsymbol x^i_r) - \nabla f_i(\boldsymbol x^i_r)\Big) \right\|^2 - \frac{\eta}{2} \left\|\nabla f(\boldsymbol x^i_{r,k})\right\|^2 \nonumber \\
 & \quad \; + \left (-\frac \eta 2+\frac{L\eta^2}{2} \right) \left\|  \boldsymbol y^i_r + \nabla f_i(\boldsymbol x^i_{r,k}) - \nabla f_i(\boldsymbol x^i_r) \right \|^2 \nonumber \\
 & = f(\boldsymbol x^i_{r,k}) + \frac{\eta}{2} \left\| \nabla f(\boldsymbol x^i_r) -\boldsymbol y^i_r + \Big( \nabla f(\boldsymbol x^i_{r,k}) - \nabla f_i(\boldsymbol x^i_{r,k})\Big)  \right. \nonumber \\
 & \quad \; \left. - \Big( \nabla f(\overline{\boldsymbol x}^r) - \nabla f_i(\overline{\boldsymbol x}^r)\Big) + \Big( \nabla f(\overline{\boldsymbol x}^r) - \nabla f_i(\overline{\boldsymbol x}^r)\Big)  \right. \nonumber \\
 & \quad \; \left. - \Big( \nabla f(\boldsymbol x^i_r) - \nabla f_i(\boldsymbol x^i_r)  \Big) \right\|^2 - \frac{\eta}{2} \left\|\nabla f(\boldsymbol x^i_{r,k}) \right\|^2 \nonumber \\
& \quad \; + \left (-\frac \eta 2+\frac{L\eta^2}{2} \right) \cdot \left\|  \boldsymbol y^i_r + \nabla f_i(\boldsymbol x^i_{r,k}) - \nabla f_i(\boldsymbol x^i_r) \right\|^2 \nonumber \\
& \leq f(\boldsymbol x^i_{r,k}) + \frac 3 2 \eta \left\| \nabla f(\boldsymbol x^i_r) -\boldsymbol y^i_r \right \|^2 + \frac 3 2 \delta^2\eta  \left \|\boldsymbol x^i_r - \overline{\boldsymbol x}^r \right\|^2 \nonumber \\
& \quad \; - \frac{\eta}{2}\left\|   \nabla f(\boldsymbol x^i_{r,k}) \right\|^2 + \frac \delta {10(K+1)}  \left\|\boldsymbol x^i_{r,k} - \overline{\boldsymbol x}^r \right\|^2,
\end{align}
where the last inequality is due to (i) in Lemma \ref{trieq}, Assumption~\ref{hessian} and $\eta \leq \min \left\{ \frac 1 L, \frac{1}{30\delta K} \right\} $.

\subsection{Proof of Lemma \ref{lemma3} }\label{pf:lem-3}
Define $E^{r,k}_i := f(\boldsymbol x^i_{r,k}) - f(\bm{x}^\star) + \delta \left(1+\frac 1 {K+1} \right)^{K+1-k} d^i_{r,k}$, $E_i^r \triangleq E_i^{r,0}$. 
Based on Eq.~\eqref{onestepdrift-lem5} in Lemma \ref{lemma1}, we have 
\begin{align}\label{onestepdrift}
\delta & \left( \frac{K+2}{K+1} \right)^{K+1-k} d^i_{r,k} \overset{(a)}{\leq}  \delta \left(\frac{K+2}{K+1}\right)^{K+1-k}\left(\frac{5K+8}{5K+5}\right)  \nonumber \\
& \quad  \cdot d^i_{r,k-1} - \frac{\delta}{10(K+1)} d^i_{r,k-1} + 60K\delta\eta^2 \left\|\boldsymbol y^i_r - \nabla f(\boldsymbol x^i_r) \right\|^2 \nonumber \\
& \quad + 60K\delta^3\eta^2 \cdot s^i_r + 60K\delta\eta^2 \cdot \left\|\nabla f(\boldsymbol x^i_{r,k-1}) \right\|^2 \nonumber \\
& \overset{(b)}{\leq} \left( \left(\frac{5K+4}{5K+5}\right)\delta\left(\frac{K+2}{K+1}\right)^{K-k+2} - \frac{\delta}{10(K+1)} \right)  d^i_{r,k-1}  \nonumber \\ 
& \quad  + 60K\delta\eta^2 \left( \left\|\boldsymbol y^i_r - \nabla f(\boldsymbol x^i_r) \right\|^2 + \left\|\nabla f(\boldsymbol x^i_{r,k-1}) \right\|^2 + \delta^2 s^i_r\right), \nonumber \\
\end{align}
where (a) is based on $\delta \left(1+\frac 1 {K+1} \right)^{K+1-k}\leq 3\delta$ and (b) is due to $\left(1+\frac 1 {2(K+1)} + \frac 1 {10(K+1)}\right)\leq \left(1-\frac 1 {5(K+1)}\right)\left(1+\frac 1 {K+1}\right)$. 

Combining \eqref{desl} in Lemma \ref{lemma2} and \eqref{onestepdrift} yields 
\begin{align}\label{poten}
E^{r,k}_i & \leq \Big( F^i_{r,k-1} + \left(\frac{5K+4}{5K+5}\right)\delta \left(\frac{K+2}{K+1}\right)^{K-k+2} d^i_{r,k-1} \Big)   \nonumber \\
& \quad +  \left(\frac 3 2 \eta +60K\delta\eta^2\right) \left( \left\|\boldsymbol y^i_r - \nabla f(\boldsymbol x^i_r)\right\|^2  + \delta^2 s^i_r\right)   \nonumber \\ 
& \quad + \left (-\frac \eta 2 + 60K\delta\eta^2 \right) \|\nabla f(\boldsymbol x^i_{r,k-1})\|^2  \nonumber \\
& \overset{(a)}{\leq}  \Big( F^i_{r,k-1} + \left(\frac{5K+4}{5K+5}\right)\delta \left(\frac{K+2}{K+1}\right)^{K-k+2} d^i_{r,k-1} \Big)  \nonumber \\
& \quad + 2\eta \delta^2 s^i_r  + 2\eta\left\|\boldsymbol y^i_r - \nabla f(\boldsymbol x^i_r)\right\|^2 - \frac \eta 4 \left\|\nabla f(\boldsymbol x^i_{r,k-1}) \right\|^2 \nonumber \\
& \overset{(b)}{\leq}   \left(1-\frac{\mu\eta}{4}\right)F^i_{r,k-1} + \delta \left(\frac{5K+4}{5K+5} \right) \left(\frac{K+2}{K+1}\right)^{K-k+2}    \nonumber \\
 & \quad \cdot \text {\small $ d^i_{r,k-1} + 2\eta\left\|\boldsymbol y^i_r - \nabla f(\boldsymbol x^i_r) \right\|^2 + 2\delta^2 \eta s^i_r - \frac \eta 8 \left\|\nabla f(\boldsymbol x^i_{r,k-1}) \right\|^2$ } \nonumber \\
 &  \overset{(c)}{\leq}  \left(1-\frac{\mu\eta}{4}\right)  E^{r,k-1}_i + 2\eta \left\|\boldsymbol y^i_r - \nabla f(\boldsymbol x^i_r)\right\|^2 + 2\delta^2\eta s^i_r\nonumber \\
 & \quad - \frac \eta 8 \left\|\nabla f(\boldsymbol x^i_{r,k-1})\right\|^2,
\end{align}
where (a) is due to $\eta \leq \frac{1}{240\delta K}$ such that $60K\delta \eta^2 \leq \frac \eta 4$ and (c) is due to $\eta \leq \frac{1}{5\mu K}$ that gives $1-\frac{1}{5(K+1)} \leq 1-\frac{\mu \eta}{4}$, and (b) is due to the strong convexity of the average loss $f$ that  $\|\nabla f(\boldsymbol x^i_{r,k}) \|^2 \geq 2\mu (f(\boldsymbol x^i_{r,k}) - f(\boldsymbol x^\star))$. 

Unrolling the recursion of  $E_i^{r,k}$ along local updates and summing up across all agents $i$, we obtain
\begin{align}\label{finaldescent}
& F_{r+1} \overset{(a)}{\leq} \frac 1 m \sum_{i=1}^m{E^{r,K+1}_i} \nonumber \\
 &\overset{(b)}{\leq} \Big(1 - \frac{\mu\eta}{4}\Big)^{K+1}  \frac 1 m \sum_{i=1}^m{E^r_i} + 2\eta\sum_{k=0}^K{\left( 1-\frac{\mu\eta} 4 \right)^{K-k}} \nonumber \\ 
 & \quad \cdot ( \Gamma_r +  \delta^2 S_r) - \frac \eta 8 \cdot \frac{\sum_{i=1}^m \sum_{k=0}^K{(1-\frac{\mu\eta}{4})^{K-k} \left\|\nabla f(\boldsymbol x^i_{r,k}) \right\|^2}} m \nonumber \\
 & \overset{(c)}{\leq} \Big(1 - \frac{\mu\eta}{4}\Big)^{{}{(K+1)}}\frac 1 m \sum_{i=1}^m{\Big(f(\overline{\boldsymbol x}^r) + \Big\langle \nabla f(\overline{\boldsymbol x}^r), \boldsymbol x^i_r - \overline{\boldsymbol x}^r\Big\rangle - f(\boldsymbol x^\star) }\nonumber \\
 & \quad { + \frac{L+6\delta}2 s^i_r \Big)} - \frac \eta 8 \cdot \frac{\sum_{i=1}^m \sum_{k=0}^K{(1-\frac{\mu\eta}{4})^{K-k} \left\|\nabla f(\boldsymbol x^i_{r,k}) \right\|^2}} m \nonumber \\ 
 & \quad + 4K\eta (\Gamma_r + \delta^2 S_r) \nonumber \\
 & \overset{(d)}{\leq} \left(1-\frac{\mu{}{(K+1)} \eta}{8} \right) F_r + \left(4\delta^2K\eta + 3(L+\delta)\right) S_r \nonumber \\
 & \quad + \frac{K\eta(64\Gamma_r-G_r)}{16},
 \end{align}
 where (a) is due to the convexity of $f$ and the identity $\overline{\boldsymbol x}^{r+1} = \overline{\boldsymbol x}_{r,K+1}$, (b) is due to \eqref{poten}, and (c) is based on $(1+\frac{1}{K+1})^{K+1}\leq 3$ the $L$-smooth of  $f$. The last inequality (d) is due to $\eta \leq \frac{1}{5\mu K} < \frac 2 {\mu K}$ and Bernoulli inequality, which implies   $\frac{1}{(1-\frac{\mu\eta}{4})^{-K}} \leq \frac{1}{1+\frac{K\mu\eta}{4}} \leq 1-\frac{K\mu\eta}{8}$ under $\eta \leq \frac{1}{5\mu K}$.

\subsection{Proof of Lemma \ref{lemma4} }\label{pf:lem-4}
We first rewrite the update \eqref{reiter} as 
\begin{align}
&\boldsymbol X_{r,K} = \boldsymbol X_r - \eta \Big( K \left(\boldsymbol Y_r -  \nabla  \boldsymbol F(\boldsymbol X_r) \right) + \sum_{k=0}^{K-1}{\nabla \boldsymbol F(\boldsymbol X_{r,k})} \Big) \nonumber \\
& \boldsymbol Y_{r,K} = \boldsymbol Y_r + \nabla \boldsymbol F(\boldsymbol X_{r,K}) - \nabla \boldsymbol F(\boldsymbol X_r).
\end{align}
Consequently we have the iteration of consensus error as
\begin{align}
\boldsymbol X_r - \overline{\boldsymbol X}^r &=  \boldsymbol X_r (\boldsymbol W - \boldsymbol J) -\eta \sum_{k=0}^K\nabla \boldsymbol F(\boldsymbol X_{r,k}) (\boldsymbol W - \boldsymbol J) \nonumber \\
& - \eta (K+1) (\boldsymbol Y_r -  \nabla \boldsymbol F(\boldsymbol X_r)) (\boldsymbol W - \boldsymbol J).
\end{align}
Thus, there is
\begin{align}
S_{r+1} &\leq  \rho^2  \left(1+\frac 1 \alpha \right) \cdot S_r  +  \rho^2 (1+\alpha)\eta^2  \  \nonumber \\ 
 & \quad \cdot \left\|(K+1) \left(\boldsymbol Y_r  - \nabla \boldsymbol F(\boldsymbol X_r) \right) + \sum_{k=0}^K{\nabla \boldsymbol F(\boldsymbol X_{r,k})} \right\|^2   \nonumber \\ 
 & =  \rho^2 \left(1+\frac 1 \alpha \right) \cdot S_r + \rho^2(1+\alpha)\eta^2 \sum_{i=1}^m \Big\|(K+1) \nonumber \\
 & \quad \cdot \left(\boldsymbol y^i_r - \nabla f(\boldsymbol x^i_r) \right) + \sum_{k=0}^K \nabla f(\boldsymbol x^i_{r,k}) + (K+1)  \nonumber \\
 & \quad  \cdot  ( \nabla f(\boldsymbol x^i_r) - \nabla f_i(\boldsymbol x^i_r)  - ( \nabla f(\overline{\boldsymbol x}^r)  - \nabla f_i(\overline{\boldsymbol x}^r) ) ) \nonumber \\
 & \quad  + \sum_{k=0}^K  \nabla f(\overline{\boldsymbol x}^r) - \nabla f_i(\overline{\boldsymbol x}^r)  -  \nabla f(\boldsymbol x^i_{r,k})  + \nabla f_i(\boldsymbol x^i_{r,k}) \Big\|^2. 
 \end{align}
Applying (ii) in Lemma \ref{trieq} and Assumption \ref{hessian} leads to
\begin{align}\label{conerrs}
& S_{r+1} \leq \rho^2\left(1+\frac 1 \alpha \right) \cdot S_r + \rho^2(1+\alpha)\sum_{i=1}^m 4\eta^2  \nonumber \\
& \cdot \left( (K+1)^2\left(\|\boldsymbol y^i_r - \nabla f(\boldsymbol x^i_r) \|^2 + \delta^2 s^i_r\right)  + \left \|\sum_{k=0}^K \nabla f(\boldsymbol x^i_{r,k}) \right\|^2   \right) \nonumber \\
& + 4\eta^2\Big\| \sum_{k=0}^K  \nabla f(\overline{\boldsymbol x}^r) - \nabla f_i(\overline{\boldsymbol x}^r) - ( \nabla f(\boldsymbol x^i_{r,k}) - \nabla f_i(\boldsymbol x^i_{r,k}) )  \Big\|^2    \nonumber \\
  & \leq \left( \rho^2 \left(1+\frac 1 \alpha \right) + 4\rho^2\delta^2(1+\alpha)(K+1)^2\eta^2 \right) S_r \nonumber \\
 & \quad +4\rho^2(1+\alpha)(K+1)\eta^2  \sum_{i=1}^m \sum_{k=0}^K {\left\|\nabla f(\boldsymbol x^i_{r,k}) \right\|^2} \nonumber \\
 & \quad {}{+}  4\rho^2(1+\alpha)(K+1)\eta^2 \left( m(K+1)\Gamma_r + \delta^2 D_r \right).
 \end{align}
To guarantee the contraction of the consensus error, we choose  $\alpha = \frac 4 {1-\rho}$. Then  based on $\eta \leq \frac{1-\rho}{8\sqrt{5} \delta K} $, we have $\rho^2\Big(1+\frac 1 \alpha + 4\delta^2(1+\alpha)(K+1)^2\eta^2\Big) \leq \frac {\rho(1+\rho)} 2$ and \eqref{conerrs} becomes 
\begin{align}\label{eq:bound-for-tracking-err}
 S_{r+1} & \overset{(a)}{\leq} \left(\frac{\rho(1+\rho)}{2} + \frac{240\rho^2\delta^2K^2\eta^2}{1-\rho}\right) S_r + \left( \Gamma_r + G_r \right) \nonumber\\
 & \quad \cdot \frac{3200m\rho^2\delta^2K^4\eta^4 + 80m\rho^2K^2\eta^2}{1-\rho} \nonumber \\
 & \overset{(b)}{\leq} \Big(\frac{1+\rho} 2\Big) \cdot S_r + \frac{90m\rho^2K^2\eta^2}{1-\rho} \left( \Gamma_r + G_r \right),
 \end{align}
 where (a) is due to  \eqref{finaldrift} in Lemma \ref{lemma1} and  (b) is due to $\eta\leq \frac{1-\rho}{8\sqrt{5}\delta K} \leq \frac{1-\rho}{4\sqrt{15}\delta K}$ such that $\frac{240\rho^2\delta^2K^2\eta^2}{1-\rho} \leq \frac{(1-\rho)(1+\rho)}{2}$ and $3200\rho^2\delta^2K^4\eta^4 \leq 10\rho^2K^2\eta^2$.

 \subsection{Proof of Lemma \ref{lemma5} }\label{pf:lem-5}
 The recursion for average of the local copies follows
 \begin{align}
 & \overline{\boldsymbol x}^{r+1} - \overline{\boldsymbol x}^r = \overline{\boldsymbol x}_{r,K} - \eta \overline{\boldsymbol y}_{r,K} - \overline{\boldsymbol x}^r \nonumber \\
 & \quad =\frac{\eta}{m} \sum_{i=1}^m \left(  (K+1) \left( \boldsymbol y^i_r - \nabla f_i(\boldsymbol x^i_r)\right) + \sum_{k=0}^K \nabla f_i(\boldsymbol x^i_{r,k})      \right)  \nonumber  \\
& \quad = \text { \small $\frac \eta m \sum_{i=1}^m   \sum_{k=0}^K  \left(  \nabla f_i(\boldsymbol x^i_{r,k}) - \nabla f(\boldsymbol x^i_{r,k})  - (\nabla f_i(\overline{\boldsymbol x}^r) - \nabla f(\overline{\boldsymbol x}^r)) \right) $ }\nonumber \\
& \quad \quad     + \frac \eta m \sum_{i=1}^m \sum_{k=0}^K \nabla f(\boldsymbol x^i_{r,k}).  
 \end{align}
 Thus, there is
\begin{align}\label{diffav}
& \left\|\overline{\boldsymbol x}^{r+1} - \overline{\boldsymbol x}^r \right\|^2 
{\leq} \frac{2(K+1)\eta^2}{m} \cdot \sum_{i=1}^m   \sum_{k=0}^K{ \left\|\nabla f(\boldsymbol x^i_{r,k}) \right\|^2}  \nonumber \\
& \quad  +  \left\|(\nabla f_i(\boldsymbol x^i_{r,k}) - \nabla f(\boldsymbol x^i_{r,k}) )  - (\nabla f_i(\overline{\boldsymbol x}^r) - \nabla f(\overline{\boldsymbol x}^r)) \right\|^2   \nonumber \\
& \quad  {\leq} \frac{2(K+1)\eta^2}{m} \left( \delta^2 D_r +  \sum_{i=1}^m\sum_{k=0}^K \left\|\nabla f(\boldsymbol x^i_{r,k}) \right\|^2 \right).
\end{align}


 
\subsection{Proof of Lemma \ref{lemma6} }\label{pf:lem-6}
Based on the iteration \eqref{reiter}, we have $\boldsymbol Y_{r+1} = \Big(\boldsymbol Y_r + \nabla \boldsymbol F(\boldsymbol X_{r+1}) - \nabla \boldsymbol F(\boldsymbol X_r) \Big) \boldsymbol W$. Then the tracking error has 
\begin{align}\label{tempY}
&\boldsymbol Y_r - \nabla f\left( \boldsymbol X_r\right){=} (\boldsymbol Y_r - \nabla \boldsymbol f(\boldsymbol X_r))(\boldsymbol W - \boldsymbol J)   \nonumber \\
& \quad  + [  \nabla \boldsymbol F(\boldsymbol X_r)  - \nabla \boldsymbol f(\boldsymbol X_r)  -  \nabla \boldsymbol F(\overline{\boldsymbol X}^r) + \nabla \boldsymbol f(\overline{\boldsymbol X}^r))] \boldsymbol J \nonumber \\
& \quad  + \text{\small $[  \nabla \boldsymbol F(\boldsymbol X_{r+1})  - \nabla \boldsymbol f(\boldsymbol X_{r+1})  -  \nabla \boldsymbol F(\overline{\boldsymbol X}^{r+1}) + \nabla \boldsymbol f(\overline{\boldsymbol X}^{r+1})  ] \boldsymbol W $} \nonumber \\
& \quad  + [ \nabla \boldsymbol F(\overline{\boldsymbol X}^{r+1}) - \nabla \boldsymbol f(\overline{\boldsymbol X}^{r+1})  -  \nabla \boldsymbol F(\boldsymbol X_r) + \nabla \boldsymbol f(\boldsymbol X_r)] \boldsymbol W \nonumber \\
& \quad   + \nabla \boldsymbol f(\boldsymbol X_{r+1}) (\boldsymbol W - \boldsymbol I) \nonumber \\
& = (\boldsymbol Y_r - \nabla \boldsymbol f(\boldsymbol X_r)) (\boldsymbol W - \boldsymbol J)  \nonumber \\
& \quad + [ \nabla \boldsymbol F(\boldsymbol X_r)  - \nabla \boldsymbol f(\boldsymbol X_r)  -  \nabla \boldsymbol F(\overline{\boldsymbol X}^r) + \nabla \boldsymbol f(\overline{\boldsymbol X}^r) ](\boldsymbol J - \boldsymbol W)  \nonumber \\
& \quad + \text{\small $[   \nabla \boldsymbol F(\boldsymbol X_{r+1})  - \nabla \boldsymbol f(\boldsymbol X_{r+1})  -  \nabla \boldsymbol F(\overline{\boldsymbol X}^{r+1}) + \nabla \boldsymbol f(\overline{\boldsymbol X}^{r+1})  ] \boldsymbol W $}\nonumber \\
 &\quad + [  \nabla \boldsymbol F(\overline{\boldsymbol X}^{r+1}) - \nabla \boldsymbol f(\overline{\boldsymbol X}^{r+1})  -  \nabla \boldsymbol F(\overline{\boldsymbol X}^r) + \nabla \boldsymbol f(\overline{\boldsymbol X}^r) ]  \nonumber \\
& \quad  \cdot  (\boldsymbol W -  \boldsymbol J)  + \nabla \boldsymbol f(\boldsymbol X_{r+1}) (\boldsymbol W - \boldsymbol I) \nonumber \\
& m \Gamma_r  \overset{(b)}{\leq} m\rho^2\left(1+\frac 1 \beta \right) \Gamma_r  +  4m\rho^2\delta^2(1+\beta) \cdot  \|\overline{\boldsymbol X}^{r+1} - \overline{\boldsymbol X}^r \|^2 \nonumber \\
& \quad + 4(\delta^2+L^2) (1+\beta)  S_{r+1} + 4m\rho^2\delta^2(1+\beta) S_r, 
\end{align}
where (a) holds {since $(\boldsymbol Y_r - \nabla \boldsymbol F(\boldsymbol X_r))\boldsymbol J = \bm{0}$ from the last equation in \eqref{eq:avg-local-GT} and $( \nabla \boldsymbol F(\overline{\boldsymbol X}^r) - \nabla \boldsymbol f(\overline{\boldsymbol X}^r) ) \bm{J} = \bm{0}$}, and (b) is due to Assumption \ref{W} and Lemma \ref{trieq}. To guarantee the contraction of tracking error, we set $\beta = \frac 2 {1-\rho}$ such that $\rho^2 \left(1+\frac 1 \beta \right) \leq \rho \left(\frac{1+\rho}2\right)$. Step size condition that $\eta \leq \frac{(1-\rho)^2}{64\rho\sqrt{\delta^2+L^2}K}$ results in following recursive bound
\begin{align}
\Gamma_{r+1} &\leq \frac {\left(1+\rho\right)^2} 4 \Gamma_r  + \frac{24\rho(\delta^2+L^2)}{1-\rho}  S_r \nonumber \\
& \quad  + \frac{528\rho^2\delta^2(\delta^2+L^2)K\eta^2}{(1-\rho)^2} \cdot \left ( \frac{D_r}{m} + 2 G_r  \right) \nonumber \\
& {\leq} \Big(\frac{1+\rho}{2}\Big) \Gamma_r + \frac{\rho(\delta^2+L^2)}{1-\rho} \left(  48 S_r + \frac{1188\rho K^2\eta^2}{1-\rho} G_r\right).
\end{align}

\subsection{Proof of Lemma \ref{lemma15}}\label{pf:lem-15}
\vspace{-7mm}
\begin{align}
\frac{1}{m} \sum_{i=1}^m & \| \boldsymbol x^i_{r,K+1} - \overline{\boldsymbol x}_{r,K+1} \|^2  \nonumber \\
& =  \frac{1}{m} \sum_{i=1}^m{ \left\| \boldsymbol x^i_{r,K+1} - \overline{\boldsymbol x}_{r+1}  \right\|^2} \nonumber \\
& \overset{(a)}{\leq} \frac{2}{m} \sum_{i=1}^m {\left\| \boldsymbol x^i_{r,K+1} - \overline{\boldsymbol x}_r \right\|^2} + 2 \left\| \overline{\boldsymbol x}_{r+1} - \overline{\boldsymbol x}_r \right\|^2 \nonumber \\
& \overset{(b)}{\leq} 6 S_r + {320K^2\eta^2} \cdot \left( \Gamma_r + G_r \right) +\frac{8K\delta^2\eta^2}{m} \cdot D_r \nonumber \\
& \quad + 16K^2\eta^2 \cdot  G_r\nonumber \\
&\overset{(c)}{\leq} \left( 6 + 48K^2\delta^2\eta^2 \right) \cdot S_r + 336 
K^2 \eta^2 \cdot ( \Gamma_r + G_r ) \nonumber \\
& \quad + 640\delta^2K^4\eta^4 \cdot (\Gamma_r + G_r ).
\end{align}
The first equality is due to the double stochastic property of $\boldsymbol W$ in Assumption \ref{W}. (a) is due to (i) of Lemma \ref{trieq}, (b) is due to \eqref{onedrift} and Lemma \ref{lemma5} and (c) is based on Lemma \ref{lemma1}.

\subsection{Proof of Lemma \ref{lemma16}}\label{pf:lem-16}
\vspace{-6mm}
\begin{align}
F_{r+1} & \overset{(a)}{\leq} \frac 1 m \sum_{i=1}^m{E^{r,K+1}_i} + \frac \beta {2m} \sum_{i=1}^m {\left\|  \boldsymbol x^i_{r,K+1} - \overline{\boldsymbol x}_{r,K+1}  \right\|^2} \nonumber \\
& \overset{(b)}{\leq} \left(1-\frac{\mu (K+1) \eta}{8}\right) \cdot F_r   \nonumber \\
& \quad + \left(4K\eta + 320 \beta K^2\eta^2 + 640 \beta\delta^2 K^4 \eta^4 \right) \cdot \Gamma_r \nonumber \\
& \quad + \left( 4\delta^2K\eta + 6(L+\delta+\beta) + 48\beta\delta^2K^2 \eta^2 \right) \cdot S_r \nonumber \\
&  \quad \text{ \small $-\left(\frac{\eta K}8 - \frac{\mu K^2\eta^2}{16} - 336\beta K^2\eta ^2 - 640 \beta \delta^2K^4 \eta^4 \right) G_r$ } \nonumber \\
& \leq \left( 1 -\frac{\mu (K+1)\eta}{8}\right) \cdot F_r + 8K\eta \cdot  \Gamma_r  \nonumber \\
& \quad + \left( 8\delta^2K\eta + 6 \left( L+\delta+\beta \right) \right) \cdot S_r - \frac{\eta K}{16} \cdot G_r,
\end{align}
(a) is based on the property of weakly convex functions as Lemma \ref{weakly-con}
 and (b) is substituting the result of Lemma \ref{lemma15}. The last inequality is based on the step size condition that $\eta \lesssim \min \left\{ \frac{1-\rho}{K\delta}, \frac{1-\rho}{K\beta} \right\}$.

\subsection{Proof of Lemma \ref{lemma11} }\label{pf:lem-7}
Denote $\pi(\boldsymbol x^i_{r,K})$ as the projection of $\boldsymbol x^i_{r,K}$ onto $\boldsymbol{\mathcal X}^\star$.
\begin{align}\label{process_con}
&S_{r+1} = \frac 1  m\|\boldsymbol X_{r,K+1}\boldsymbol W(\boldsymbol I - \boldsymbol J)\|^2 = \frac 1  m\|\boldsymbol X_{r,K+1}(\boldsymbol W - \boldsymbol J)\|^2 \nonumber \\
& \leq \frac{\rho^2} m  \|\boldsymbol X_{r,K+1}(\boldsymbol I - \boldsymbol J)\|^2  = \rho^2S_{r,K+1} \nonumber \\
& = \frac {\rho^2} m \left\|\Big( \boldsymbol X_{r,K} - \eta \boldsymbol F(\boldsymbol X_{r,K})\Big) (\boldsymbol I - \boldsymbol J) \right\|^2 \nonumber \\
& \overset{(a)}{\leq} \left(1+ \frac 1 {\epsilon_2}\right)\rho^2 S_{r,K} + \frac{1+\epsilon_2}{m} \rho^2 \eta^2 \sum_{i=1}^m {\|\nabla f_i(\boldsymbol x^i_{r,K})\|^2 } \nonumber \\
& =  \frac{1+\epsilon_2}{m} \rho^2 \eta^2 \sum_{i=1}^m  \left\|\nabla f_i(\boldsymbol x^i_{r,K})  - \nabla f(\boldsymbol x^i_{r,K}) + \nabla f(\boldsymbol x^i_{r,K}) \right\|^2 \nonumber \\
& \quad + \left(1+\frac 1 {\epsilon_2}\right) \rho^2 S_{r,K} \nonumber \\
& \overset{(b)}{\leq} \left( 1+\frac 1 {\epsilon_2} \right)\rho^2 S_{r,K} + \frac{2(1+\epsilon_2)}{m} \rho^2\eta^2 \sum_{i=1}^m  \|\nabla f(\boldsymbol x^i_{r,K})\|^2  \nonumber \\
& \; + \text{ \small $\left\| \nabla f_i(\boldsymbol x^i_{r,K})  - \nabla f(\boldsymbol x^i_{r,K}) - \nabla f_i( \pi(\boldsymbol x^i_{r,K}) ) + \nabla f( \pi(\boldsymbol x^i_{r,K})  ) \right\|^2$ }  \nonumber \\
& \; {\leq} \frac{2(1+\epsilon_2)}{m} \rho^2 \eta^2 \sum_{i=1}^m  \delta^2 \|\boldsymbol x^i_{r,K}  - \pi(\boldsymbol x^i_{r,K}) \|^2   + \|\nabla f(\boldsymbol x^i_{r,K})\|^2  \nonumber \\ 
& \quad  +  \left(1+\frac 1 {\epsilon_2} \right)\rho^2 S_{r,K} \nonumber \\
& \overset{(c)}{\leq} \left(1+\frac 1 {\epsilon_2}\right)\rho^2 S_{r,K} + 2 (1+\epsilon_2)\left(1+\frac {\delta^2}{\mu^2}\right) \rho^2\eta^2 \cdot \nonumber \\
& \quad \frac{\sum_{i=1}^m \|\nabla f(\boldsymbol x^i_{r,K})\|^2 } m \nonumber \\
& \overset{}{\leq} \rho^2 \left(1+\frac 1 {\epsilon_2}\right)^{K+1} S_r + 2\rho^2\left(1+\epsilon_2 \right)\left(1+\frac{\delta^2}{\mu^2}\right)\eta^2 \cdot \nonumber \\
& \quad \sum_{k=0}^K \left(1+\frac 1 {\epsilon_2}\right)^{K-k}  \frac{\sum_{i=1}^m {\|\nabla f(\boldsymbol x^i_{r,k})\|^2 }} m  \nonumber \\
& \overset{(d)}{\leq} \Big( \frac{1+\rho} 2 \Big) \cdot S_r + \frac{64\rho^2K^2(\delta^2+\mu^2)\eta^2}{(1-\rho)\mu^2} \cdot G_r,
\end{align}
where  (a) is due to (i) of Lemma \ref{trieq} and (b) is due to  Assumption \ref{overpara} that the local loss functions and the average loss have the same zero gradients at optimal solutions. (c) is due to Lemma \ref{PL-EB-QG}, which states that PL implies the EB condition.  The last inequality  (d) is by setting $\epsilon_2 = \frac 1 {\ln \left(1+\frac{1-\rho}{2} \right)} (K+1)$, which gives $\left(  1+\frac{1}{\epsilon_2} \right)^{K+1} = \left(1+\frac{\ln \left (1+\frac{1-\rho}{2} \right)}{K+1} \right)^{K+1} \leq 1+\frac{1-\rho}{2}$ and $\rho\left(1+\frac{1-\rho}{2} \right) \leq 1- \frac{1-\rho}{2}$. Finally, $\left( 1+\frac{1-\rho}{2} \right)\left(1+\frac{K+1}{\ln(1+\frac{1-\rho}{2})}\right) \leq \frac{16K}{1-\rho}$, which is based on Jensen inequality that $\frac{x}{1+x} \leq \ln(1+x)$, for $x>-1$.

\subsection{Proof of Lemma \ref{lemma12} }\label{pf:lem-8}
\vspace{-6mm}
\begin{align}\label{f_step}
&F^i_{r,k+1} \leq f(\boldsymbol x_{r,k}^i - \eta \nabla f_i(\boldsymbol x_{r,k}^i)) -f^\star \nonumber \\
& \quad  \leq F^i_{r,k} - \eta \langle  \nabla f(\boldsymbol x_{r,k}^i), \nabla f_i(\boldsymbol x_{r,k}^i)\rangle  + \frac L 2 \eta^2 \|\nabla f_i(\boldsymbol x_{r,k}^i)\|^2 \nonumber \\
& \quad = F^i_{r,k} + \frac{\eta}{2} \|\nabla f(\boldsymbol x_{r,k}^i) -\nabla f_i(\boldsymbol x_{r,k}^i)\|^2  \nonumber \\
& \quad \quad - \frac \eta 2 \|\nabla f(\boldsymbol x_{r,k}^i)\|^2 - \frac \eta 2 \|\nabla f_i(\boldsymbol x_{r,k}^i)\|^2  + \frac L 2 \eta^2 \|\nabla f_i(\boldsymbol x^i_{r,k})\|^2 \nonumber \\
& \quad = F^i_{r,k} - \frac{\eta}{2} \|\nabla f(\boldsymbol x_{r,k}^i)\|^2 -\frac {\eta\left(  1-L\eta \right)} 2 \|\nabla f_i(\boldsymbol x_{r,k}^i)\|^2 + \frac \eta 2   \nonumber \\
& \quad \; \cdot \text{ \small $\|(\nabla f(\boldsymbol x_{r,k}^i) - \nabla f_i(\boldsymbol x_{r,k}^i))  - (\nabla f( \pi(\boldsymbol x^i_{r,k} ) ) - \nabla f_i( \pi(\boldsymbol x^i_{r,k} ))\|^2$ }   \nonumber \\
& \quad \overset{(a)}{\leq} F^i_{r,k} + \frac \eta 2 \delta^2 \|\boldsymbol x_{r,k}^i -{\pi(\boldsymbol x^i_{r,k}) }\|^2 - \frac \eta 2 \|\nabla f(\boldsymbol x_{r,k}^i)\|^2 \nonumber \\
& \quad \overset{(b)}{\leq}F^i_{r,k} + \left(\frac \eta 2 \frac{\delta^2}{\mu^2} - \frac{\eta}{2} \right)\|\nabla f(\boldsymbol x_{r,k}^i)\|^2 \nonumber \\
& \quad = F^i_{r,k} - 2\times \frac {\eta\zeta}  4 \|\nabla f(\boldsymbol x_{r,k}^i)\|^2  \nonumber \\
& \quad \overset{(c)}{\leq}  \left( 1- \frac{\eta \mu\zeta}{2}\right) F^i_{r,k}  - \frac {\eta\zeta} 4  \|\nabla f(\boldsymbol x_{r,k}^i)\|^2,  
\end{align}
where (a) is due to $\eta \leq \frac 1 L$ and   Assumption~\ref{rehessian}. (b) and (c) are due to the EB and PL condition, respectively. Then based on the $\beta$-weakly convex and $L$-smooth property of $f$, there is 
\begin{align}\label{dgddescent}
&F_{r+1} \overset{(a)}{\leq} \frac{\sum_{i=1}^m F^i_{r,K+1}}{m}  + \frac \beta {2m^2} \sum_{i=1}^{m-1}\sum_{j=i+1}^m \|\boldsymbol x^i_{r,K+1}  - \boldsymbol x^j_{r,K+1}\|^2  \nonumber \\
& \overset{(b)}{\leq} \frac{\sum_{i=1}^m F^i_{r,K+1} }{m} + 2\beta  S_r + \frac{32\beta K(\delta^2+\mu^2)\eta^2}{(1-\rho)\mu^2}  \nonumber \\
& \quad \cdot  \frac{\sum_{k=0}^K\sum_{i=1}^m \|\nabla f(\boldsymbol x^i_{r,k})\|^2}{m}  \nonumber \\
& \overset{(c)}{\leq}  \left( 1 - \frac{\eta\mu\zeta}{2}\right)^{K+1} \frac 1 m \sum_{i=1}^mf(\overline{\boldsymbol x}^r) - f^\star + \langle \nabla f(\overline{\boldsymbol x}^r), \boldsymbol x_r^i - \overline{\boldsymbol x}^r \rangle \nonumber \\
& \quad    +  \text{\small $\frac L 2\|\boldsymbol x_r^i - \overline{\boldsymbol x}^r\|^2 - \frac {\zeta\eta} 4 \left(1-\frac{\eta\mu\zeta}2\right)^K  \frac{\sum_{i=1}^m\sum_{k=0}^K\|\nabla f(\boldsymbol x_{r,k}^i)\|^2}{m} $}\nonumber \\
& \quad  + 2\beta  S_r + \frac{32 \beta K(\delta^2+\mu^2)\eta^2}{(1-\rho)\mu^2} \cdot   \frac{\sum_{k=0}^K\sum_{i=1}^m \|\nabla f(\boldsymbol x^i_{r,k})\|^2}{m}  \nonumber \\
& \overset{(d)}{\leq} \left(1-\frac{\mu\zeta}{2}\right)^{{}{K+1}}  F_r  - \left(\frac{ \eta\zeta}{4} - \frac{\mu K\eta^2\zeta^2} 8 - \frac{32K\beta(\delta^2+\mu^2)\eta^2}{(1-\rho)\mu^2} \right) \nonumber \\
& \quad \cdot \frac{\sum_{k=0}^K\sum_{i=1}^m\|\nabla f(\boldsymbol x^i_{r,k})\|^2}{m} + \left( \frac{ L}{2} + 2\beta \right) S_r \nonumber \\
& \overset{(e)}{\leq}  \left(1- \frac{\eta\mu\zeta}{2}\right)^{{}{K+1}} F_r - \frac {\zeta K\eta} {8}    \frac{\sum_{i=1}^m\sum_{k=0}^K{\|\nabla f({}{\boldsymbol x^i_{r,k}})\|^2}}{m(K+1)} \nonumber \\
& \quad + 2 (L+\beta)  S_r \nonumber \\
& \overset{(f)}{\leq} \left( 1 -\frac{{{}{(K+1)}}\eta\mu\zeta}{4}\right)   F_r  - \frac {\zeta K\eta} {8}  \frac{\sum_{i=1}^m\sum_{k=0}^K{\|\nabla f(\boldsymbol x_{r,k}^i)\|^2}}{m(K+1)} \nonumber \\
& \quad + 2(L+\beta)S_r,
\end{align}
where (a) is due to  Lemma \ref{weakly-con},  (b) is obtained by substituting into \eqref{dgd-con}, 
(c) is due to unrolling the recursion in \eqref{f_step} along the local updates, $L$-smooth of $f$ and
\begin{align}\label{total_des}
& \frac{\sum_{i=1}^m F^i_{r,K+1} }{m} \leq \left( 1-\frac{\eta\mu\zeta}{2}\right)^{K+1} \cdot \frac{\sum_{i=1}^m  F^i_r }{m}  - \frac{\zeta\eta}{4} \cdot \nonumber \\
& \quad \frac{\sum_{i=1}^m{\sum_{k=0}^K{\left(1-\frac{\eta\mu\zeta}{2}\right)^{K-k}\|\nabla f(\boldsymbol x_{r,k}^i)\|^2}}}{m}.
\end{align}
\noindent (d) is based on Bernoulli inequality that $\left(1-\frac{\eta\mu\zeta}{2}\right)^K \geq 1 - \frac{K\eta\mu\zeta}{2} $ and $\eta \leq \frac{1}{2K \mu\zeta}$, (e) is due to $\eta \leq \frac{1}{2K \mu\zeta}$ and $\eta \leq \frac{(1-\rho)(\mu^2 - \delta^2)}{512K\beta(\delta^2 + \mu^2)}$ such that $\frac{\mu K\zeta^2\eta^2}{8} \leq \frac{\zeta\eta}{16}$ and 
$ \frac{32K\beta(\delta^2+\mu^2)\eta^2}{(1-\rho)\mu^2} \leq \frac{\zeta\eta}{16}$. (f) is based on Bernoulli inequality that $\left(1-\frac{\eta\mu \zeta}{2}\right)^K = \frac{1}{\left(1-\frac{\eta\mu \zeta}{2}\right)^{-K}} \leq \frac{1}{1+\frac{K\eta\mu\zeta}{2}} \leq 1-\frac{K\eta\mu\zeta}{4}$.

\subsection{Proof of Lemma \ref{lemma13} }\label{pf:lem-9}
Based on the iteration in local DGD, we have the property that $\forall k = 0, 1, \cdots, K$
\vspace{-2mm}
\begin{align}
& \overline{\boldsymbol x}_{r,k+1} = \overline{\boldsymbol x}_{r,k} - \frac{\eta}{m}\sum_{i=1}^m{\nabla f_i(\boldsymbol x_{r,k}^i)}\label{aviter} \\
& \boldsymbol X^{\perp}_{r,k+1} = \left(\boldsymbol X^{\perp}_{r,k} - \eta \nabla \boldsymbol F(\boldsymbol X_{r,k}) \right) \left( \boldsymbol I_m  - \boldsymbol J \right),\label{coniter} \\
& \boldsymbol X_{r+1,0}^\perp = \boldsymbol X_{r,K+1} \boldsymbol W\left( \boldsymbol I_m - \boldsymbol J \right) = \boldsymbol X_{r,K+1}^\perp \left( \boldsymbol W - \boldsymbol J \right) .\label{conseneqr}
\end{align}
According to (\ref{aviter}) and the explicit formula of gradient of quadratic function, the distance to optimum has the formula 
\begin{align}
& \overline {\boldsymbol x}_{r,k+1} - \boldsymbol x^\star = \overline {\boldsymbol x}_{r,k} - \boldsymbol x^\star - \eta \frac 1 m \sum_{i=1}^m{\left(\boldsymbol c_i \boldsymbol x^i_{r,k} - \frac 1 n \boldsymbol A_i^T \boldsymbol b_i\right)} \nonumber \\
& \quad = \overline {\boldsymbol x}_{r,k} - \boldsymbol x^\star - \frac \eta m \sum_{i=1}^m{\boldsymbol c_i \boldsymbol x^i_{r,k} } + \frac \eta m\sum_{i=1}^m{\frac{\boldsymbol A_i^T \boldsymbol b_i}{n}} \nonumber  \\
& \quad = \overline {\boldsymbol x}_{r,k} -  \boldsymbol x^\star  - \frac \eta m \sum_{i=1}^m \left(\boldsymbol c_i - \boldsymbol c \right)\left(\boldsymbol x^i_{r,k} - \overline {\boldsymbol x}_{r,k} \right) \nonumber \\
& \quad \quad - \frac \eta m\sum_{i=1}^m{\boldsymbol c \left(\boldsymbol x^i_{r,k} - \overline{\boldsymbol x}_{r,k} \right)} - \frac \eta m \sum_{i=1}^m\boldsymbol c_i \overline {\boldsymbol x}_{r,k}  + \eta \frac{\boldsymbol A^T \boldsymbol b}{N} \nonumber \\
& \quad = \overline {\boldsymbol x}_{r,k} - \boldsymbol x^\star - \frac \eta m \sum_{i=1}^m (\boldsymbol c_i - \boldsymbol c)(\boldsymbol x^i_{r,k} - \overline {\boldsymbol x}_{r,k} )  \nonumber \\
& \quad \quad - \eta \left(\boldsymbol c \overline{\boldsymbol x}_{r,k} - \frac{\boldsymbol A^T \boldsymbol b}{N} \right) \nonumber \\
& \quad \overset{(a)}{=} \overline {\boldsymbol x}_{r,k} -  \boldsymbol x^\star - \frac \eta m \sum_{i=1}^m(\boldsymbol c_i - \boldsymbol c)(\boldsymbol x^i_{r,k} - \overline {\boldsymbol x}_{r,k} ) \nonumber \\
& \quad \quad - \eta \left(\boldsymbol c \overline{\boldsymbol x}_{r,k} - \frac{\boldsymbol A^T \boldsymbol A \boldsymbol x^\star}{N} \right) \nonumber \\
& \quad = (\boldsymbol I - \eta \boldsymbol c)(\overline{\boldsymbol x}_{r,k} - \boldsymbol x^\star) + \frac \eta m \sum_{i=1}^m{(\boldsymbol c - \boldsymbol c_i)(\boldsymbol x^i_{r,k} - \overline{\boldsymbol x}_{r,k})} \nonumber \\
& e_{r,k+1} {\leq} (1-\eta\mu)^2(1+\epsilon_1) e_{r,k} + \eta^2 \delta^2 \left( 1+\frac{1}{\epsilon_1} \right) S_{r,k},
\end{align}

\noindent where $\boldsymbol c_i \triangleq  \frac{\boldsymbol A_i^T\boldsymbol A_i}{n} $ and   $\boldsymbol c \triangleq  \frac{1}{m} \sum_{i=1}^m {\boldsymbol c_i}$. (a) is due to the interpolation over-parameterization property. The last inequality is due to $\eta \leq \frac 1 L$, (i) of Lemma \ref{trieq} and Assumption \ref{hessian}.

\subsection{Proof of Lemma \ref{lemma14} }\label{pf:lem-10}
\vspace{-6mm}
\begin{align}
\boldsymbol x^i_{r,K} - \boldsymbol x^j_{r,K } &= \boldsymbol x^i_{r,K-1} - \boldsymbol x^j_{r,K-1}  \nonumber \\
& \quad - \eta \Big(\nabla f_i(\boldsymbol x^i_{r,K-1})- \nabla f_j(\boldsymbol x^j_{r,K-1})\Big) \nonumber \\
& = \boldsymbol x^i_{r,K-1} - \boldsymbol x^j_{r,K-1}  \nonumber \\
& \quad - \eta \left( \nabla f(\boldsymbol x^i_{r,K-1}) - \nabla f(\boldsymbol x^j_{r,K-1}) \right) \nonumber \\
& \quad + \eta \left( \nabla f(\boldsymbol x^i_{r,K-1}) - \nabla f_i (\boldsymbol x^i_{r,K-1}) \right) \nonumber \\
& \quad + \eta \left( \nabla f_j(\boldsymbol x^j_{r,K-1})  - \nabla f(\boldsymbol x^j_{r,K-1})\right).
\end{align}
Then there is 
\begin{align}\label{scvfold}
\vspace{1mm}
& \left\|\boldsymbol x^i_{r,K} - \boldsymbol x^j_{r,K } \right\|^2  \nonumber \\
& \; \leq  \text{\small $\left( 1+\frac{1}{\epsilon_2}\right) \left\| \boldsymbol x^i_{r,K-1} - \boldsymbol x^j_{r,K-1} - \eta \left( \nabla f(\boldsymbol x^i_{r,K-1})  - \nabla f(\boldsymbol x^j_{r,K-1}) \right) \right\|^2 $} \nonumber \\
& \quad + \left(1+\epsilon_2\right)\eta^2  \cdot \left\| \nabla f(\boldsymbol x^i_{r,K-1}) - \nabla f_i(\boldsymbol x^i_{r,K-1})\right\|^2  \nonumber \\
& \quad  + \left(1+\epsilon_2\right)\eta^2 \left\|\nabla f_j(\boldsymbol x^j_{r,K-1}) - \nabla f(\boldsymbol x^j_{r,K-1}) \right\|^2 - 2\left(1+\epsilon_2 \right)\eta^2  \nonumber \\
& \quad  \langle \nabla f(\boldsymbol x^i_{r,K-1}) - \nabla f_i(\boldsymbol x^i_{r,K-1}), \nabla f(\boldsymbol x^j_{r,K-1}) - \nabla f_j(\boldsymbol x^j_{r,K-1})\rangle  \nonumber \\
& \; \overset{(a)}{\leq } \left( 1+\frac{1}{\epsilon_2} \right) \left[ -  \frac{2\eta}{L+\mu} \left\| \nabla f(\boldsymbol x^i_{r,K-1}) - \nabla f(\boldsymbol x^j_{r,K-1})\right\|^2 \right.\nonumber \\
& \quad \left. + \left\| \boldsymbol x^i_{r,K-1} - \boldsymbol x^j_{r,K-1}\right\|^2 + \eta^2 \left\| \nabla f(\boldsymbol x^i_{r,K-1})  - \nabla f(\boldsymbol x^j_{r,K-1})\right\|^2 \right.  \nonumber \\
& \quad  \left. -2\eta \frac{L\mu}{L+\mu} \left\| \boldsymbol x^i_{r,K-1} - \boldsymbol x^j_{r,K-1} \right\|^2 \right] - 2\left(1+\epsilon_2 \right)\eta^2 \nonumber \\
& \quad  \langle \nabla f(\boldsymbol x^i_{r,K-1}) - \nabla f_i(\boldsymbol x^i_{r,K-1}), \nabla f(\boldsymbol x^j_{r,K-1}) - \nabla f_j(\boldsymbol x^j_{r,K-1})\rangle \nonumber \\
& \quad + \left(1+\epsilon_2 \right)\eta^2 \cdot \left\| (\boldsymbol c -\boldsymbol c_i)(\boldsymbol x^i_{r,K-1} - \boldsymbol x^\star) \right\|^2 + \left(1+\epsilon_2 \right)\eta^2 \nonumber \\
& \quad \cdot  \left\| (\boldsymbol c -\boldsymbol c_j)  (\boldsymbol x^j_{r,K-1} - \boldsymbol x^\star)  \right\|^2  \nonumber \\  
&\overset{(b)}{\leq} \left( 1 + \frac{1}{\epsilon_2}\right)\left(1-2\eta \frac{L\mu}{L+\mu}\right)\cdot \left\| \boldsymbol x^i_{r,K-1} -  \boldsymbol x^j_{r,K-1} \right\|^2 \nonumber \\
& \;  + 2\left(1+\epsilon_2 \right)\delta^2\eta^2 \left( \left\| \boldsymbol x^i_{r,K-1} - \overline{\boldsymbol x}_{r,K-1}\right\|^2 + \left\| \overline{\boldsymbol x}_{r,K-1} - \boldsymbol x^\star\right\|^2 \right.  \nonumber \\ 
& \;  \left.  +  \left\| \boldsymbol x^j_{r,K-1} - \overline{\boldsymbol x}_{r,K-1}\right\|^2  + \left\| \overline{\boldsymbol x}_{r,K-1} - \boldsymbol x^\star\right\|^2 \right) - 2\left(1+\epsilon_2 \right)\eta^2 \nonumber \\
& \;  \cdot  \text{\small $\langle \nabla f(\boldsymbol x^i_{r,K-1}) - \nabla f_i(\boldsymbol x^i_{r,K-1}), \nabla f(\boldsymbol x^j_{r,K-1}) - \nabla f_j(\boldsymbol x^j_{r,K-1})\rangle $} \nonumber \\
&\leq  \left( 1 + \frac{1}{\epsilon_2}\right)\left(1-\eta \mu\right)\cdot \left\| \boldsymbol x^i_{r,K-1} -  \boldsymbol x^j_{r,K-1} \right\|^2  + 2\left(1+\epsilon_2 \right)\delta^2  \nonumber \\
& \; \eta^2 \left( \left\| \boldsymbol x^i_{r,K-1} - \overline{\boldsymbol x}_{r,K-1}\right\|^2 + \left\| \overline{\boldsymbol x}_{r,K-1} - \boldsymbol x^\star\right\|^2 \right. \nonumber \\
& \;  \left. +  \left\| \boldsymbol x^j_{r,K-1} - \overline{\boldsymbol x}_{r,K-1}\right\|^2 + \left\| \overline{\boldsymbol x}_{r,K-1} - \boldsymbol x^\star\right\|^2 \right) - 2\left(1+\epsilon_2 \right)\eta^2 \nonumber \\
& \;  \cdot  \text{\small $\langle \nabla f(\boldsymbol x^i_{r,K-1}) - \nabla f_i(\boldsymbol x^i_{r,K-1}), \nabla f(\boldsymbol x^j_{r,K-1}) - \nabla f_j(\boldsymbol x^j_{r,K-1})\rangle$},
\end{align}
{where (a) is due to the co-coercive property of restricted $\mu$-strongly convex and $L$-smooth of $f$ with row space of $\boldsymbol A$ that $\langle \nabla f(\boldsymbol x) - \nabla f(\boldsymbol y), \boldsymbol x - \boldsymbol y \rangle \geq \frac{L\mu}{L+\mu}\|\boldsymbol x - \boldsymbol y\|^2 + \frac{1}{L+\mu} \left\| \nabla f(\boldsymbol x) - \nabla f(\boldsymbol y) \right\|^2$ for all $\boldsymbol x, \boldsymbol y$ lie in row space of $\boldsymbol A$.} (b) is due to $\eta \leq \frac 1 L$. 

Combining with the  consensus error equality
$\frac{1}{m} \sum_{i=1}^m\left\|\boldsymbol x^i_r - \overline{\boldsymbol x}_r \right\|^2  = \frac {\sum_{i=1}^{m}\sum_{j=1}^m{\left\|\boldsymbol x^i_r - \boldsymbol x^j_r \right\|^2}} {2m^2}$ gives
\begin{align}\label{r,K}
& S_{r+1} \leq \rho^2 S_{r,K+1}  = \frac{1}{m} \sum_{i=1}^m{\left\|\boldsymbol x^i_{r,K+1} - \overline{\boldsymbol x}_{r,K+1} \right\|^2 } \nonumber \\
& \quad \leq  \rho^2\left( 1 + \frac{1}{\epsilon_2} \right) \left( 1-\mu\eta \right)  \frac{\sum_{i=1}^m\sum_{j=1}^m{\left\| \boldsymbol x^i_{r,K} - \boldsymbol x^j_{r,K}\right\|^2}}{2m^2}  \nonumber \\
& \quad \quad + 2\rho^2\left( 1+\epsilon_2 \right)\delta^2\eta^2 \left( S_{r,K} +  e_{r,K} \right)\nonumber \\
& \quad \quad  - 2\left(1+\epsilon_2 \right)\eta^2 {\left\|  \sum_{i=1}^m{\nabla f(\boldsymbol x^i_{r,K}) - \nabla f_i(\boldsymbol x^i_{r,K})} \right\|^2} \nonumber \\
& \quad \leq \rho^2\left(1+\frac 1 {\epsilon_2} \right)\left(1-\eta \mu \right)\cdot S_{r,K} + 2\rho^2\left(1+\epsilon_2\right)\delta^2\eta^2\cdot \nonumber \\
& \quad \quad \left( S_{r,K} + e_{r,K} \right) \nonumber \\
& \quad = \rho^2\left[\left(1+\frac 1 {\epsilon_2} \right) \left(1-\mu\eta \right) + 2\left(1+\epsilon_2\right)\delta^2\eta^2   \right]  S_{r,K}\nonumber \\
& \quad \quad + 2\left(1+\epsilon_2\right)\rho^2\delta^2\eta^2 \cdot e_{r,K} \nonumber \\
& \quad \leq {\rho^2 a^{K+1}} S_r  + \rho^2 b\sum_{i=0}^{K}a^{K-i} e_{r,i},
\end{align}
where the term $\left\| \sum_{i=1}^m { \left(  \nabla f\left(\boldsymbol x^i_{r,K} \right) - \nabla f_i\left( \boldsymbol x^i_{r,K} \right) \right)  } \right\|^2$ is due to the summation of inner product term $\langle \nabla f(\boldsymbol x^i_{r,K-1}) - \nabla f_i(\boldsymbol x^i_{r,K-1}) , \nabla f(\boldsymbol x^j_{r,K-1}) - \nabla f_j(\boldsymbol x^j_{r,K-1})\rangle$ in \eqref{scvfold}.  

For simplicity, we define
$a \triangleq \left(1+\frac 1 {\epsilon_2} \right)\left(1 - \eta \mu \right) + b, b \triangleq 2\left(1+\epsilon_2\right)\delta^2\eta^2$. 
Then choosing $\eta$ to satisfy the following condition $2(1+\epsilon_2) \eta^2\delta^2 \leq \left(1+\frac{1}{\epsilon_2}\right)\frac{\eta\mu}{2} \Longrightarrow \eta \leq \frac {\mu}{4\epsilon_2 \delta^2}, \quad a \leq \left(1 +\frac{1}{\epsilon_2}\right)\left(1 - \frac{\eta\mu}{2}\right)$ leads to
\begin{align}\label{recon}
& S_{r+1} \leq \rho^2 \left[\left(1 +\frac{1}{\epsilon_2}\right)\left(1 - \frac{\eta\mu}{2}\right) \right]^{K+1} S_r \nonumber \\
& \quad \quad + 2(1+\epsilon_2)\left(1+\frac 1 {\epsilon_2} \right)^K\delta^2 \sum_{i=0}^{K}{\eta^2 e_{r,i} } \nonumber \\
& \quad \leq \rho \left(1 + \frac{1}{\epsilon_2}\right)^{K+1}  S_r + 2\delta^2(1+\epsilon_2)\left(1 +\frac{1}{\epsilon_2}\right)^K  \sum_{i=0}^{K}\eta^2 e_{r,i} \nonumber \\
 & \quad \leq \left( \frac{1+\rho}{2}\right) \cdot S_r + \frac{32K\delta^2}{1-\rho} \sum_{i=0}^{K}{\eta^2 e_{r,i}},
\end{align}
\noindent where we set $\epsilon_2 = \frac{1}{\ln\left(1+\frac{1-\rho}{2}\right)}(K+1)$, then $\left(1+\frac{1}{\epsilon_2} \right)^{K+1} = \left(1+\frac{\ln\left(1+\frac{1-\rho}{2}\right)}{K+1}\right)^{K+1} \leq {1+\frac{1-\rho}{2}}$. The last inequality is due to $1+\frac{K+1}{\ln \left(1+\frac{1-\rho}{2} \right)} \leq \frac{8K}{1-\rho}$ based on Jensen inequality. {Based on \eqref{r,K}, there is 
\begin{align}
S_{r,k} &\leq a^k S_r + b \sum_{i=0}^{k-1} a^{k-1-i} e_{r,i} \nonumber \\
 & \overset{(a)}{\leq} \left( 1 + \frac 1 {\epsilon_2}\right)^K \cdot S_r  + 2(1+\epsilon_2) \left(1 + \frac 1 {\epsilon_2}\right)^K \delta^2 \eta^2 \sum_{i=0}^{k-1} e_{r,i} \nonumber \\
 & \overset{(b)}{\leq} \left( 1+\frac{1-\rho}{2} \right) \cdot S_r + 2(1 + \epsilon_2) \left( 1 + \frac{1-\rho} 2 \right) \delta^2 \eta^2 \sum_{i=0}^{k-1} e_{r,i} \nonumber \\
 & \overset{(c)}{\leq} 2 S_r + 4 \delta^2 \eta^2 (1+\epsilon_2) \sum_{i=0}^{k-1} e_{r,i} \nonumber \\
 & \overset{(d)}{\leq}  2 S_r + \frac{32K\delta^2\eta^2}{1-\rho} \sum_{i=0}^{k-1} e_{r,i},
\end{align}
where (a) follows from the definitions of $a$ and $b$, (c) is due to $1+\frac{1-\rho}{2} <2$ and (b), (d) is due to the setting of $\epsilon_2$. Summing over $k$ would obtain the first upper bound in the lemma. }

\subsection{Additional Definitions and Lemmas}
We first recall the error bound condition of $L$-smooth function $f: \mathbb{R}^d \to \mathbb{R}$. Let set $\boldsymbol{\mathcal X}^\star$ denote the minimizers of $f$ with minimum value $f^\star$. Given $\bm{x}$, let $\pi(\boldsymbol x)$ be the Euclidean projection of $\bm{x}$ onto $\boldsymbol{\mathcal X}^\star$.
\begin{definition}
(Error Bound (EB)) The function $f$ satisfies error bound condition with parameter $\mu$ if 
$
 \| \nabla f(\boldsymbol x) \| \geq \mu\| \boldsymbol x - \pi(\boldsymbol x)\|  , \quad \forall \boldsymbol x \in \mathbb{R}^d. 
$

\end{definition}

The following definition of $\tau$-slow sequence given in \cite{stich2020error} is useful for unrolling recursions and deriving convergence rates.
\begin{definition}\label{def:slow-seq} 
The sequence $\{a_t\}_{t\geq 0}$ of positive values is $\tau$-slow decreasing for parameter $\tau>0$ if 
\begin{align*}
a_{t+1} \leq a_t,\quad \forall t \geq 0 \quad \text{and}, \quad a_{t+1}\left( 1+ \frac{1}{2\tau}\right) \geq a_t,\quad \forall t \geq 0.
\end{align*}
The sequence $\{a_t\}_{t\geq 0}$ is $\tau$-slow increasing if $\{a_t^{-1}\}_{t \geq 0}$ is $\tau$-slow decreasing.
\end{definition}

\begin{lemma}[\!\!\!\cite{karimi2016linear}]\label{PL-EB-QG}
{
If $f$ satisfies the PL condition with parameter $\mu$, then the error bound condition holds for $f$ with the same parameter $\mu$.
}
\end{lemma}

\begin{lemma}[\!\!\cite{zlobec2004jensen}]\label{weakly-con}
  If $f:\mathbb{R}^d \to \mathbb{R}$ is  $\rho$-weakly convex, then $\forall \boldsymbol x_1, \cdots, \boldsymbol x_m \in \mathbb{R}^d$ and $a_i \geq 0 $, $i \in [m]$ such that $\sum_{i=1}^m{a_i} =1$, the following holds:
 \begin{align*}
f\left( \sum_{i=1}^m {a_i \boldsymbol x_i}\right) \leq \sum_{i=1}^m {a_if(\boldsymbol x_i)} + \frac{\rho}{2} \sum_{i=1}^m \sum_{j=1}^m {a_ia_j\|\boldsymbol x_i -\boldsymbol x_j\|^2}.
 \end{align*}
\end{lemma}

\begin{lemma}\label{linearcon}
(\!\!\cite{koloskova2020unified}) If non-negative sequences $\{r_t\}_{t\geq 0}, \{e_t\}_{t \geq 0}$ satisfy the following condition for some constants $a, b>0, c, A, B \geq 0$, 
\begin{align}\label{conver_rate_expre}
 \frac{\sum_{t=0}^T{bz_te_t}}{2z_T} &\leq \frac{1}{z_T} \sum_{t=0}^T { \left( \frac{(1-a\eta_t)z_t}{\eta_t}r_t - \frac{z_t}{\eta_t}r_{t+1} \right) } \nonumber \\
 & \quad + \frac{c}{Z_T} \sum_{t=0}^T {z_t\eta_t} + \frac{64BA}{Z_T} \sum_{t=0}^T {z_t\eta_t^2}.    
\end{align}
then there exists a constant step size $\eta_t = \eta \leq \frac 1 d$ such that for weights $z_t = \left(1-a\eta\right)^{-(t+1)}$ and $Z_T = \sum_{t=0}^T {z_t}$ it holds:   
\begin{align}
\frac{1}{2Z_T} &\sum_{t=0}^T {bz_te_t} + a r_{T+1} \nonumber \\
& \leq \Tilde{\mathcal O} \left( r_0 d \exp\left[  -\frac{a(T+1)}{d}\right]   + \frac{c}{aT} + \frac{AB}{a^2T^2}\right),
\end{align}
where $\Tilde{\mathcal O}$ hides polylogarithmic factors. 
\end{lemma}

\begin{lemma}\label{trieq}
 Let $\{\boldsymbol v_1, \cdots, \boldsymbol v_\tau\}$ be $\tau$ vectors in $\mathbb R^d$. Then the following properties hold\\
$(i) \left\|  \boldsymbol v_i+\boldsymbol v_j\right\|^2 \leq (1+a)\left\|\boldsymbol v_i\right\|^2 + (1+\frac 1 a)\left\| \boldsymbol v_j \right\|^2, \forall a>0, \\
(ii) \left\| \sum_{i=1}^\tau {\boldsymbol v_i} \right\|^2 \leq \tau \sum_{i=1}^\tau {\left\| \boldsymbol v_i\right\|^2}.$
\end{lemma}

\end{document}